\documentclass{article}
\PassOptionsToPackage{numbers,sort&compress}{natbib}
\usepackage{natbib}
\usepackage[font=small]{caption}
\usepackage[rmargin = 2.5cm, lmargin = 2.5cm, tmargin = 2.5cm, bmargin = 2.5cm]{geometry}
\usepackage{authblk}

\usepackage{graphicx}
\usepackage{physics}
\usepackage{amsmath}
\usepackage{amssymb}
\usepackage{amsthm}

\usepackage[utf8]{inputenc} 
\usepackage[T1]{fontenc}    
\usepackage{xr-hyper}
\usepackage{hyperref}       
\usepackage{url}            
\usepackage{booktabs}       
\usepackage{amsfonts}       
\usepackage{nicefrac}       
\usepackage{microtype}      
\usepackage{xcolor}         
\usepackage{chngcntr}
\usepackage{tabularx}
\usepackage{booktabs}
\counterwithin{equation}{section}
\usepackage[]{algorithm2e}

\newtheorem{theorem}{Theorem}
\newtheorem{corollary}{Corollary}

\usepackage{setspace} 

\def    \be            {\begin{equation}}
\def    \ee            {\end{equation}}
\def    \bea           {\begin{eqnarray}}
\def    \eea           {\end{eqnarray}}

\DeclareMathOperator*{\argmax}{arg\,max}
\DeclareMathOperator*{\argmin}{arg\,min}

\newcommand{\subsubsubsection}[1]{\paragraph{#1}\mbox{}\\}
\makeatletter
\newcommand*{\addFileDependency}[1]{
  \typeout{(#1)}
  \@addtofilelist{#1}
  \IfFileExists{#1}{}{\typeout{No file #1.}}
}
\makeatother

\title{Complex behavior from intrinsic motivation to occupy future action-state path space
} 

\author[1]{Jorge Ramírez-Ruiz}
\author[1]{Dmytro Grytskyy}
\author[1]{Chiara Mastrogiuseppe}
\author[1]{Yamen Habib}
\author[1,2]{Rubén Moreno-Bote}

\affil[1]{Center for Brain and Cognition, and Department of Information and Communication Technologies, Universitat Pompeu Fabra, Barcelona, Spain 08005}
\affil[2]{Serra Húnter Fellow Programme, Universitat Pompeu Fabra, Barcelona, Spain}
\date{\today}

\usepackage[mathlines]{lineno}
   
\begin{document}
   	
\maketitle

\begin{abstract}
Most theories of behavior posit that agents tend to maximize some form of reward or utility. 
However, animals very often move with curiosity and seem to be motivated in a reward-free manner. 
Here we abandon the idea of reward maximization, and propose that the goal of behavior is maximizing occupancy of future paths of actions and states. According to this maximum occupancy principle, rewards are the means to occupy path space, not the goal per se; goal-directedness simply emerges as rational ways of searching for resources so that movement, understood amply, never ends. 
We find that action-state path entropy is the only measure consistent with additivity and other intuitive properties of expected future action-state path occupancy. 
We provide analytical expressions that relate the optimal policy and state-value function, and prove convergence of our value iteration algorithm. 
Using discrete and continuous state tasks, including a high--dimensional controller, we show that complex behaviors such as `dancing', hide-and-seek and a basic form of altruistic behavior naturally result from the intrinsic motivation to occupy path space. 
All in all, we present a theory of behavior that generates both variability and goal-directedness in the absence of reward maximization.
\end{abstract}

\section{Introduction}

Natural agents are endowed with a tendency to move, explore and interact with their environment \cite{ryan_intrinsic_2000,oudeyer_intrinsic_2007}.
For instance, human newborns unintentionally move their body parts \cite{adolph_motor_2007}, and 7 to 12-months infants spontaneously babble vocally \cite{macneilage_origin_2000} and with their hands
\cite{petitto_babbling_1991}.
Exploration and curiosity are major drives for learning and discovery through information-seeking
\cite{dietrich_cognitive_2004,kidd_psychology_2015,gottlieb_information-seeking_2013}.
These behaviors seem to elude a simple explanation in terms of extrinsic reward maximization.
However, these intrinsic motivations push agents to visit new states by performing novel courses of action, which helps learning and the discovery of even larger rewards in the long run
\cite{gittins_multi-armed_2011,averbeck_theory_2015}. Therefore, it has been argued that exploration and curiosity could arise as a consequence of seeking extrinsic reward maximization by endowing agents with the necessary inductive biases to learn in complex and ever-changing natural environments \cite{doll_ubiquity_2012,wang_latent_2021}. 

While most theories of rational behavior do posit that agents are reward maximizers \cite{von_neumann_theory_2007,sutton_introduction_1998,kahneman_prospect_2013,silver_reward_2021}, very few of us would agree that the sole goal of living agents is maximizing money gains or food intake.
Indeed, expressing excessive emphasis on these types of goals is usually seen as a sign of psychological disorders 
\cite{rash_review_2016,agh_systematic_2016}. 
Further, setting a reward function by design as the goal of artificial agents is, more often than not, arbitrary
\cite{sutton_introduction_1998,mcnamara_common_1986,klyubin_empowerment_2005,lehman_abandoning_2011}, leading to the recurrent problem faced by theories of reward maximization of defining what rewards are
\cite{singh_where_2009,zhang_endotaxis_2021,schmidhuber_possibility_1991,hadfield-menell_inverse_2017,eysenbach_diversity_2018}. 
In some cases, like in artificial games, rewards can be unambiguously defined, such as number of collected points or wins
\cite{schrittwieser_mastering_2020}. 
However, in most situations defining rewards is task-dependent, non-trivial and problematic.
For instance, a vacuum cleaner robot could be designed to either maximize the weight or volume of dust collected, energy efficiency, or a weighted combination of them \cite{asafa_development_2018}.
In more complex cases, companies can aim at maximizing profit, but without a suitable innovation policy profit maximization can be self-defeating \cite{kline_overview_2010}.

Here, we abandon the idea that the goal is maximizing extrinsic rewards and that movement over space is a means to achieve this goal. Instead, we adopt the opposite view, inspired by the nature of our intrinsic drives: we propose that the objective {\em is} to maximally occupy action-state path space, understood in a broad sense, in the long term. We call this principle the maximum occupancy principle (MOP), which posits that the goal of agents is to generate all sort of behaviors and occupy, on average, as much space (action-state paths) as possible in the future.
According to MOP, extrinsic rewards serve to obtain the energy necessary to move in order to occupy action-state space, they are not the goals per se. The usual exploration--exploitation tradeoff \cite{wilson_balancing_2021} therefore disappears: agents that seek to occupy space ``solve'' this issue naturally because they care about rewards only as means to an end. 
Furthermore, in this sense, surviving is only preferred because it is needed to keep visiting action-state space.
Our theory
provides a rational account of exploratory and curiosity-driven behavior where the problem of defining a reward function vanishes, and captures the variability of behavior
\cite{moreno-bote_bayesian_2011,recanatesi_metastable_2022,corver_distinct_2021,dagenais_elephants_2021,mochol_prefrontal_2021,cazettes_reservoir_2021} by taking it as a principle.

In this work, we model a MOP agent interacting with the environment as a Markov decision process (MDP) where the intrinsic, immediate reward is the occupancy of the next action-state visited, which is largest when performing an uncommon action and visiting a rare state --there are no extrinsic rewards (i.e., no task is defined) that drive the agent.
We show that (weighted) action-state path entropy is the only measure of occupancy consistent with additivity per time step, positivity and smoothness. 
Due to the additivity property, the value of being in a state, defined as the expected future time-discounted action-state path entropy, can be written in the form of a Bellman equation, which has a unique solution that can be found with an iterative map. Following this entropy objective leads to agents that seek variability, while being sensitive to the constraints imposed by the agent-environment interaction on the future path availability. We demonstrate in various simulated experiments with discrete and continuous state and action spaces that MOP generates complex behaviors that, to the human eye, look genuinely goal-directed and playful, such as hide-and-seek in a prey-predator problem, dancing of a cartpole, a basic form of altruism in an agent-and-pet example, and rich behaviors in a high-dimensional quadruped. 

MOP builds over an extensive literature on entropy-regularized reinforcement learning (RL) \cite{todorov_efficient_2009,ziebart_modeling_2010,haarnoja_reinforcement_2017,haarnoja_soft_2018,schulman_equivalence_2017,neu_unified_2017,hausman_learning_2018,tishby_information_2011,nachum_bridging_2017,galashov_information_2019,grytskyy_general_2023} or pure entropic objectives \cite{hazan_provably_2019, liu_behavior_2021,mutti_task-agnostic_2021,seo_state_2021,zhang_exploration_2021,amin_survey_2021}. This body of work emphasizes the regularization benefits of entropy for learning, but extrinsic rewards still serve as the major drive of behavior, and arbitrary mixtures of action-state entropy are rarely considered \cite{grytskyy_general_2023}.
Our work also relates to reward-free theories of behavior. These minimize predictions errors \cite{burda_exploration_2018,achiam_surprise-based_2017,fountas_deep_2020,burda_large-scale_2019,pathak_curiosity-driven_2017,hafner_action_2022}, seek novelty \cite{bellemare_unifying_2016,tang__2017,aubret_information-theoretic_2023}, or maximize data compression \cite{schmidhuber_driven_2009}, and therefore the major behavioral driver depends on the agent's experience with the world. On the other hand, MOP agents find the action-states that lead to high future occupancy ``interesting'', regardless of experience.
There are two other approaches that sit closer to this description, one maximizing mutual information between actions and future states (empowerment, MPOW) \cite{klyubin_empowerment_2005,jung_empowerment_2011,still_information-theoretic_2012,mohamed_variational_2015}, and the other minimizing the distance between the actual and a desired state distribution (free energy principle, FEP) \cite{friston_free_2006,buckley_free_2017}. We show that both MPOW and FEP tend to collapse to deterministic policies with little behavioral variability. 
In contrast, MOP results in lively and seemingly goal-directed behavior by taking behavioral variability and the constraints of embodied agents as principles.  

\section{Maximum occupancy principle}

\begin{figure}[t]
\center
\includegraphics[width=0.95\textwidth]{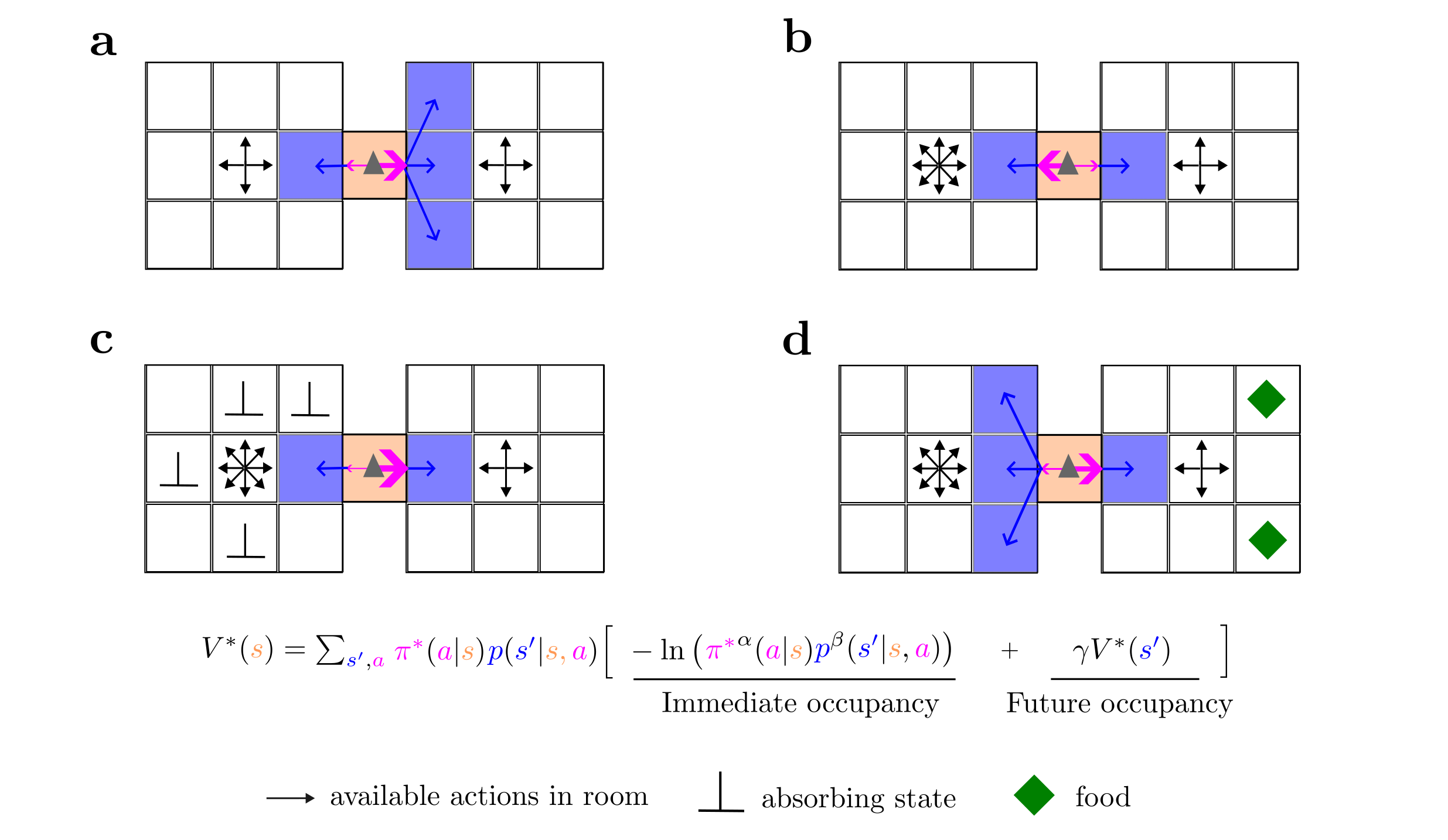}
\caption{MOP agents maximize action-state path occupancy. 
(a) A MOP agent (grey triangle) in the middle of two rooms has the choice between going left or right. When the number of actions (black arrows) in each room is the same, the agent prefers going to the room with more state transitions (blue arrows indicate random transitions after choosing moving right or moving left actions, and pink arrow width indicates the probabilities of those actions).
(b) When the states transitions are the same in the two rooms, the MOP agent prefers the room with more available actions.
(c) If there are many absorbing states in the room where many actions are available, the MOP agent avoids it.
(d) Even if there are action and state-transition incentives (in the left room), a MOP agent might prefer a region of state space where it can reliably get food (right room), ensuring occupancy of future action-state paths. See Supplemental Fig. \ref{fig:schematic_formal} for a more formal example.
}
\label{fig:schematic}
\end{figure}

\subsection{Entropy measure of path space occupancy}

We model an agent as a finite action-state MDP in discrete time. The policy $\pi$ describes the probability $\pi(a|s)$ of performing action $a$, from some set ${\cal{A}}(s)$, given that the agent is at state $s$ at some time step, and $p(s'|s,a)$ is the transition probability from $s$ to a successor state $s'$ in the next time step given that action $a$ is performed.
Starting at $t=0$ in state $s_0$, an agent performing a sequence of actions and experiencing state transitions $\tau \equiv (s_0,a_0,s_1,...,a_t,s_{t+1},...)$ gets a return defined as
\be
R(\tau) = \sum_{t=0}^{\infty} \gamma^t R(s_t,a_t) =
- \sum_{t=0}^{\infty} \gamma^t \ln  \left( \pi^{\alpha}(a_t|s_t) p^{\beta}(s_{t+1}|s_t,a_t) \right)
\;,
\label{eq_return}
\ee
\noindent
with action and state weights $\alpha>0$ and $\beta \geq 0$, respectively, and discount factor $0 < \gamma < 1$.
A larger return is obtained when, from $s_t$, a low-probability action $a_t$ is performed and followed by a low-probability transition to a state $s_{t+1}$. 
Therefore, maximizing the return in Eq. (\ref{eq_return}) favors `visiting' action-states $(a_t,s_{t+1})$ with a low transition probability.
From $s_{t+1}$, another low-probability action-state transition is preferred and so on, such that low-probability trajectories $\tau$ are more rewarding than high-probability ones.
Thus, the agent is pushed to visit action-states that are rare or `unoccupied', implementing the intuitive notion of MOP.
Due to the freedom to choose action $a_t$ given state $s_t$ and the uncertainty of the resulting next state $s_{t+1}$, apparent in Eq. (\ref{eq_return}), the term `action-states' used here is more natural than `state-actions'. 
We stress that this return is purely intrinsic, namely, there is no extrinsic reward that the agent seeks to maximize. We define intrinsic rewards as any reward signal that depends on the policy or the state transition probability, and therefore it can change with the course of learning as the policy is improved, or the environment is learnt. An extrinsic reward is the complementary set of reward signals: any function $R(s,a)$ that is both policy-independent and transition probability-independent, and therefore it does not change with the course of improving the policy or learning the state transition probability of the environment. 
 
The agent is assumed to optimize the policy $\pi$ to maximize the state-value $V_{\pi}(s)$, defined as the expected return
\be
V_{\pi}(s) 
\equiv 
\mathbb{E}_{a_t \sim \pi,s_{t+1} \sim p} [ R({\tau} ) | s_0 = s  ] 
=
\mathbb{E}_{a_t \sim \pi,s_{t+1} \sim p} \left[  
\sum_{t=0}^{\infty} \gamma^t  
\left( \alpha \mathcal{H}(A|s_t) + \beta \mathcal{H}(S'|s_t,a_t) \right)
\Big| s_0 = s  \right]
\label{eq_expected_return}
\ee
\noindent
given the initial condition $s_0=s$ and following policy $\pi$, that is,
the expectation is over the $a_t \sim \pi(\cdot|s_t)$ and $s_{t+1} \sim p(\cdot|s_t,a_t)$, $t \geq 0$. In the last identity, we have rewritten the expectations of the terms in Eq. (\ref{eq_return}) as a discounted and weighted sum of action and successor state conditional entropies $\mathcal{H}(A|s) = - \sum_a \pi(a|s) \ln \pi(a|s)$ and $\mathcal{H}(S'|s,a) = - \sum_{s'}  p(s'|s,a) \ln p(s'|s,a)$, respectively, averaged over previous states and actions.

We define a MOP agent as the one that optimizes the policy to maximize the state-value in Eq. (\ref{eq_expected_return}). The entropy representation in Eq. (\ref{eq_expected_return}) of MOP has several implications. First, agents prefer regions of state space that lead to a large number of successor states (Fig. \ref{fig:schematic}a) or larger number of actions (Fig. \ref{fig:schematic}b).
Second, death (absorbing) states where only one action-state (i.e., ``stay'') is available forever are naturally avoided by a MOP agent, as they promise zero future action and state entropy (Fig. \ref{fig:schematic}c). Therefore, our framework implicitly incorporates a survival instinct.
Finally, regions of state space where there are ``rewarding'' states that increase the capacity of the agent to visit further action-states (such as filling an energy reservoir) are more frequently visited than others (Fig. \ref{fig:schematic}d).

We found that maximizing the discounted action-state path entropy in Eq. (\ref{eq_expected_return}) is the only reasonable way of formalizing MOP, as it is the only measure of action-state path occupancy in Markov chains consistent with the following intuitive conditions (Supplemental Sec. \ref{sec:entropy-measures}): 
if a path $\tau$ has probability $p$, visiting it results in an occupancy gain $C(p)$ that (i) decreases with $p$ and (ii) is first-order differentiable. Condition (i) implies that visiting a low probability path increases occupancy more than visiting a high probability path, and our agents should tend to occupy `unoccupied' path space; condition (ii) requires that the measure should be smooth. We also ask that (iii) the occupancy of paths, defined as the expectation of occupancy gains over paths given a policy,
is the sum of the expected occupancies of their subpaths (additivity condition). This last condition implies that agents can accumulate occupancy over time by keeping visiting low-probability action-states, but the accumulation should be consistent with the Markov property of the decision process. These conditions are similar but not exactly the same as Shannon's information measure \cite{shannon_mathematical_1948} (Supplemental Sec. \ref{sec:entropy-measures}).

\subsection{Optimal policy and state-value function}

The state-value $V_{\pi}(s)$ in Eq. (\ref{eq_expected_return}) can be recursively written using the values of successor states through the standard Bellman equation
\bea
\nonumber
V_{\pi}(s) 
&=& 
 \alpha \mathcal{H}(A|s) + \beta \sum_a \pi(a|s) \mathcal{H}(S'|s,a)
+
\gamma \sum_{a,s'} \pi(a|s) p(s'|s,a) V_{\pi}(s')
\\
& =& \sum_{a,s'} \pi(a|s) p(s'|s,a) 
\left( -\alpha \ln \pi(a|s) - \beta \ln p(s'|s,a) + \gamma V_{\pi}(s')
\right)
,
\label{eq_expected_value}
\eea
\noindent
where the sum is over the available actions $a$ from state $s$, ${\cal{A}}(s)$, and over the successor states $s'$ given the performed action at state $s$. The optimal policy $\pi^*$ that maximizes the state-value is defined as $\pi^*= \argmax_{\pi}  V_{\pi}$ and the optimal state-value is 
\be 
V^*(s) = \max_{\pi} V_{\pi}(s),
\label{eq_bellman_opt_eq}
\ee
where the maximization is with respect to the $\{\pi(\cdot|\cdot)\} $ for all actions and states.
To obtain the optimal policy, we first determine the critical points of the expected return $V_{\pi}(s)$ in Eq. (\ref{eq_expected_value}) using Lagrange multipliers (Supplemental Sec. \ref{Sec:critical-policies}). 
The optimal state-value $V^*(s)$ is found to obey the non-linear self-consistency set of equations
\be
V^*(s) = \alpha \ln Z(s) = \alpha \ln	\left[
\sum_{a} \exp 
\left(
\alpha^{-1} \beta \mathcal{H}(S'|s,a)
+ \alpha^{-1} \gamma \sum_{s'} p(s'|s,a)  V^*(s')
\right)
\right]
,
\label{eq_v_opt_m}
\ee
\noindent
where $Z(s)$ is the partition function, defined by substitution, and the critical policy satisfies
\be
\pi^*(a|s) = \frac{1}{Z(s)} \exp 
\left(
\alpha^{-1} \beta \mathcal{H}(S'|s,a)
+ \alpha^{-1} \gamma \sum_{s'} p(s'|s,a)  V^*(s')
\right)
.
\label{eq_pi_opt_m}
\ee
\noindent
We find that the solution to the non-linear system of Eqs. (\ref{eq_v_opt_m}) is unique and, moreover, 
the unique solution is the absolute maximum of the state-values over all policies (Supplemental Sec. \ref{Sec:unicity}).

To determine the actual value function from such non-linear set of equations, we derive an iterative map, a form of value iteration that exactly incorporates the optimal policy at every step. Defining $z_i=\exp( \alpha^{-1} \gamma V(s_i) )$, $p_{ijk}=p(s_j|s_i,a_k)$ and $\mathcal{H}_{ik} = \alpha^{-1} \beta \mathcal{H}(S'|s_i,a_k)$, Eq. (\ref{eq_v_opt_m}) can be turned into the iterative map 
\be
z_i^{(n+1)} = 
\left(
\sum_k w_{ik} e^{\mathcal{H}_{ik}} 
\prod_j \left(z_j^{(n)}\right)^{p_{ijk}}	
\right)^{\gamma} 
\label{eq_z_map_m}
\ee
\noindent
for $n \geq 0$ and with initial conditions $z_i^{(0)} > 0$. 
Here, the matrix with coefficients $w_{ik} \in \{0,1\}$ indicate whether action $a_k$ is available at state $s_i$ ($w_{ik} = 1$) or not ($w_{ik}=0$), and $j$ extends over all states, with the understanding that if a state $s_j$ is not a possible successor from state $s_i$ after performing action $a_k$ then $p_{ijk}=0$. 
We find that the infinite series $z_i^{(n)}$ defined in Eq. (\ref{eq_z_map_m}) converges to a finite limit $z_i^{(n)} \rightarrow z_i^{\infty}$ regardless of the initial condition in the positive first orthant, and that $V^*(s_i) = \alpha \gamma^{-1} \ln z_i^{\infty} $ is the optimal state-value function, which solves Eq. (\ref{eq_v_opt_m}) (Supplemental Sec. \ref{Sec:unicity}). 
Iterative maps similar to Eq. (\ref{eq_z_map_m}) have been studied before \cite{todorov_efficient_2009,todorov_linearly-solvable_2006}, subsequently shown to have uniqueness \cite{rubin_trading_2012} and convergence guarantees \cite{nachum_bridging_2017,leibfried_unified_2019} in the absence of state entropy terms. A summary of results and particular examples can be found in Supplemental Sec. \ref{sec:examples}.

We note that in the definition of return in Eq. (\ref{eq_expected_return}) we could replace the absolute action entropy terms $\mathcal{H}(A|s)$ by relative entropies of the form 
$-D_{\text{KL}}(\pi(a|s)||\pi_0(a|s)) =\sum_a \pi(a|s) \ln (\pi_0(a|s) / \pi(a|s))$, 
as in KL-regularization \cite{todorov_efficient_2009,todorov_linearly-solvable_2006,schulman_equivalence_2017,galashov_information_2019}, but in the absence of any extrinsic rewards. 
In this case, one obtains an equation identical to (\ref{eq_z_map_m}) where the coefficients $w_{ik}$ are simply replaced by $\pi_0(a_k|s_i)$, one to one. This apparently minor variation undercovers a major qualitative difference between absolute and relative action entropy objectives: as $\sum_k w_{ik} \geq 1$, absolute entropy-seeking favors visiting states with a large action accessibility, that is, where the sum $\sum_k w_{ik}$ and thus the argument of Eq. (\ref{eq_z_map_m}) tends to be largest. In contrast, as $\sum_k \pi_0(a_k|s_i) = 1$, maximizing relative entropies provides no preference for states $s$ with large number of accessible actions $|\mathcal{A} (s)|$. This happens even if the default policy is uniform in the actions, as then the immediate intrinsic return becomes $-D_{\text{KL}}(\pi(a|s)||\pi_0(a|s)) = \mathcal{H}(A|s) - \ln |\mathcal{A} (s)|$, instead of $\mathcal{H}(A|s)$. The negative logarithm penalizes visiting states with large number of actions, which is the opposite goal to occupying action-state path space (see details in Supplemental Sec. \ref{sec:supplemental_KL}). 

\section{Results}

\subsection{MOP agents quickly fill physical space} 

In very simple environments with high symmetry and little constraints, like open space, maximizing path occupancy amounts to performing a random walk that chooses at every step any available action with equal probability. However, in realistic environments where space is not homogeneous, where there are energetic limitations for moving, or where there are absorbing states, a random walk is no longer optimal.
To illustrate how interesting behaviors arise from MOP in these cases, we first tested how a MOP agent moving in a 4-room and 4-food-sources environment (Fig. \ref{fig:fourrooms}a)
compares in occupying physical space to a random walker (RW) and to a reward seeking agent (R agent).
The definition of the three agents are identical in most ways. They have nine possible movement actions, including not moving; they all have an internal state corresponding to the available energy, which reduces one unit at every time step and gets increased by a fixed amount (food gain) whenever a food source is visited; and they can move as long as their energy is non-zero. 
The total state space is the Cartesian product between physical space and internal energy. The agents differ however in their objective function.
The MOP agent has a reward-free objective and implements MOP by maximizing path action entropy, Eq. (\ref{eq_expected_return}). In contrast, the R agent maximizes future discounted reward (in this case, food), 
and displays stochastic behavior through  an $\epsilon$-greedy action selection, with $\epsilon$ matched to the survival of the MOP agent (Supplemental Sec. \ref{sec:experiments} and Fig. \ref{fig:survival}a). Finally, the random walker is simply an agent that in each state takes a uniformly random action from the available actions at that state.

\begin{figure}[t]
\center
\includegraphics[width=\textwidth]{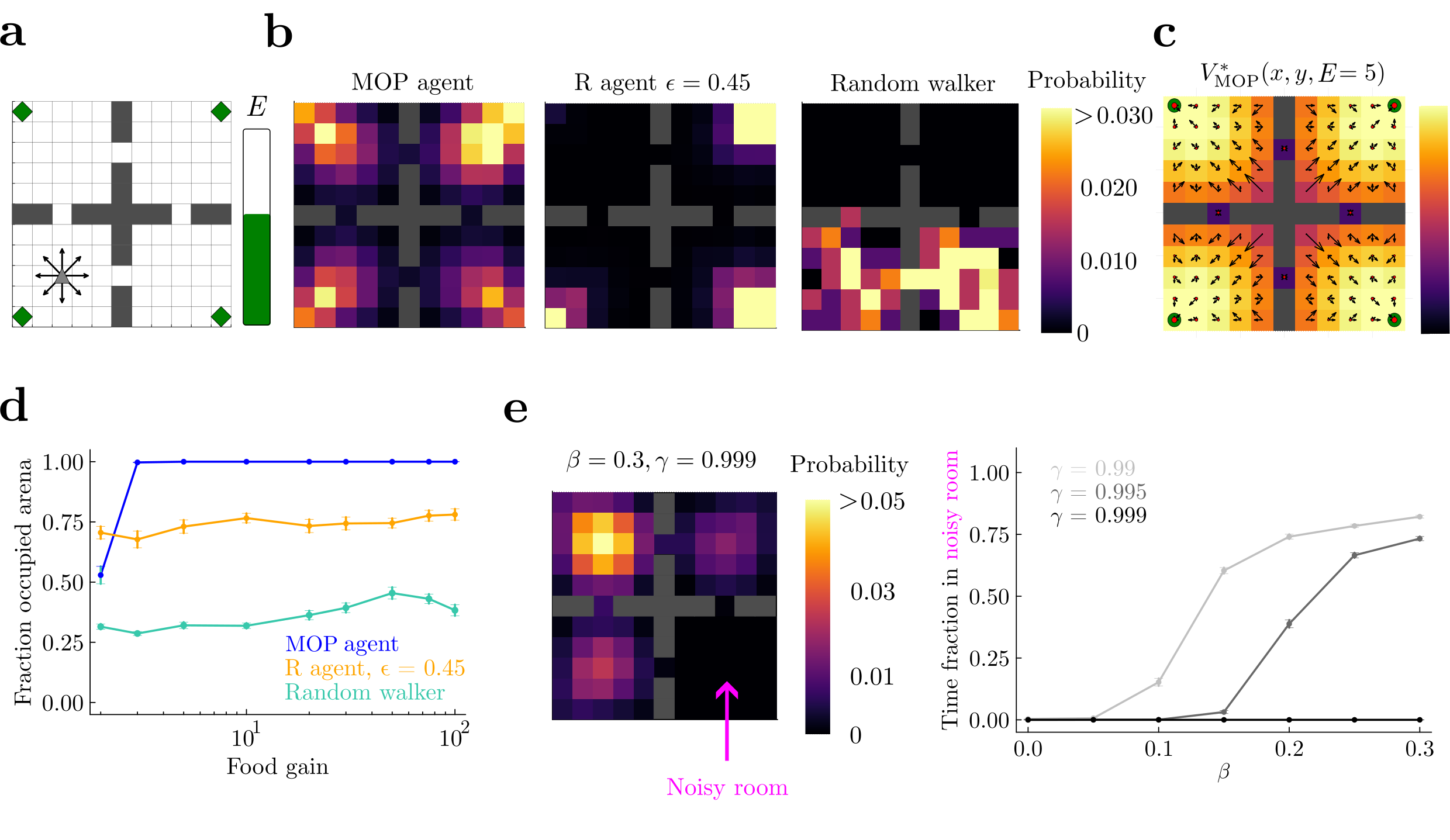}
\caption{Maximizing future path occupancy leads to high occupancy of physical space. (a) Grid-world arena. The agents have nine available actions (arrows, and staying still) when alive (internal energy $E$ larger than zero) and away from walls. There are four rooms, each with a small food source in a corner (green diamonds).  
(b) Probability of visited spatial states for a MOP agent, an $\epsilon$-greedy reward (R) agent that survives as long as the MOP agent, and a random walker. Food gain $=10$ units, maximum reservoir energy $=100$, episodes of $5\times 10^4$ time steps, and $(\alpha,\beta)=(1,0)$ for the MOP agent. All agents are initialized in the middle of the lower left room.
(c) Optimal value function $V^*(s)$ over locations when energy is $E = 5$. Black arrows represent the optimal policy given by Eq. \ref{eq_pi_opt_m}; their length is proportional to the probability of each action. The size of red dots is proportional to the probability of the \texttt{do nothing} action.
(d) Fraction of locations of the arena visited at least once per episode as a function of food gain. Error bars correspond to s.e.m over $50$ episodes.
(e) Noisy room problem. The bottom right room of the arena was noisy, such that agents in this room jump randomly to neighboring locations regardless of their actions. Food gain equals maximum reservoir energy $=100$. Histogram of visited locations for an episode as long as in (b) for a MOP agent with $\beta=0.3$ (left) and time fraction spent in the noisy room (right) show that MOP agents with $\beta> 0$ can either be attracted to the room or repelled depending on $\gamma$.}
\label{fig:fourrooms}
\end{figure}

We find that the MOP agent generates behaviors that can be dubbed goal-directed and curiosity-driven (\href{https://youtu.be/cOZEFtiU0Ig}{Video 1}). First, by storing enough energy in its reservoir, the agent reaches far, entering the four rooms in the long term (Fig. \ref{fig:fourrooms}b, left panel), and visiting every location of the arena except when food gain is small (Fig. \ref{fig:fourrooms}d, blue line). In contrast, the R agent lingers over one of the food sources for most of the time (Fig. \ref{fig:fourrooms}b, middle panel; \href{https://youtu.be/cOZEFtiU0Ig}{Video 1}). Although its $\epsilon$-greedy action selection allows for brief exploration of other rooms, the R agent does not on average visit the whole arena (Fig. \ref{fig:fourrooms}d, orange line). Finally, the random walker dies before it has time to visit a large fraction of the physical space (Fig. \ref{fig:fourrooms}b, right panel). These differences hold for a large range of food gains (Fig. \ref{fig:fourrooms}d).
The MOP agent, while designed to generate variability, is also capable of deterministic behavior: when its energy is low, it moves toward the food sources with little to no variability, a distinct mark of goal-directedness (Fig. \ref{fig:fourrooms}c, corner spots show that only one action is considered by optimal policy).

We next considered a slightly more complex environment where actions in one of the rooms lead to uniformly stochastic transitions to any of the neighboring locations (noisy room --a spatial version of the noisy TV problem \cite{schmidhuber_curious_1991,burda_large-scale_2019}). 
A MOP agent with $\beta > 0$ (see Eq. (\ref{eq_expected_return})) has a preference for stochastic state transitions, and \textit{a priori} it could get attracted and stuck in the noisy room, where actions do not have any predictable effect. 
Indeed, we see that for larger $\beta$, which measures the strength of the state entropy contribution to the agent's objective, the attraction to the noisy room increases (Fig. \ref{fig:fourrooms}e, right panel). 
However, MOP agents also care about future states, and thus getting stuck in regions where energy cannot be predictably obtained is avoided by sufficiently long-sighted agents, as shown by the reduction of the time spent in the noisy room with increasing $\gamma$ (Fig. \ref{fig:fourrooms}e; Supplemental Sec. \ref{sec:4room}).
This shows how MOP agents can tradeoff immediate with future action-state occupancy.

\subsection{Hide and seek in a prey-predator interaction} 

\begin{figure}[t]
\center
\includegraphics[width=\textwidth]{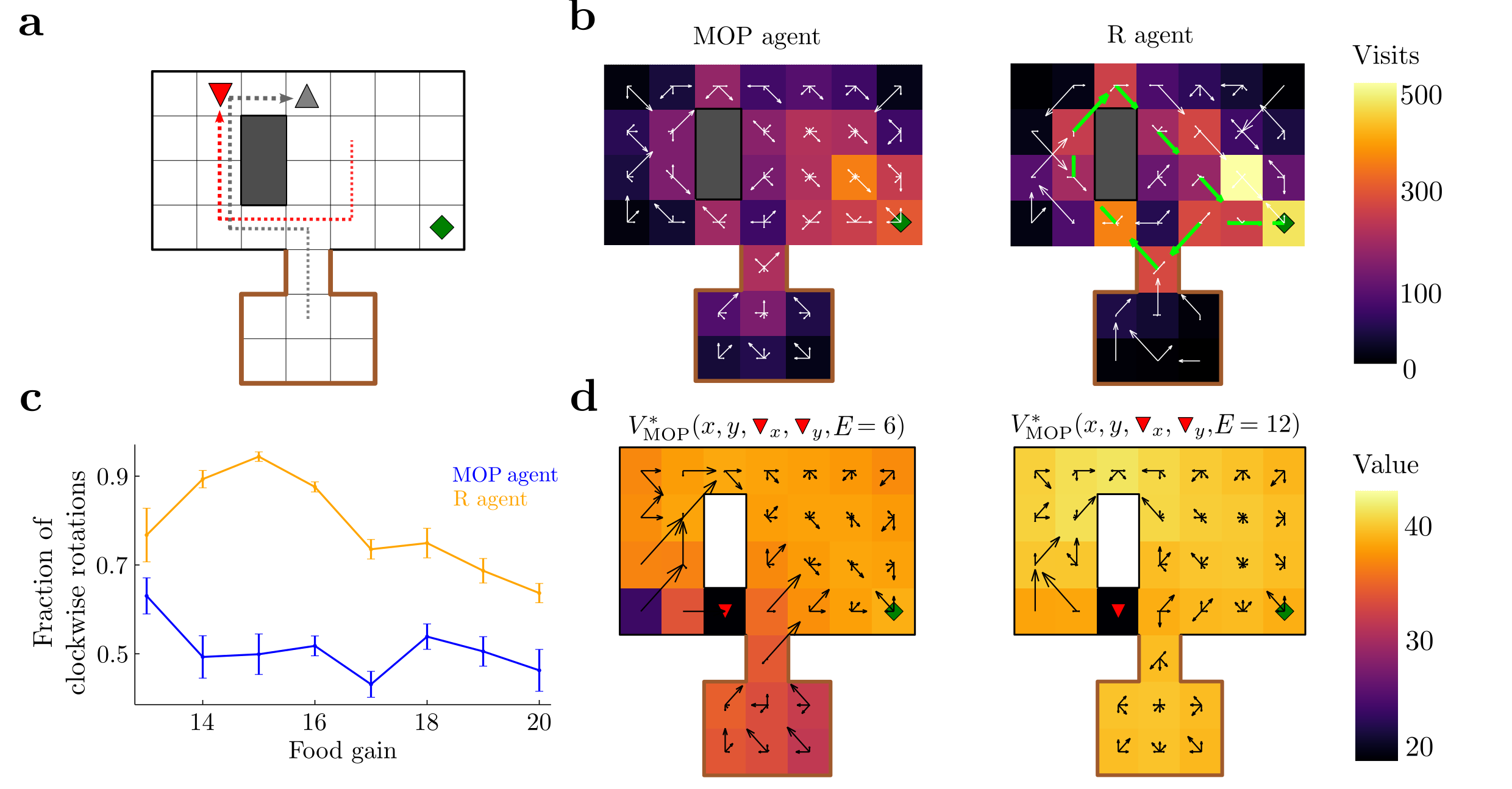}
\caption{Complex hide-and-seek and escaping strategies in a prey-predator example. 
(a) Grid-world arena. The agent has nine available actions when alive and far from walls. There is a small food source in a corner (green diamond). A predator (red, down triangle) is attracted to the agent (gray, up triangle), such that when they are at the same location, the agent dies. The predator cannot enter the locations surrounded by the brown border. Arrows show a clockwise trajectory. (b) Histogram of visited spatial states across episodes for the MOP and R agents. The vector field at each location indicates probability of transition at each location. Green arrows on R agent show major motion directions associated with its dominant clockwise rotation.
(c) Fraction of clockwise rotations (as in panel (a)) to total rotations as a function of food gain, averaged over epochs of 500 timesteps. Error bars are s.e.m. 
(d) Optimal value functions for different energy levels, and same predator position; black arrows indicate optimal policy, as in Fig. \ref{fig:fourrooms}c.}
 \label{fig:cat_mouse}
 \end{figure}

More interesting behaviors arise from MOP in increasingly complex environments. To show this, we next considered a prey and a predator in a grid world with a safe area (a ``home'') and a single food source (Fig. \ref{fig:cat_mouse}a). The prey (a ``mouse'', gray up triangle) is the agent whose behavior is optimized by maximizing future action path entropy, while the predator (a ``cat'', red down triangle) acts passively chasing the prey.
The state of the agent consists of its location and energy level, but it also includes the predator's location being accurately perceived. 
The prey can move as in the previous 4-room grid world and it also has a finite energy reservoir. For simplicity, we only considered a food gain equal to the size of the energy reservoir, such that the agent fully replenishes its reservoir each time it visits the food source.
The predator has the same available actions as the agent and is attracted to it stochastically, i.e. actions that move the predator towards the agent are more probable than those that move it away from it (Supplemental Sec. \ref{sec:details_mouse}). 

MOP generates complex behaviors, not limited to visiting the food source to increase the energy buffer and hide at home. In particular, the agent very often first teases the cat and then performs a clockwise rotation around the obstacle, which forces the cat to chase it around, leaving the food source free for harvest (Fig. \ref{fig:cat_mouse}a, arrows show an example; \href{https://youtu.be/1aaqHi8vtT4}{Video 2}, MOP agent).
Importantly, this behavior is not restricted to clockwise rotations, as the agent performs an almost equal number of counterclockwise rotations to free the food area (Fig. \ref{fig:cat_mouse}c, MOP agent, blue line).
The variability of these rotations in the MOP agent are manifest in the lack of virtually any preferred directionality of movement in the arena at any single position. Indeed, arrows pointing toward several directions indicate that on average the mouse moves following different paths to get to the food source (Fig. \ref{fig:cat_mouse}b, MOP agent). Finally, the optimal value function and optimal policy show that the MOP agent can display deterministic behaviors as a function of internal state as well as distance to the cat (Fig. \ref{fig:cat_mouse}d): for instance, it prefers running away from the cat when energy is large (right), and it risks getting caught to avoid starvation if energy is small (left), both behaviors starkly opposite to stochastic actions.

The behavior of the MOP agent was compared with an R agent that receives a reward of 1 each time it is alive and 0 otherwise.
To promote variable behavior in this agent as well, we implemented an $\epsilon$-greedy action selection (Supplemental Sec. \ref{sec:details_mouse}), where $\epsilon$ was chosen to match the average lifetime of the MOP agent (Supplemental Fig. \ref{fig:survival}b).
The behavior of the R agent was strikingly less variable than that of the MOP agent, spending more time close to the food source (Fig. \ref{fig:cat_mouse}b, R agent).
Most importantly, while the MOP agent performs an almost equal number of clock and counterclockwise rotations, the R agent strongly prefers the clockwise rotations, reaching $90\%$ of all observed rotations (\href{https://youtu.be/qp3dDjefuVI}{Video 3}, R-agent; Fig. \ref{fig:cat_mouse}c, orange line). 
This shows that the R agent mostly exploits only one strategy to survive and displays a smaller behavioral repertoire than the MOP agent.

\subsection{Dancing in an entropy-seeking cartpole}
\begin{figure}[t]
\center
\includegraphics[width=\textwidth]{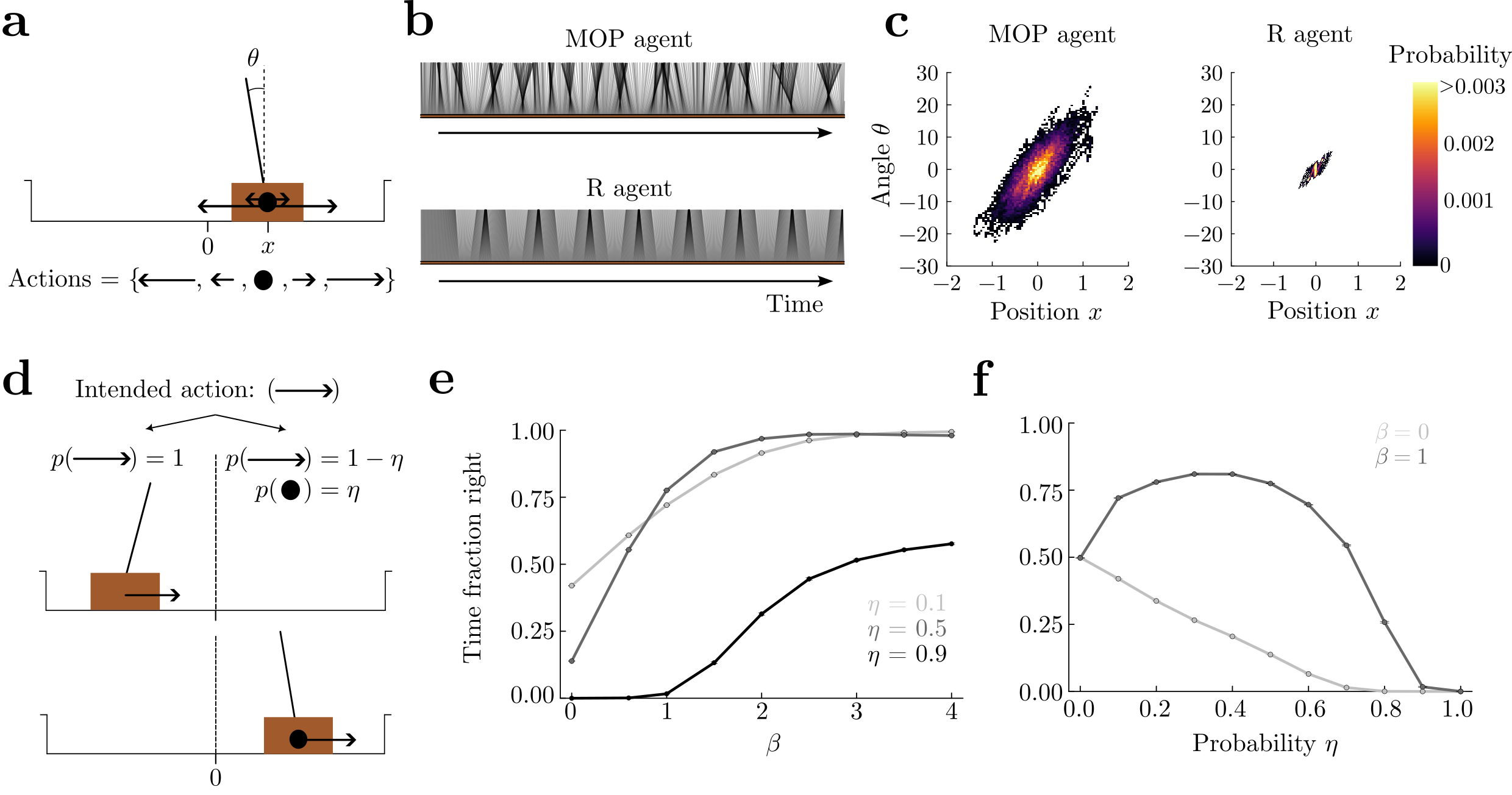}
\caption{Dancing of a MOP cartpole. (a) The cart (brown rectangle) has a pole attached. The cartpole reaches an absorbing state if the magnitude of the angle $\theta$ exceeds $36 \deg$ or its position reaches the borders. There are 5 available actions when alive: a big and a small force to either side (arrows on cartpole), and doing nothing (full circle). (b) Time-shifted snapshots of the pole in the reference frame of the cart as a function of time for the MOP (top) and R (bottom) agents. (c) Position and angle occupation for a $2 \times 10^5$ time step episode. (d) Here, the right half of the arena is stochastic, while the left remains deterministic. In the stochastic half, the intended state transition due to an applied action (force) succeeds with probability $1-\eta$ (and thus zero force is applied with probability $\eta$).
(e) Fraction of time spent on the right half of the arena increases as a function of $\beta$, regardless of the failure probability $\eta$.
(f) The fraction has a non-monotonic behavior as a function of $\eta$ when state entropy is important for the agent ($\beta=1$), highlighting a stochastic resonance behavior.
When the agents do not seek state entropy ($\beta=0$) the fraction of time spent by the agent on the right decreases with the failure probability, and thus they avoid the stochastic right side. 
$\gamma = 0.99$ for panels (e,f).
}  

 \label{fig:cartpole}
 \end{figure}
In the previous examples, complex behaviors emerge as a consequence of the presence of obstacles, predators and limited food sources, but the actual dynamics of the agents are very coarse-grained. Here, we considered a system with physically realistic dynamics, the balancing cartpole 
\cite{barto_neuronlike_1983,florian_correct_2007}, composed of a moving cart with an attached pole free to rotate (Fig. \ref{fig:cartpole}a). The cartpole is assumed to reach an absorbing state when either it hits a border, or when the pole angle exceeds $36$ degrees. Thus, we consider a broad range of angles that makes the agents reach a larger state space than in standard settings \cite{brockman_openai_2016}. We discretized the state space and used a linear interpolation to solve for the optimal value function in Eq. (\ref{eq_bellman_opt_eq}), and to implement the optimal policy in Eq. (\ref{eq_pi_opt_m}), (Supplemental Sec. \ref{sec:details_cartpole}). 
The MOP agent widely occupies the horizontal position, and more strikingly it produces a wide variety of pole angles,
constantly swinging sideways as if it were dancing (\href{https://youtu.be/XwJB1-Vu394}{Video 4}, MOP agent; Fig. \ref{fig:cartpole}b,c). 

We compared the behavior of a MOP agent with that of an R agent that receives a reward of 1 for being alive and 0 otherwise. The R agent gets this reward regardless of the pole angle and cart position within the allowed broad ranges, so that behaviors of the MOP and R agents can be better compared without explicitly favoring in any of them any specific behavior, such as the upright pole position.
As expected, the R agent maintains the pole close to the balanced position throughout most of a long episode (Fig. \ref{fig:cartpole}b, bottom), because it is the furthest to the absorbing states and thus the safest. 
Therefore, the R agent produces very little behavioral variability (Fig. \ref{fig:cartpole}c, right panel) and no movement that could be dubbed `dancing' (\href{https://youtu.be/XwJB1-Vu394}{Video 4}, R agent). Although both MOP and R agents use a similar strategy which keeps the pole pointing towards the center for substantial amounts of time (Fig. \ref{fig:cartpole}c, positive angles correlate with positive positions in both panels), the behavior of the R agent is qualitatively different, and is best described as a bang-bang sort of control for which the angle is kept very close to zero while the cart is allowed to travel and oscillate around the origin, which is more apparent in the actual paths of the agent (see trajectories in phase space in \href{https://youtu.be/B8QjBdVIM_o}{Video 5}).
We also find that the R agent does not display much variability in state space even after using an $\epsilon$-greedy action selection (Supplemental Fig. \ref{fig:cartpoleeps}, \href{https://youtu.be/BBiJhxPfkjw}{Video 6}), with $\epsilon$ chosen to match average lifetimes between agents (Supplemental Fig. \ref{fig:survival}c). This result showcases that the MOP agent exhibits the most appropriate sort of variability for a given average lifetime.

We finally introduced a slight variation to the environment, where the right half of the arena has stochastic state transitions. Here, when agents choose an action (force) to be executed, a state transition in the desired direction occurs with probability $1-\eta$, and a transition corresponding to zero force occurs with probability $\eta$ (Fig. \ref{fig:cartpole}d).
Therefore, a MOP agent that seeks state entropy ($\beta>0$) will show a preference for the right side, where there is in principle higher state entropy resulting from the stochastic transitions over more successor states than on the left side.
Indeed, we find that MOP agents spend more time on the right side as $\beta$ increases, regardless of the probability $\eta$ (Fig. \ref{fig:cartpole}e). 
For fixed $\gamma$, spending more time on the right side can bring the life expectancy to decrease significantly depending on $\beta$ and $\eta$ (Supplemental Fig. \ref{fig:survival} d-e). 
Interestingly, for $\beta > 0$ there is an optimal value of the noise $\eta$ that maximizes the fraction of time spent on the right side (Fig. \ref{fig:cartpole}f), which is a form of stochastic resonance. Therefore, for different $\beta$, qualitatively different behaviors emerge as a function of the noise level $\eta$.

\subsection{MOP agents can also seek entropy of others}

\begin{figure}[t]
\center
\includegraphics[width=
\textwidth]{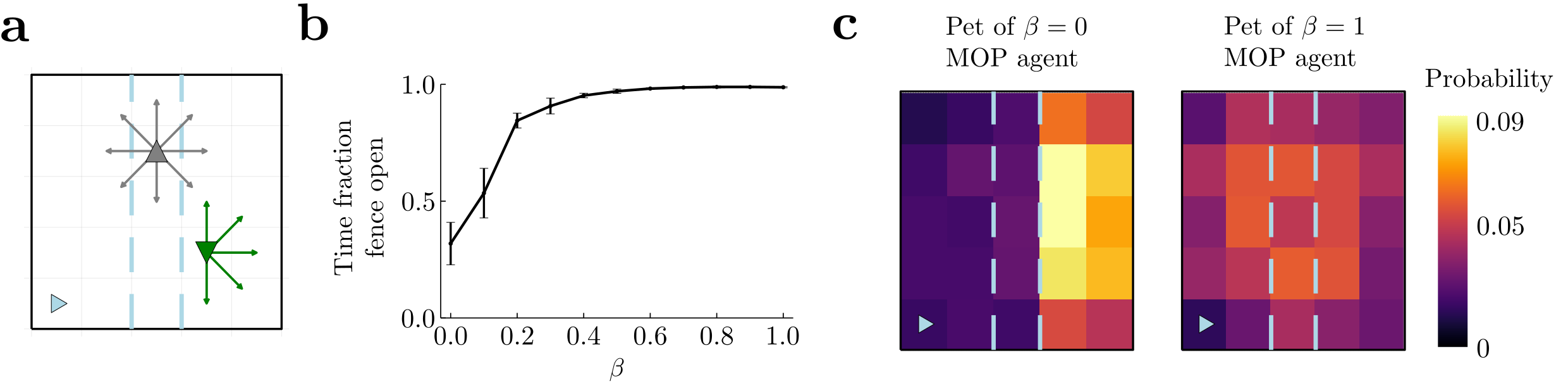} 
\caption{Modelling altruism through an optimal tradeoff between own action entropy and other's state entropy.
(a) An agent (gray up triangle) has access to nine movement actions (gray arrows and doing nothing), and open or close a fence (dashed blue lines). This fence does not affect its movements. A pet (green, down triangle) has access to the same actions, and chooses one randomly at each timestep, but is constrained by the fence when closed. Pet location is part of the state of the agent. 
(b) As $\beta$ in Eq. (\ref{eq_expected_return}) is increased, the agent tends to leave the fence open for a larger fraction of time. This helps its pet reach other parts of the arena. Error bars correspond to s.e.m. (c) Occupation heatmaps for 2000 timestep-episodes for $\beta = 0$ (left) and $\beta=1$ (right). In all cases $\alpha=1$.}
 \label{fig:friendly_cat}
 \end{figure}

Next, we considered an example where an agent seeks to occupy path space, which includes another agent's location as well as its own. The agent can freely move (Fig. \ref{fig:friendly_cat}a; grey triangle) and open or close a fence by pressing a lever in a corner (blue triangle).
The pet of the agent (green triangle) can freely move if the fence is open, but when the fence is closed the pet is confined to move in the region where it is currently located. The pet moves randomly at each step, but its available actions are restricted by its available space (Supplemental Sec. \ref{sec:details_pet}).

To maximize action-state path entropy, 
the agent ought to trade off the state entropy resulting from letting the pet free with the action entropy resulting from using the open-close action when visiting the lever location.
The optimal tradeoff depends on the relative strength of action and state entropies. 
In fact, when state entropy weighs as much as action entropy ($\alpha=\beta=1$), the fraction of time that the agent leaves the fence open is close to $1$ (rightmost point in Fig. \ref{fig:friendly_cat}b) so that the pet is free to move (Fig. \ref{fig:friendly_cat}c, right panel; $\beta = 1$ MOP agent). However, when the state entropy has zero weight ($\alpha=1, \beta=0$), the fraction of time that the fence remains open is close to $0.5$ (leftmost point in Fig. \ref{fig:friendly_cat}b) and the pet remains confined to the right side for most of the time
(Fig. \ref{fig:friendly_cat}c, left panel; $\beta = 0$ MOP agent), the region where it was initially placed.
As a function of $\beta$ the fraction of time the fence is open increases. 
Therefore, the agent gives more freedom to its pet, as measured by the pet's state entropy, by curtailing its own action freedom, as measured by action entropy, thus becoming more "altruistic".

\subsection{MOP compared to other reward-free approaches}

One important question is how MOP compares to other reward-free, motivation-driven theories of behavior. Here we focus on two popular approaches: empowerment and the free energy principle. 
In empowerment (MPOW) \cite{klyubin_empowerment_2005,jung_empowerment_2011,still_information-theoretic_2012,mohamed_variational_2015} agents maximize the mutual information between $n$-step actions and the successor states resulting from them \cite{klyubin_empowerment_2005,blahut_computation_1972}, 
a measure of their capability to perform diverse courses of actions with predictable consequences. 
MPOW formulates behavior as greedy maximization of empowerment \cite{klyubin_empowerment_2005,jung_empowerment_2011}, such that agents move to accessible states with the largest empowerment (maximal mutual information), and stay there with high probability. 

We applied MPOW to the gridworld and cartpole environments (Fig. \ref{fig:MPOW_EFE}). In the gridworld, MPOW agents (5-step MPOW, see Supplemental Sec. \ref{sec:empowerment_supplemental}) prefer states from where they can reach many distinct states, such as the middle of a room. However, due to energetic constraints, they also gravitate towards the food source when energy is low, and they alternate between these two locations \textit{ad nauseam} (Fig. \ref{fig:MPOW_EFE}a, middle; \href{https://youtu.be/WDcGfsjKlcI}{Video 7}). In the cartpole, MPOW agents (3-step MPOW \cite{jung_empowerment_2011}, see Supplemental Sec. \ref{sec:empowerment_supplemental})
favour the upright position because, being an unstable fixed point, it is the state with highest empowerment, as previously reported \cite{klyubin_keep_2008,jung_empowerment_2011}. Given the unstable equilibrium, the MPOW agent gets close to it but needs to continuously adjust its actions when greedily maximizing empowerment (Fig. \ref{fig:MPOW_EFE}b, middle; \href{https://youtu.be/8SlPWjoobwo}{Video 8}). The paths traversed by MPOW agents in state space are highly predictable, and they are similar to the ones of the R agent (see Fig. \ref{fig:cartpole}c). The only source of stochasticity comes from the algorithm, which approximately calculates empowerment, and thus a more precise estimation of empowerment leads to even less variability.

\begin{figure}[t]
\center
\includegraphics[width=
\textwidth]{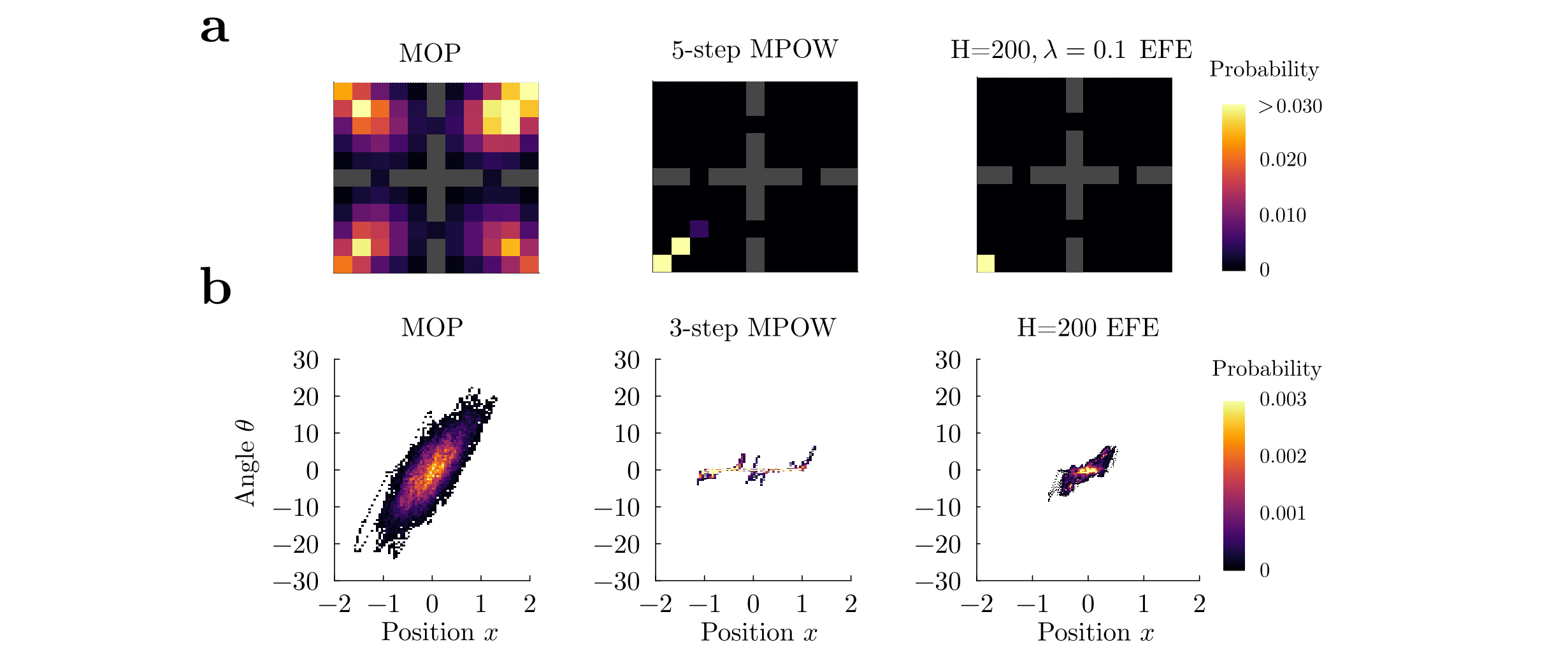} 
\caption{Empowerment (MPOW) and Free Energy Principle (FEP) lack robust occupation of action-states. (a) In the grid-world environment, MPOW and expected free energy (EFE) only visit a restricted portion of the arena. Initial position was the center of a room $(x,y) = (3,3)$.
(b) In the cartpole environment, both MPOW and EFE shy away from large angles, producing a limited repertoire of predictable behaviors.}
 \label{fig:MPOW_EFE}
 \end{figure}

In the free energy principle (FEP), agents seek to minimize the minus log probability, called surprise, of a subset of desired states via the minimization of an upper bound, called free energy. This minimization reduces behavioral richness by making a set of desired (homeostatic) states highly likely \cite{friston_free_2006,buckley_free_2017}, rendering this approach almost opposite to MOP. 
In a recent MDP formalization, FEP agents aim to minimize the (future) expected free energy (EFE) \cite{da_costa_reward_2023}, which equals the future cumulative KL divergence between the probability of states and the desired (target) probability of those states (see Supplemental Sec. \ref{sec:active_inference} for details). Even though this objective contains the standard exploration entropy term on state transitions \cite{buckley_free_2017,tschantz_reinforcement_2020}, we prove that the optimal policy is deterministic (see Supplemental Sec. \ref{sec:active_inference}). 

As a consequence, we find that in both the gridworld and cartpole environments, the behavior of the EFE agent (receding horizon $H=200$) is much less variable than the MOP agent in general (Fig. \ref{fig:MPOW_EFE}a, right panel for the gridworld, \href{https://youtu.be/WDcGfsjKlcI}{Video 7}; and b, right panel, for the cartpole, \href{https://youtu.be/8SlPWjoobwo}{Video 8}). The only source of action variability in the EFE agent is due to the degeneracy of the expected free energy, and thus behavior collapses to a deterministic policy as soon as the target distribution is not perfectly uniform (see Supplemental Sec. \ref{sec:supplemental_activeinference_details} for details). We finally prove that under discounted infinite horizon, and assuming a deterministic environment, the EFE agent is equivalent to a classical reward maximizer
agent with reward $R=1$ for all non-absorbing states and $R=0$ for the absorbing states (Supplemental Sec. \ref{sec:active_inference}). 
In conclusion, MOP generates much more variable behaviors than MPOW and FEP.


\subsection{MOP in continuous and large action-state spaces}

\begin{figure}[t!]
\center
\includegraphics[width=
\textwidth]{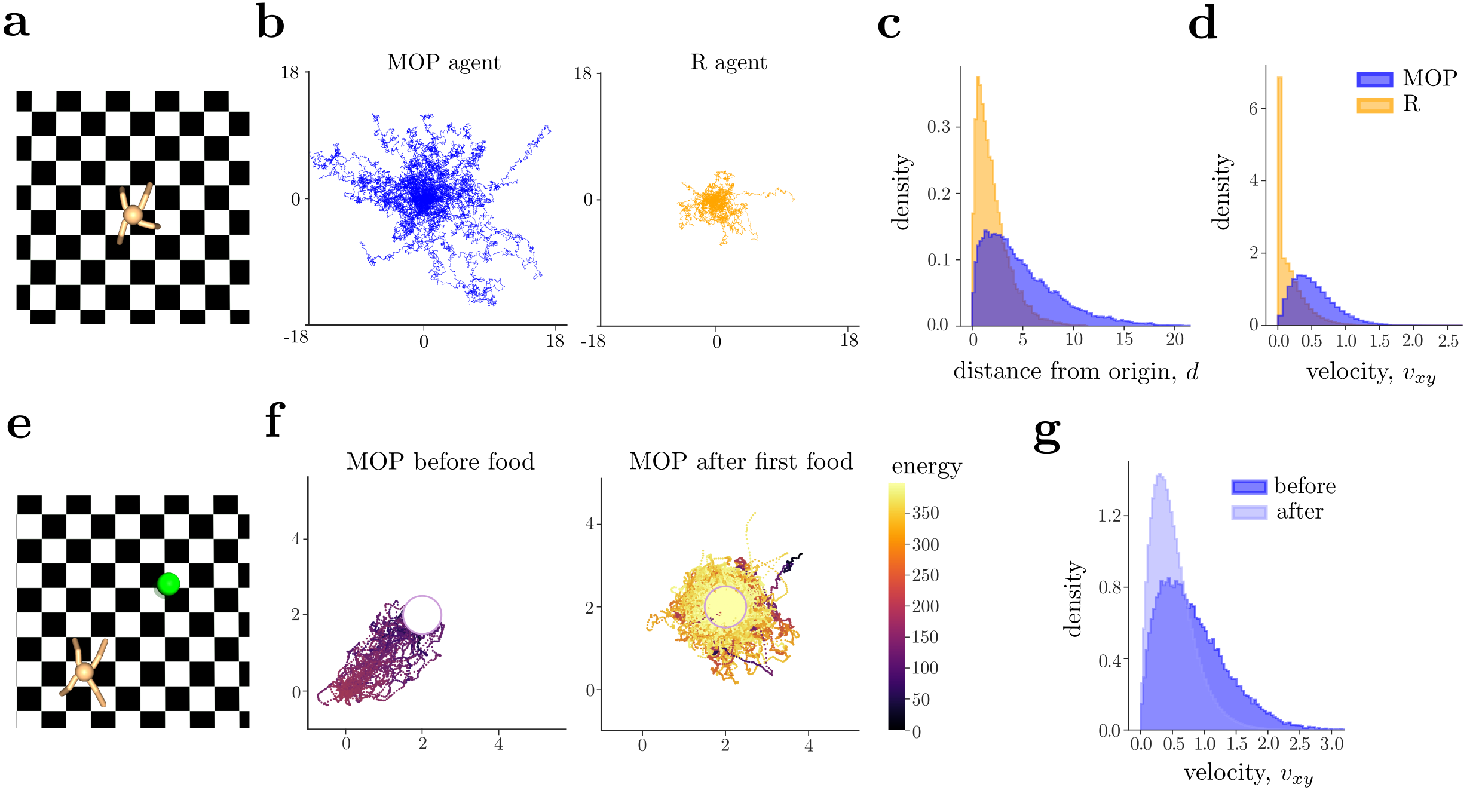} 
\caption{MOP in high-dimensional states generates variable and goal-directed behaviors. (a) The quadruped environment, adapted from Gymnasium, serves as the testing environment. The $x,y$ dimensions are unbounded. (b) Trajectories of the center of mass of the torso of the MOP (left panel) and R (right) agents. MOP occupies more space for approximately the same survival time (see Supplemental Fig. \ref{fig:supplemental_ant}a). (c-d) Distribution of the planar distance $d$ from the origin (c) and planar speed $v_{xy}$ (d) for MOP (blue) and R (yellow) agents.
(e) In a new environment, a food source (green ball) is available so that the MOP agent can replenish its internal energy to avoid starvation.
(f) Trajectories of the MOP agent before (left) and after (right) getting to the food source. 
Colormap defined by the energy level of the agent. 
(g) Distribution of the planar speed showcasing changes before (dark blue) and after (light blue) the MOP agent reaches the food source for the first time. Distributions computed only on the tests where the quadruped finds the food source.}
 \label{fig:ant}
 \end{figure}

The examples so far can be solved exactly with numerical methods, without relying on function approximation of the value function or the policy, which could obscure the richness of the resulting behaviors. 
However, one important question is whether our approach scales up to large continuous action-state spaces where no exact solutions are available. 
To show that MOP generates rich behaviors even in high-dimensional agents, we used a quadruped from Gymnasium \cite{towers_gymnasium_2023} without imposing any explicit fine-tuned reward function (Fig. \ref{fig:ant}a). 
The only externally imposed conditions are the absorbing states, which are reached when either the agent falls (given by the torso touching the ground), or the torso reaches a maximum height \cite{towers_gymnasium_2023}.

We first trained the MOP agent by approximating the state-value function, Eq. (\ref{eq_v_opt_m}), using the soft-actor critic (SAC) architecture \cite{haarnoja_soft_2018} with zero rewards, which corresponds to the case $\alpha=1$ and $\beta=0$. 
The MOP agent learns to stabilize itself and walk around, sometimes jumping, spinning and moving up and down the legs, without any instructions to do so (\href{https://youtu.be/DDOoyAqYFEk}{Video 9)}. The MOP agent exhibits variable and long excursions over state space (Fig. \ref{fig:ant}b,c blue) and displays a broad distribution of speeds (Fig. \ref{fig:ant}d, blue).
We compared the MOP agent with an R agent that obtains a reward of $R=1$ whenever it is alive and $R=0$ when it reaches an absorbing state. As before, we add variability to the R agent with an $\epsilon$-greedy action selection, adjusting $\epsilon$ so that the average lifetime of the R agent matched that of the MOP agent (Supplemental Fig. \ref{fig:supplemental_ant}a).
In contrast to the MOP agent, the R agents exhibit much shorter excursions (Fig. \ref{fig:ant}b,c yellow) and a velocity distribution that peaks around zero, indicating prolonged periods spent with no translational movement (Fig. \ref{fig:ant}d, yellow). 
When visually compared, the behavior for MOP and R agents shows stark differences (\href{https://youtu.be/DDOoyAqYFEk}{Video 9)}.

While the MOP agent elicits variable behaviors, it is also capable of generating deterministic, goal-directed behaviors when needed. To show this, we added a food source in the arena and extended the state of the agent with its internal energy. Now the agent can also die of starvation when the internal energy hits zero (Fig. \ref{fig:ant}e).
As expected, when the initial location of the MOP quadruped is far from the food source, it directly moves to the food source to avoid dying from starvation (Fig. \ref{fig:ant}f).  
After the food source is reached for the first time, the MOP quadruped generates random excursions away from the food source. 
During these two phases, the agent displays very different speed distributions (Fig. \ref{fig:ant}g), showing also quantitative differences in the way it moves (see a comparison with the R agent in Supplemental Fig. \ref{fig:supplemental_ant}, and \href{https://youtu.be/Zwx1029e5eg}{Video 10}). 

Finally, we modified the environment by adding state transition noise of various magnitudes in one half of the arena ($x>0$), while the other half remained deterministic. We find that the agent's behavior is modulated by $\beta$, which controls the preference of state transition entropy (see details in Supplemental Sec. \ref{sec:supplemental_ant}). As expected, for fixed $\alpha$ and positive noise magnitude, MOP agents show increasing preference toward the noisy side as $\beta$ increases (Supplemental Fig. \ref{fig:ant_beta}). However, as noise magnitude increases, and for fixed $\beta$, MOP agents tend to avoid the noisy side to prevent them from falling. This shows that MOP agents can exhibit approach and avoidance behaviors depending on the environment's stochasticity and their $\beta$ hyperparameter.

\section{Discussion}
Often, the success of agents in nature is not measured by the amount of reward obtained, but by their ability to expand in state space and perform complex behaviors. Here we have proposed that a major goal of intelligence is to `occupy path space'.
Extrinsic rewards are thus the means to move and occupy action-state path space, not the goal of behavior. In an MDP setting, we have shown that the intuitive notion of path occupancy is captured by future action-state path entropy, and we have proposed that behavior is driven by the maximization of this intrinsic goal --the maximum occupancy principle (MOP). 
We have solved the associated Bellman equation and provided a convergent iterative map to determine the optimal policy.

In several discrete and continuous state examples we have shown that MOP, along with the agent's constraints and dynamics, leads to complex behaviors that are not observed in other simple reward maximizing agents. Quick filling of physical space by a moving agent, hide-and-seek behavior and variable escaping routes in a predator-prey example, dancing in a realistic cartpole dynamical system, altruistic behavior in an agent-and-pet duo and successful, vigorous movement in a high-dimensional quadruped are all behaviors that strike as being playful, curiosity-driven and energetic. 
To the human eye, these behaviors look genuinely goal-directed, like approaching to the food source when the energy level is low or escaping from the cat when it gets close to the mouse (see Figs. \ref{fig:fourrooms}c and \ref{fig:cat_mouse}d). 
Although MOP agents do not have any extrinsically designed goal, like eating or escaping, they generate these deterministic, goal-directed behaviors whenever necessary so that they can keep moving in the future and maximize future path action-state entropy (see Supplemental Sec. \ref{sec:maxentRLgoal}). These results show that the presence of internal states (e.g. energy) and absorbing states (e.g, having zero energy or being eaten) are critical for generating interesting behaviors, as getting close to different types of absorbing states triggers qualitatively different behaviors. This capability of adapting variability depending on internal states has been overlooked in the literature and is essential to obtaining the goal-directed behaviors we have shown here.
In parallel, when basic energy and safety conditions are met, behaviors are lively and somewhat risky, like when the cartpole gets close to the borders of the arena, and therefore our approach can lead to novel ways of thinking about risk--seeking behaviors \cite{fei_risk-sensitive_2020}.

A related set of algorithms, known as empowerment, have also proposed using reward-free objectives as the goal of behavior \cite{klyubin_empowerment_2005,jung_empowerment_2011,mohamed_variational_2015}. In this approach, the mutual information between a sequence of actions and the final state is maximized.
This makes empowerment agents prefer states where actions lead to large and predictable changes, such as unstable fixed points \cite{jung_empowerment_2011}.
We have shown that one drawback is that empowered agents tend to remain close to those states without producing diverse behavioral repertoires (see Fig. \ref{fig:MPOW_EFE}b and \href{https://youtu.be/8SlPWjoobwo}{Video 8}), as it also happens in causal entropy approaches \cite{wissner-gross_causal_2013}.  
Another difference is that empowerment is not additive over paths because the mutual information of a path of actions with the path of states is not the sum of the per-step mutual information, and thus it cannot be formalized as a cumulative per-step objective (Supplemental Sec. \ref{sec:non-additivity-MI}) \cite{jung_empowerment_2011,leibfried_unified_2019,mohamed_variational_2015,volpi_goal-directed_nodate}, in contrast to action-state path entropy. 
We note, however, that an approximation to empowerment having the desired additive property could be obtained from our framework by putting $\beta <0$ in Eq. (\ref{eq_expected_return}), such that more predictable state transitions are preferred. 
Similarly to empowerment, we have also shown that agents following the free energy principle \cite{friston_free_2006,da_costa_reward_2023} collapse behavior to deterministic policies in known environments (see Fig. \ref{fig:MPOW_EFE}b and \href{https://youtu.be/8SlPWjoobwo}{Video 8}).
Other reward-free RL settings and pure exploration objectives have been proposed in the past \cite{hazan_provably_2019,lee_efficient_2019,jin_reward-free_2020,zhang_exploration_2021,mutti_intrinsically-motivated_2020,mutti_task-agnostic_2021,pathak_curiosity-driven_2017,eysenbach_maximum_2021}, but this body of work typically investigates how to efficiently sample MDPs to construct near-optimal policies when reward functions are introduced in the exploitation phase. More importantly, this work differs from ours in that the goal-directedness that MOP displays entails behavioral variability at its core, even in known environments (see examples above). Finally, other overlapping reward-free approaches focus on the unsupervised discovery of skills, by encouraging diversity \cite{gregor_variational_2016,eysenbach_diversity_2018,sharma_dynamics-aware_2020,park_controllability-aware_2023}. While the motivation is similar, they focus on skill-conditioned policies, whereas our work demonstrates that complex sequences of behaviors are possible working from the primitive actions of agents, although a possible future avenue for MOP is to apply it to temporally extended actions  \cite{sutton_between_1999}. In addition, these works define tasks based on extrinsic rewards, whereas we have shown that internal state signals are sufficient to let agents define sub-tasks autonomously.

Our approach is conceptually different as well to hybrid approaches that combine extrinsic rewards with action entropy or KL regularization terms \cite{ziebart_modeling_2010,todorov_efficient_2009,schulman_equivalence_2017,hausman_learning_2018,grau-moya_planning_2016} for two main reasons. First, entropy seeking behavior does not pursue any form of extrinsic reward maximization. But most importantly, using KL-regularization using a default policy $\pi_0(a|s)$ in our framework would be self-defeating. 
This is because the absolute action entropy terms $\mathcal{H}(A|s)$ in the expected return in Eq. (\ref{eq_expected_return}) favor visiting states where a large set of immediate and future action-states are accessible. 
In contrast, using relative action entropy (KL) precludes this effect by normalizing the number of accessible actions, as we have shown above. 
Additionally, minimizing the KL divergence with a uniform default policy and without extrinsic rewards leads to an optimal policy that is uniform regardless of the presence of absorbing states,
equivalent to a random walker, which shows that a pure KL objective does not lead to interesting behaviors (Supplemental Sec. \ref{sec:supplemental_KL}, Supplemental Fig. \ref{fig:comparisonKL}).
The idea of having a variable number of actions that depend on the state is consistent with the concept of affordance \cite{gibson_ecological_2014}. While we do not address the question of how agents get the information about the available actions, an option would be to use the notion of affordances as actions \cite{khetarpal_what_2020}.
Secondly, while previous work has studied the performance benefits of either action \cite{haarnoja_soft_2018}, state \cite{hazan_provably_2019,neu_unified_2017} or equally weighted action-state \cite{peters_relative_2010,tishby_information_2011} steady-state entropies, our work proposes mixing them arbitrarily through path entropy, leading to a more general theory without any loss in mathematical tractability \cite{grytskyy_general_2023}. 
This arbitrary weight mixing lets us model more diverse phenomena. For example, for the right combination of state entropy weight $\beta$, and lookahead horizon, controlled by $\gamma$, MOP agents could get stuck in a noisy TV, consistent with the observation that humans have a preference for noisy TVs under particular conditions \cite{modirshanechi_curse_2023}. However, it can also capture the avoidance of noisy TVs for sufficiently-long-sighted agents (see Fig. \ref{fig:fourrooms}e).

We have also shown that MOP is scalable to high-dimensional problems and when the state-transition matrix is unknown, using the soft-actor critic architecture \cite{haarnoja_soft_2019} to approximate the optimal policy prescribed by MOP. 
Nevertheless, several steps remain to have a more complete MOP theory with learning. 
Previous related attempts have introduced Z-learning \cite{todorov_linearly-solvable_2006,todorov_efficient_2009} and G-learning \cite{fox_taming_2015} using off-policy methods, so our results could be extended to learning following similar lines. Other possibilities are using transition estimators using counts or pseudo-counts \cite{bellemare_unifying_2016}, or hashing \cite{tang__2017}, for the learning of the transition matrices.
One potential advantage of our framework is that, as entropy-seeking behavior obviates extrinsic rewards, those rewards do not need to be learned and optimized, and thus the learning problem reduces to transition matrices learning. 
In addition, modeling and injecting prior information could be particularly simple in our setting in view that intrinsic entropy rewards can be easily bounded before the learning process if action space is known. Therefore, initializing the state-value function to the lower or upper bounds of the action-state path entropy could naturally model pessimism or optimism during learning, respectively.  

All in all, we have introduced MOP as a novel theory of behavior, which promises new ways of understanding goal-directedness without reward maximization, and that can be applied to artificial agents to discover by themselves ways of surviving and occupying action-state space.

\section*{Acknowledgments}
This work is supported by the Howard Hughes Medical Institute (HHMI, ref 55008742), ICREA Academia 2022 and MINECO (Spain; BFU2017-85936-P) to R.M.-B, and MINECO/ESF (Spain; PRE2018-084757) to J.R.-R. We thank Jose Apesteguia, Luca Bonatti, Ignasi Cos, and Benjamin Hayden for very useful comments.

\section*{Code and data availability}

The code to generate the results and various figures is available as Python and Julia code along with guided notebooks to reproduce the figures at this public  \href{https://github.com/jorgeerrz/occupancy_max_paper}{GitHub repository}: 
\linebreak \verb|https://github.com/jorgeerrz/occupancy_max_paper|. All data are in this public repository, except for the specific data for the ant experiment, which are available upon request.

\typeout{}
\bibliographystyle{unsrtnat}
\bibliography{references_z}

\pagebreak

\appendix

\section{Appendix}

\subsection{Entropy measures the occupancy of action-state paths}
\label{sec:entropy-measures}

In this section, we show that entropy is the only measure of action-state path occupancy that obeys some basic intuitive notions of occupancy. We first list the intuitive conditions in mathematical form, present the main theorem and then discuss some implications through some corollaries. 

We consider a time-homogeneous Markov decision process with finite state set $\mathcal{S}$ and finite action set $\mathcal{A} (s)$ for every state $s \in \mathcal{S}$.
Henceforth, the action-state $x_j=(a_j,s_j)$ is any joint pair of one available action $a_j$ and one possible successor state $s_j$ that results from making that action under policy $\pi \equiv \{ \pi(a|s) \} $ from the action-state $x_i=(a_i,s_i)$. 
By assumption, the availability of action $a_j$ depends on the previous state $s_i$ alone, not on $a_i$. 
Thus, the transition probability from $x_i$ to $x_j$ in one time step is $p_{ij} = \pi(a_j|s_i) p(s_j|s_i,a_j)$, where $p(s_j|s_i,a_i)$ is the conditional probability of transitioning from state $s_i$ to $s_j$ given that action $a_j$ is performed.
Although there is no dependence of the previous action $a_i$ on this transition probability, it is notationally convenient to define transitions between action-states.
We conceive of rational agents as maximizing future action-state path occupancy. Any measure of occupancy should obey the intuitive Conditions {\em 1-4} listed below.

\vspace{0.2cm}
\noindent
{\bf Intuitive Conditions for a measure of action-state occupancy:}
\begin{enumerate}
	{\em
	\item Occupancy gain of action-state $x_j$ from $x_i$ is a function of the transition probability $p_{ij}$, $C(p_{ij})$
	\item Performing a low probability transition leads to a higher occupancy gain than performing a high probability transition, that is, $C(p_{ij})$ decreases with $p_{ij}$
	\item The first order derivative $C'(p_{ij})$ is continuous for $p_{ij} \in (0,1)$
	\item (Definition: the action-state occupancy of a one-step path from action-state $x_i$ is the expectation over occupancy gains of the immediate successor action-states, $C^{(1)}_i \equiv \sum_j p_{ij} C(p_{ij})$)
	
	The action-state occupancy of a two-steps path is additive, 
	
	$C_i^{(2)} \equiv \sum_{jk} p_{ij} p_{jk} C(p_{ij} p_{jk}) = C_i^{(1)} + \sum_{j} p_{ij} C_j^{(1)}$ 
	
	for any choice of the $p_{ij}$ and initial $x_i$
	}
\end{enumerate}

Condition {\em 1} simply states that occupancy gain from an initial action-state is defined over the transition probabilities to successor action-states in a sample space. 
Condition {\em 2} implies that performing a low probability transition leads to a higher occupancy of the successor states than performing a high probability transition. This is because performing a rare transition allows the agent to occupy a space that was left initially unoccupied. 
Condition {\em 3} imposes smoothness of the measure. 

In Condition {\em 4} we have defined the occupancy of the successor action-states (one-step paths) in the Markov chain as the expected occupancy gain. 
Condition {\em 4} is the central property, and it imposes that the occupancy of action-states paths with two steps can be broken down into a sum of the occupancies of action-states at each time step. Note that the action-state path occupancy  can be written as

\be
\nonumber
C_i^{(2)} \equiv \sum_{jk} p_{ij} p_{jk} C(p_{ij} p_{jk}) = \sum_{j} p_{ij} C(p_{ij}) + \sum_{jk} p_{ij} p_{jk} C(p_{jk}) = \sum_{jk} p_{ij} p_{jk} \left( C(p_{ij}) + C(p_{jk}) \right)
,
\ee

\noindent
which imposes a strong condition on the function $C(p)$.
Note also that the sum $\sum_{jk} p_{ij} p_{jk} C(p_{ij} p_{jk})$ extends the notion of action-state to a path of two consecutive action-states, each path having probability $p_{ij} p_{jk}$ due to the (time-homogeneous) Markov property. The last equality is an identity.   
While here we consider paths of length equal to $2$, further below we show that there is no difference in imposing additivity to paths of any fixed or random length (Corollary \ref{corollary2}).

\begin{theorem}[]
	$C(p)=-k \ln p$ with $k>0$ is the only function that satisfies Conditions {\em 1-4}
	\label{th1}
\end{theorem}

\begin{corollary}[]
	The entropy $C_i^{(1)} = - k \sum_j p_{ij} \ln p_{ij}$ is the only measure of action-state occupancy of successor action-states $x_j$ from $x_i$ with transition probabilities $p_{ij}$ consistent with Conditions {\em 1-4}.
\end{corollary}

\begin{proof}
Put $p_{1,1}=1$ and $p_{1,j}=0$ for $j \neq 1$. Then, Condition {\em 4} reads $C(1) = C(1) + C(1)$ when the initial action-state is $x_1$, which implies $C(1)=0$.

Now, take a Markov chain with $p_{0,0}=1$, $p_{1,0}=1-t>0$, $p_{1,2}=t>0$, $p_{2,0}=p_{2,1}=0$, $p_{2,j}=1/n$ for $j=3,...,n+2$ and $n > 0$, and $p_{k,0}=1$ for $k=3,...,n+2$. In this chain, the state $0$ is absorbing and all others are transient (here action-states are simply referred to as states). Starting from state $1$, transition to the transient state $2$ happens with probability $t$ and to the absorbing state $0$ with probability $1-t$. From state $2$ a transition to states $j=3,...,n+2$ happens with equal probability. From any of those states, a deterministic transition to $0$ ensues.
(These last transitions can only happen in the third time step, and although it will be relevant later on, it is no used in the current proof, which only uses additivity on paths of length two.) 
Then, Condition {\em 4} with initial state $1$ reads  
$t C(t/n) + (1-t) C(1-t) = t C(t) + (1-t) C(1-t) + t C(1/n) + (1-t) C(1)$,
and hence 
$C(t/n) = C(t) + C(1/n)$
for any $0 < t < 1$ and integer $n>0$. 
By Condition {\em 3} and taking derivative with respect to $t$ in both sides, we obtain 
$C'(t/n) = n C'(t)$,
and multiplying in both sides by $t$ we obtain
$ \frac{t}{n} C'(\frac{t}{n}) = t C'(t)$.
By replacing $t$ with $nt$, we get 
$ t C'(t) = nt C'(nt)$, provided that $nt<1$.

We will now show that $t C'(t)$ is constant. In the last equation replace $t$ by $t/m$ by integer $m>0$ to get the last equivalence in $t C'(t) = \frac{t}{m} C'(\frac{t}{m}) = \frac{n}{m}t C'(\frac{n}{m}t)$ (the first equivalence is obvious). 
These equivalences are valid for positive $t<1$ and $\frac{n}{m}t<1$.
Let $0< s < 1$ and $n = \left \lfloor{ms/t}\right \rfloor$ be the largest integer smaller than $ms/t$.
Therefore, as $m$ increases $\frac{n}{m}t<1$ and approaches $s$ as close as desired. By Condition {\em 3} the function $xC'(x)$ is continuous, and therefore $\lim_{m \rightarrow \infty} \frac{n}{m}t C'(\frac{n}{m}t) = sC'(s)$. The basic idea is that we can first compress $t$ as much as needed by the integer factor $m$ and then expand it by the integer factor $n$ so that $nt/m$ is as close as desired to $s$. This shows that $sC'(s) = tC'(t)$ for $s,t \in (0,1)$, and therefore $tC'(t)$ is constant.

Assume that $tC'(t)=-k$. Then, by integrating we obtain
$C(t) = -k \ln t + a$,
but $a=0$ due to $C(1)=0$, and $k>0$ due to Condition {\em 2}.
Together with the above, we can now proof the theorem by noticing that the solution satisfies Condition {\em 4} for any choice of the $p_{ij}$. 
\end{proof}

\noindent
{\em Remark:} We have found that entropy is the measure of occupancy. The famous derivation of entropy as a measure of information \cite{shannon_mathematical_1948} uses similar elements, but some differences are worthy to be mentioned. 
First, our proof uses the notion of additivity of occupancy on MDPs of length two (our Condition {\em 4}), while Shannon's notion of additivity uses sequences of random variable of arbitrary length (his Condition {\em 3}), and therefore his condition is in a sense stronger than ours.
Second, our proof enforces continuous derivative of the measure, while Shannon enforces continuity of the measure, rendering our Condition {\em 3} stronger. 
Finally, we enforce a specific form of the measure as an average over occupancy gains (our Condition {\em 4} again), because it intuitively captures the notion of occupancy, while Shannon does not enforce this structure in his information measure.

\begin{corollary}[]
	\label{corollary2}
	\normalfont
	Condition {\em 4} can be replaced by the stronger condition that requires additivity of paths of any finite length $n$ with no change in the above proof. We first introduce some notation: the probability of path $i_0,i_1,...,i_n$ is $p_{i_{0},i_{1}} p_{i_{1},i_{2}} ... p_{i_{n-1},i_{n}}$, where $i_t$ refers to the state visited at step $t$ and $i_0$ is the initial state.
	Then the new Condition {\em 4} reads in terms of the action-state occupancy of paths of length $n$ as
	
	\bea
	\nonumber
	C_{i_0}^{(n)} &=& \sum_{i_1,i_2,...,i_n} p_{i_{0},i_{1}} p_{i_{1},i_{2}} ... p_{i_{n-1},i_{n}}  C\left(p_{i_{0},i_{1}} p_{i_{1},i_{2}}... p_{i_{n-1},i_{n}}\right)
	\\
	\nonumber
	&&  = \sum_{i_{1}} p_{i_{0},i_{1}} C(p_{i_{0},i_{1}}) + \sum_{i_1,i_2} p_{i_0,i_1} p_{i_{1},i_{2}} C(p_{i_1,i_2}) + ... 
	+ \sum_{i_1,i_2,...,i_n} p_{i_{0},i_{1}} p_{i_{1},i_{2}}... p_{i_{n-1},i_{n}}  C\left(p_{i_{n-1},i_{n}}\right)
	\\
	\nonumber
	&& = \sum_{i_1,i_2,...,i_n} p_{i_{0},i_{1}} p_{i_{1},i_{2}} ... p_{i_{n-1},i_{n}}  \left( C(p_{i_0,i_1}) + C(p_{i_1,i_2})... +  C(p_{i_{n-1},i_{n}}) \right) \;,
	\eea
	
	\noindent
	for any time-homogeneous Markov chain. By choosing the particular chains used in Theorem \ref{th1}, we arrive again to the same unique solution $C(p) = -k \ln p$ after using $C(1)=0$ repeated times, which obviously solves the above equation for any chain and length path. 
	Indeed, note that for the second chain in Theorem \ref{th1}, from initial state $1$ the absorbing state is reached in three time steps with probability one, and thus the above sum contains all $C(1)$ starting from the third terms, which contribute zero to the sum. 
\end{corollary}

	The above entropy measure of action-state path occupancy can be extended to the case where there is a discount factor $0< \gamma < 1$.
	To do so, we assume now that the paths can have a random length $n \geq 1$ that follows a geometric distribution, $p_n = \gamma^{n-1}(1-\gamma)$.  
	In this case, the occupancy of the paths is
	
	\bea
	\nonumber
	   C_{\text{global}} &=& (1-\gamma)\sum_{i_{1}} p_{i_0,i_1} C(p_{i_0,i_1}) + \gamma(1-\gamma) \sum_{i_1,i_2} p_{i_0,i_1} p_{i_1,i_2} C(p_{i_0,i_1}p_{i_1,i_2})
	   \\
	   && + \gamma^2(1-\gamma) \sum_{i_1,i_2, i_3} p_{i_0,i_1} p_{i_1,i_2} p_{i_2,i_3} C(p_{i_0,i_1}p_{i_1,i_2}p_{i_2,i_3}) + ...
	   \label{eq_C_global}
	\eea
	
	\noindent
	where the $n$-th term in the sum is the expected occupancy gain of paths of length $n$ weighted by the probability of a having a path with exactly such a length. 
	
	Equivalently, a path in course can grow one step further with probability $\gamma$ or be extinguished with probability $1-\gamma$. Therefore, 
	the occupancy in Eq. (\ref{eq_C_global}) should also be equal to the sum of the expected occupancy gains of the local states along the paths, defined as 
	
	\be
	C_{\text{local}} = \sum_{i_{1}} p_{i_0,i_1} C(p_{i_0,i_1}) + \gamma \sum_{i_1,i_2} p_{i_0,i_1} p_{i_1,i_2} C(p_{i_1,i_2})
	+ \gamma^2 \sum_{i_1,i_2, i_3} p_{i_0,i_1} p_{i_1,i_2} p_{i_2,i_3} C(p_{i_2,i_3}) + ...
	 \label{eq_C_local}
	\ee
	
	\noindent
	where the first term is the expected occupancy gain given by the initial condition, the second term is the expected occupancy gain in the next step weighted by the probability of having a path length of at least two steps, and so on.  
	
	Eqs. (\ref{eq_C_global}-\ref{eq_C_local}), after using the Markov chain in Corollary \ref{corollary2}, reduce to

	\bea
		\nonumber
		C_{\text{global}} &=& (1-\gamma)\sum_{i_{1}} p_{i_0,i_1} C(p_{i_0,i_1}) + \gamma(1-\gamma) \sum_{i_1,i_2} p_{i_0,i_1} p_{i_1,i_2} C(p_{i_0,i_1}p_{i_1,i_2})
		\\
		\nonumber
		&& + \gamma^2(1-\gamma) \sum_{i_1,i_2} p_{i_0,i_1} p_{i_1,i_2} C(p_{i_0,i_1}p_{i_1,i_2}) + ...
		\\
		\nonumber
		&=& (1-\gamma)\sum_{i_{1}} p_{i_0,i_1} C(p_{i_0,i_1}) + \gamma \sum_{i_1,i_2} p_{i_0,i_1} p_{i_1,i_2} C(p_{i_0,i_1}p_{i_1,i_2})
	\eea
	
	\noindent
	and
	
	\be
	\nonumber
	C_{\text{local}} = \sum_{i_{1}} p_{i_0,i_1} C(p_{i_0,i_1}) + \gamma \sum_{i_1,i_2} p_{i_0,i_1} p_{i_1,i_2} C(p_{i_1,i_2}) ,
	\ee
	
	\noindent
	where we have used $p_{i_2,i_3}=1$ because all transitions in the third step are deterministic.	
	
	Equality of these two quantities leads to Condition {\em 4},
	specifically, $\sum_{i_1,i_2} p_{i_0,i_1} p_{i_1,i_2} C(p_{i_0,i_1}p_{i_1,i_2}) = \sum_{i_{1}} p_{i_0,i_1} C(p_{i_0,i_1}) + \sum_{i_1,i_2} p_{i_0,i_1} p_{i_1,i_2} C(p_{i_1,i_2})$. 
	Therefore, the only consistent measure of occupancy with temporal discount is the entropy. Obviously, the equality of global and local time-discounted occupancies measured by entropy holds for any time-homogeneous or inhomogeneous Markov chain.	

\subsection{Critical policies and critical state-value functions}
\label{Sec:critical-policies}

Here, the expected return following policy $\pi$ in Eq. (\ref{eq_C_local}), known as the state-value function, 
is written recursively using the Bellman equation.  
Then, we find a non-linear system of equations for the critical policy and critical state-value function by taking partial derivatives with respect to the policy probabilities (Theorem \ref{th_opt_pi_v}).  


Using Eq. (\ref{eq_C_local}) and Theorem \ref{th1} with $k=1$, we define the expected return from state $s$ under policy $\pi$ as

\be
V_{\pi}(s) = - \sum_{i_{1}} p_{s,i_1} \ln p_{s,i_1} - \gamma \sum_{i_1,i_2} p_{s,i_1} p_{i_1,i_2} \ln p_{i_1,i_2}
- \gamma^2 \sum_{i_1,i_2, i_3} p_{s,i_1} p_{i_1,i_2} p_{i_2,i_3} \ln p_{i_2,i_3} + ...
\label{eq_V}
\ee

\noindent
where $p_{s,i_1}$ is the transition probability from state $s$ to action-state $x_{i_1}=(a_{i_1},s_{i_1})$. 
Note that in Eq. (\ref{eq_C_local}) we have replaced the initial action-state $i_{0}$ by the initial state $s$ alone, as the previous action that led to it does no affect the transition probabilities in the Markov decision process setting.
The expected returns satisfy the standard recurrence relationship \cite{sutton_introduction_1998}

\bea
V_{\pi}(s) &=& \sum_{a,s'} p_{s,(a,s')} 
\left(
- \ln p_{s,(a,s')} 
+ \gamma V_{\pi}(s')
\right)
\nonumber
\\
&=& \sum_{a,s'} \pi(a|s) p(s'|s,a) 
\left(
- \ln \pi(a|s) p(s'|s,a) 
+ \gamma V_{\pi}(s')
\right)
.
\label{eq_V_rec}
\eea

\noindent
Here, we have unpacked the sum over the action-state $i_{1}$ into a sum over $(a,s')$, where $a$ is the action made in state $s$ and $s'$ is its successor. The second equation shows, in a more standard notation, the explicit dependence of the expected return on the policy. It also highlights that the intrinsic immediate reward takes the form $R_{\text{intrinsic}}(s,a,s')=- \ln \pi(a|s) p(s'|s,a) $, which is unbounded. 

From Eq. (\ref{eq_V}) it is easy to see that the expected return exists (is finite) for any policy $\pi$ if the Markov decision process has a finite number of actions and states. 
Due to the properties of entropy, Eq. (\ref{eq_V}) is a sum of non-negative numbers bounded by $H_{max}=\ln (|A|_{max}|S|)$ ($|A|_{max}$ is the maximum number of available actions from any state) weighted by the geometric series, which guarantees convergence of the infinite sum for $-1<\gamma<1$.
An obvious, but relevant, implication of the above is that the expected return is non-negative and bounded, $0 \leq V_{\pi}(s) \leq H_{max} / (1-\gamma)$, for any state and policy.

While in Eq. (\ref{eq_V_rec}) the immediate intrinsic reward is the sum of the action and state occupancies, 
$R_{\text{intrinsic}}(s,a,s')=- \ln \pi(a|s) p(s'|s,a) = -\ln \pi(a|s) - \ln p(s'|s,a)$, we can generalize this reward to consider any weighted mixture of entropies as
$R_{\text{intrinsic}}(s,a,s')= - \alpha \ln \pi(a|s) - \beta \ln p(s'|s,a)$ for any two numbers $\alpha>0$ and $\beta \geq 0$. In particular, for $(\alpha,\beta)=(1,1)$ we recover the action-state occupancy of Eq. (\ref{eq_V_rec}), and for $(\alpha,\beta)=(1,0)$ and $(\alpha,\beta)=(0,1)$ we only consider action or state occupancy, respectively. The case $(\alpha,\beta)=(0,1)$ is understood as the limit case where $\alpha$ becomes infinitely small.
We note that the case $(\alpha,\beta)=(1,0)$ has often been used along with an external reward with the aim of regularizing the external reward objective
\cite{ziebart_modeling_2010,todorov_efficient_2009,schulman_equivalence_2017,haarnoja_soft_2018,hausman_learning_2018}. 
We also note that the case $(\alpha,\beta)=(1,-1)$, with negative $\beta$, constitutes an approximation to empowerment 
\cite{klyubin_empowerment_2005,jung_empowerment_2011}: the agent tries to maximize action entropy while minimizing state entropy conditioned to the previous action-state, which favors paths where there is more control on the resulting states. However, we do not consider this case in this paper.

Under the more general intrinsic reward, the expected return obeys 

\be
V_{\pi}(s)
= \sum_{a,s'} \pi(a|s) p(s'|s,a) 
\left(
- \ln \pi^{\alpha}(a|s) p^{\beta}(s'|s,a) 
+ \gamma V_{\pi}(s')
\right).
\label{eq_V_rec_gen}
\ee

\noindent
Our goal is to maximize the expected return over the policy probabilities $\pi = \{\pi(a|s) : a \in A(s), s \in S\}$ to obtain the optimal policy. 
Note that for $\alpha>0$ and $\beta \geq 0 $ the expected return is non-negative, $V_{\pi}(s) \geq 0$.

\begin{theorem}[]
	The critical values $V^c(s)$ of the expected returns $V_{\pi}(s)$ in equation (\ref{eq_V_rec_gen})
	with respect to the policy probabilities $\pi = \{\pi(a|s): a \in A(s), s \in S\}$ obey
	
	\be
	V^c(s) = \alpha \ln Z(s) = \alpha \ln	\left[
	\sum_{a \in \mathcal{A}(s)} \exp 
	\left(
	\alpha^{-1} \beta \mathcal{H}(S'|s,a)
	+ \alpha^{-1} \gamma \sum_{s'} p(s'|s,a)  V^c(s')
	\right)
	\right]
	\label{eq_v_opt}
	\ee
	
	\noindent
	where $\mathcal{H}(S'|s,a)= -\sum_{s'} p(s'|s,a) \ln p(s'|s,a)$ is the entropy of the successors of $s$ after performing action $a$, and $Z(s)$ is the partition function.
	
	The critical points (critical policies) are
	
	\be
	\pi^c(a|s) = \frac{1}{Z(s)} \exp 
	\left(
	\alpha^{-1} \beta \mathcal{H}(S'|s,a)
	+ \alpha^{-1} \gamma \sum_{s'} p(s'|s,a)  V^c(s')
	\right)
	,
	\label{eq_pi_opt}
	\ee
	
	\noindent
	one per critical value, where the partition function $Z(s)$ is the normalization constant.

	Defining $z_i=\exp( \alpha^{-1} \gamma V^c(s_i) )$, $p_{ijk}=p(s_j|s_i,a_k)$ and $\mathcal{H}_{ik} = \alpha^{-1} \beta \mathcal{H}(S'|s_i,a_k)$, Eq. (\ref{eq_v_opt}) can be compactly rewritten as
	
	\be
	z_i^{\gamma^{-1}} = \sum_k w_{ik} e^{\mathcal{H}_{ik}} 
	\prod_j z_j^{p_{ijk}}	 
	\label{eq_z_opt}
	\ee
	
	\noindent
	where the matrix with coefficients $w_{ik} \in \{0,1\}$ indicates whether action $a_k$ is available at state $s_i$ ($w_{ik} = 1$) or not ($w_{ik}=0$), and $j$ extends over all states, with the understanding that if a state $s_j$ is not a possible successor from state $s_i$ and action $a_k$ then $p_{ijk}=0$. 
		
	\label{th_opt_pi_v}
\end{theorem}

Note that the we simultaneously optimize $|S|$ expected returns, one per state $s$, each with respect to the set of probabilities $\pi=\{\pi(a|s) : a \in A(s), s \in S\}$.

\begin{proof}
    We first note that the expected return in Eq. (\ref{eq_expected_return}) is continuous and has continuous derivatives with respect to the policy except at the boundaries (i.e., $\pi(a|s)=0$ for some action-state $(a,s)$). 
	Choosing a state $s$, we first take partial derivatives with respect to $\pi(a|s)$ for each $a \in \mathcal{A}(s)$ in both sides of (\ref{eq_V_rec_gen}), and then evaluate them at a critical point $\pi^c$ to obtain the condition
	
	\bea
	\nonumber
	\lambda(s,s) &
	= & \sum_{s'} p(s'|s,a) 
	\left(
	- \ln (\pi^c(a|s))^{\alpha} p^{\beta}(s'|s,a) 
	+ \gamma V^c(s')
	\right)
	- \alpha
	+ \gamma \sum_{b,s'} \pi^c(b|s) p(s'|s,b) 
	\lambda(s',s)
	\\
	\nonumber
	&=& 
	- \alpha \ln \pi^c(a|s) 
	- \beta \sum_{s'} p(s'|s,a) \ln p(s'|s,a) 
	- \alpha
	\\
	&& 
	+ \gamma \sum_{s'} p(s'|s,a)  V^c(s')
	+ \gamma \sum_{b,s'} \pi^c(b|s) p(s'|s,b) 
	 \lambda(s',s) ,
	\label{eq_V_rec_gen_der}
	\eea
	
	\noindent
	where we have defined the partial derivative at the critical point $\frac{\partial V_{\pi}(s')}{\partial \pi(a|s)}|_{\pi^c} \equiv \lambda(s',s)$ and used the fact that this partial derivative should be action-independent. 
	To understand this, note that the critical policy should lie in the simplex $\sum_a \pi(a|s)=1$, $\pi(a|s) \geq 0$, and therefore the gradient of $V_{\pi}(s')$ with respect to the $\pi(a|s)$ at the critical policy should be along the normal to the constraint surface, i.e., the diagonal direction (hence, action-independent), or be zero.  
	Indeed, the action-independence of the $\lambda(s',s)$ also results from interpreting them as Lagrange multipliers: $\lambda(s',s)$ is the Lagrange multiplier corresponding to the state-value function at $s'$, $V_{\pi}(s')$, associated to the constraint $\sum_a \pi(a|s)=1$, $\pi(a|s) \geq 0$, defining the simplex where the probabilities $\{\pi(a|s) : a \in A(s)\}$ lie. 
	
	Noticing that the last term of Eq. (\ref{eq_V_rec_gen_der}) does not depend on $a$, we can solve for the critical policy $\pi^c(a|s)$ to obtain equation (\ref{eq_pi_opt}).
	Eq. (\ref{eq_pi_opt}) implicitly relates the critical policy with the critical value of the expected returns from each state $s$.
	Inserting the critical policy (\ref{eq_pi_opt}) into Eq.  (\ref{eq_V_rec_gen}), we get (\ref{eq_v_opt}),
	which is an implicit non-linear system of equations exclusively depending on the critical values. 
	
	It is easy to verify that the partial derivatives of $V_{\pi}(s)$ in Eq. (\ref{eq_V_rec_gen}) with respect to $\pi(a'|s')$ for $s \neq s'$ are
	
	\be
	\lambda(s,s') =  
	\gamma \sum_{s''} p(s''|s)  \lambda(s'',s')
	,
	\nonumber
	\ee
	
	\noindent 
	and thus they provide no additional constraint on the critical policy.
 \footnote{This set of equations along with Eq.  (\ref{eq_V_rec_gen_der}) generates a linear system of $\mathcal{S}^2$ equations for the $\mathcal{S}^2$ unknowns $\lambda(s,s')$. In the next section we show that the critical values $V^c(s)$ and critical policy $\pi^c(a|s)$ exists and are unique, and thus the system of equations for $\lambda(s,s')$ is of the type $\Lambda = \gamma P^\intercal \Lambda + F$, with unique matrices $\Lambda_{s s'}= \lambda(s,s')$, $P_{s' s}=p(s'|s) \equiv \sum_a \pi^c(a|s) p(s'|s,a)$ and $F_{s' s}$ is a diagonal matrix with $F_{s s} = V^c(s) - \alpha$. Because $P$ is a stochastic matrix, it does not have eigenvalues larger than one. Therefore the matrix $\mathbb{I}  - \gamma P^\intercal$ with $\gamma<1$ does not have zero eigenvalues, and thus it is invertible. The solution to the system is then unique and given thenby $\Lambda = (\mathbb{I}  - \gamma P^\intercal)^{-1} F$. }

\end{proof}
	 
We finally show that the optimal expected returns, as defined from the Bellman optimality equation

\be
V^*(s)
= \max_{\pi(\cdot|s)}
\sum_{a,s'} \pi(a|s) p(s'|s,a) 
\left(
- \ln \pi^{\alpha}(a|s) p^{\beta}(s'|s,a) 
+ \gamma V^*(s')
\right),
\label{eq_V_bellman}
\ee

\noindent
obey the same Eq. (\ref{eq_v_opt}) as the critical values of Eq. (\ref{eq_V_rec_gen}) do. 
To see this, note that after taking partial derivatives with respect to $\pi(a|s)$ for each $a \in \mathcal{A}(s)$ on the right-hand side of Eq. (\ref{eq_V_bellman}) we get

\be
0 = 
- \alpha \ln \pi(a|s) 
- \beta \sum_{s'} p(s'|s,a) \ln p(s'|s,a) 
+ \gamma \sum_{s'} p(s'|s,a)  V^*(s')
- \alpha 
+ \lambda(s) ,
\label{eq_bellman_der}
\ee

\noindent
where $\lambda(s)$ is the Lagrange multiplier associated to the constraint $\sum_a \pi(a|s) = 1$.
This equation, except for the irrelevant action-independent Lagrange multipliers, is identical to Eq. (\ref{eq_V_rec_gen_der}).
Eq. (\ref{eq_v_opt}) follows from inserting the resulting optimal policy into the Bellman optimality equation.


\subsection{Unicity of the optimal value and policy, and convergence of the algorithm}
\label{Sec:unicity}
	
We now prove that the critical value $V^c(s)$ is unique, in other words, equation (\ref{eq_v_opt}) admits a single solution (Theorem \ref{th_opt_pi_z}). We later prove that the solution is the optimal expected return (Theorem \ref{th_opt_pi_z_opt}). 

\begin{theorem}[]
	With the definitions in Theorem \ref{th_opt_pi_v}, the system of equations	
	
	\be
	z_i^{\gamma^{-1}} = \sum_k w_{ik} e^{\mathcal{H}_{ik}} 
	\prod_j z_j^{p_{ijk}}	 
	\label{eq_z_opt2}
	\ee
	
	\noindent
	with $0< \gamma < 1$, $\alpha >0$ and $\beta \geq 0$ has a unique solution in the positive first orthant $z_i > 0$, provided that for all $i$ there exists at least one $k$ such that $w_{ik}=1$.	
	The solution satisfies $z_i \geq 1$.
	
	Moreover, given any initial condition $z_i^{(0)}>0$ for all $i$, the infinite series $z_i^{(n)}$ defined through the iterative map 
	
	\be
	z_i^{(n+1)} = 
	\left(
	\sum_k w_{ik} e^{\mathcal{H}_{ik}} 
	\prod_j \left(z_j^{(n)}\right)^{p_{ijk}}	
	\right)^{\gamma} 
	\label{eq_z_map}
	\ee
	
	\noindent
	for $n \geq 0$ converges to a finite limit $z_i^{\infty} \geq 1$, and this limit is the unique solution of equation (\ref{eq_z_opt2})

	\label{th_opt_pi_z}
\end{theorem}

Note that the condition that for all $i$ there exists at least one $k$ such that $w_{ik}=1$ imposes virtually no restriction, as it only asks for the presence of at least one available action in each state. For instance, in absorbing states, the action leads to the same state. 

Importantly, proving that the map (\ref{eq_z_map}) has a single limit regardless of the initial condition in the positive first orthant $z_i^{(0)}>0$ suffices to prove that equation (\ref{eq_z_opt2}) has a unique solution in that region, as then no other fix point of the map can exist.
Additionally, since the solution is unique and satisfies $z_i^{\infty} \geq 1$, the critical state-value function that solves equation (\ref{eq_v_opt}) is unique, and $V^c(s_i)=\alpha \gamma^{-1} \ln z_i^{\infty} \geq 0$, consistent with its properties. 

The map (\ref{eq_z_map}) provides a useful value-iteration algorithm used in examples shown in the Results section, and empirically is found to rapidly converge to the solution.

\begin{proof}
	We call the series $z_i^{(n)}$ with initial condition $z_i^{(0)}=1$ for all $i$ the {\em main} series.
	We first show that the main series is monotonic non-decreasing. 
	
	For $n=1$, we get
	
	\be
	z_i^{(1)} = 
	\left(
	\sum_k w_{ik} e^{\mathcal{H}_{ik}} 
	\prod_j \left(1\right)^{p_{ijk}}	
	\right)^{\gamma}
	\geq 1 = z_i^{(0)} 
	\label{eq_z_0}
	\ee
	
	\noindent
	for all $i$, using that there exists $k$ for which, $w_{ik} = 1$, $w_{ik}$ is non-negative for all $i$ and $k$, $\mathcal{H}_{ik} \geq 0$ and the power function $x^\gamma$ is increasing with its argument. 
	
	Assume that for some $n>0$, $z_i^{(n)} \geq z_i^{(n-1)}$ for all $i$. Then 
	
	\be
	z_i^{(n+1)} = 
	\left(
	\sum_k w_{ik} e^{\mathcal{H}_{ik}} 
	\prod_j \left(z_j^{(n)}\right)^{p_{ijk}}	
	\right)^{\gamma}
	\geq 
	\left(
	\sum_k w_{ik} e^{\mathcal{H}_{ik}} 
	\prod_j \left(z_j^{(n-1)}\right)^{p_{ijk}}	
	\right)^{\gamma}
	= z_i^{(n)}
	\label{eq_z_n}
	\ee
	
	\noindent
	using the same properties as before, which proves the assertion for all $n$ by induction.
	
	Now let us show that the main series is bounded. Define $\mathcal{H}_{\text{max}}=\max_{ik} \mathcal{H}_{ik}$, and obviously $\mathcal{H}_{\text{max}} \geq 0$. 
	
	For $n=1$ we have

	\be
	z_i^{(1)} = 
	\left(
	\sum_k w_{ik} e^{\mathcal{H}_{ik}} \right)^{\gamma}
	\leq 
	\left(
	|A|_{\text{max}}  e^{\mathcal{H}_{\text{max}}}
	\right)^{\gamma}
	\equiv c^{\gamma}
	\label{eq_z_0_c}
	\ee
	
	\noindent
	(remember that $|A|_{\text{max}}$ is the maximum number of available actions from any state).
	
	For $n=2$, 
	
	\bea
	\nonumber
	z_i^{(2)} &=& 
	\left(
	\sum_k w_{ik} e^{\mathcal{H}_{ik}} 
	\prod_j \left(z_j^{(1)}\right)^{p_{ijk}}	
	\right)^{\gamma}
	\leq 
	\left(
	\sum_k w_{ik} e^{\mathcal{H}_{ik}} 
	\prod_j c^{\gamma p_{ijk}}	
	\right)^{\gamma}
	\\
	\nonumber
	&=& \left(
	\sum_k w_{ik} e^{\mathcal{H}_{ik}} c^\gamma{}	
	\right)^{\gamma}
	= c^{\gamma^2}
	\left(
	\sum_k w_{ik} e^{\mathcal{H}_{ik}} 
	\right)^{\gamma}
	\leq c^{\gamma + \gamma^2} 
	\label{eq_z_n}
	\eea
	
	\noindent
	using the standard properties, $\sum_j p_{ijk}=1$ and Eq. (\ref{eq_z_0_c}).
	
	Assume that for some $n>1$ we have $z_i^{(n)} \leq c^{\gamma + \gamma^2+ ... + \gamma^{n}}$. We have just showed that this is true for $n=2$. Then
	
	\bea
	\nonumber
	z_i^{(n+1)} &=& 
	\left(
	\sum_k w_{ik} e^{\mathcal{H}_{ik}} 
	\prod_j \left(z_j^{(n)}\right)^{p_{ijk}}	
	\right)^{\gamma}
	\leq 
	\left(
	\sum_k w_{ik} e^{\mathcal{H}_{ik}} 
	c^{\gamma + ... + \gamma^{n}}	
	\right)^{\gamma}
	\\
	\nonumber
	&=&
	c^{\gamma^2+ ... + \gamma^{n+1}}
	\left(
	\sum_k w_{ik} e^{\mathcal{H}_{ik}} 
	\right)^{\gamma}
	\leq c^{\gamma + ...+ \gamma^{n+1}} 
	\label{eq_z_n}
	\eea
	
	\noindent
	and therefore it is true for all $n \geq 0$ by induction.
	
	Therefore the series $z_i^{(n)}$ is bounded by $c^{1/(1-\gamma)}$. Together with the monotonicity of the series, we have now proved that the limit $z_i^{\infty}$ of the series exists. Moreover, $z_i^{\infty} \geq z_i^{0} = 1$. 
	
	The above results can be intuitively understood: the `all ones' initial condition of the main series corresponds to an initial guess of the state-value function equal to zero everywhere. The iterative map corresponds to state-value iteration to a more optimistic value: as intrinsic reward based on entropy is always non-negative, the $z$-values monotonically increase after every iteration. 
	Finally, the $z$-values reach a limit because the state-value function is bounded.  
	
	We now show the central result that the series obtained by using the iterative map starting from any initial condition in the positive first orthant can be bounded below and above by two series that converge to the main series. Therefore, by building `sandwich' series we will confirm that any other series has the same limit as the main series. 
	
	Let the $y_i^{(0)}=u_i>0$ be the initial condition of the series $y_i^{(n)}$ obeying the iterative map (\ref{eq_z_map}), and define $u_{\text{min}}= \min_i u_i$ and $u_{\text{max}}= \max_i u_i$. Obviously, $u_{\text{min}} > 0$ and $u_{\text{max}} > 0$.
	Applying the iterative map once, we get  
	
	\bea
	\nonumber
	y_i^{(1)} &=& 
	\left(
	\sum_k w_{ik} e^{\mathcal{H}_{ik}} 
	\prod_j \left(y_j^{(0)}\right)^{p_{ijk}}	
	\right)^{\gamma}
	\leq 
	\left(
	\sum_k w_{ik} e^{\mathcal{H}_{ik}} 
	\prod_j \left(u_{\text{max}}\right)^{p_{ijk}}	
	\right)^{\gamma}
	\\
	\nonumber
	&=&
	\left(
	\sum_k w_{ik} e^{\mathcal{H}_{ik}} 
	u_{\text{max}}	
	\right)^{\gamma}
	=
	u_{\text{max}}^{\gamma}
	\left(
	\sum_k w_{ik} e^{\mathcal{H}_{ik}} 
	\right)^{\gamma}
	= 
	u_{\text{max}}^{\gamma} z_i^{(1)}
	\label{eq_y_0}
	\eea
	
	\noindent
	where in the last step we have used the values of the main series in the first iteration. We can similarly lower-bound $y_i^{(1)}$ to finally show that it is both lower- and upper-bounded by $z_i^{(1)}$ with different multiplicative constants,
	
	\be
	  u_{\text{min}}^{\gamma} z_i^{(1)} \leq  y_i^{(1)} \leq u_{\text{max}}^{\gamma} z_i^{(1)}
	  \label{eq_y_sandwich0}
	\ee
	
	Now, assume that 
	
	\be
	u_{\text{min}}^{\gamma^n} z_i^{(n)} \leq  y_i^{(n)} \leq u_{\text{max}}^{\gamma^n} z_i^{(n)}
	\label{eq_y_sandwichn}
	\ee
	
	\noindent
	is true for some $n>0$. Then, for $n+1$ we get
	
		\bea
	\nonumber
	y_i^{(n+1)} &=& 
	\left(
	\sum_k w_{ik} e^{\mathcal{H}_{ik}} 
	\prod_j \left(y_j^{(n)}\right)^{p_{ijk}}	
	\right)^{\gamma}
	\leq 
	\left(
	\sum_k w_{ik} e^{\mathcal{H}_{ik}} 
	\prod_j \left(u_{\text{max}}^{\gamma^n} z_i^{(n)}\right)^{p_{ijk}}	
	\right)^{\gamma}
	\\
	\nonumber
	&=&
	u_{\text{max}}^{\gamma^{n+1}}
	\left(
	\sum_k w_{ik} e^{\mathcal{H}_{ik}} 
	\prod_j \left( z_i^{(n)}\right)^{p_{ijk}}	
	\right)^{\gamma}
	= 
	u_{\text{max}}^{\gamma^{n+1}} z_i^{(n+1)}
	\label{eq_y_0}
	\eea
	
	\noindent
	by simply extracting the common factor in the fourth expression, remembering that $\sum_j p_{ijk}=1$, and using the definition of the main series in the last one. By repeating the same with the lower bound, we finally find that (\ref{eq_y_sandwichn}) holds also for $n+1$, and then, by induction, for every $n>0$. 
	
	The proof concludes by noticing that the limit of both  $u_{\text{max}}^{\gamma^{n}}$ and $u_{\text{min}}^{\gamma^{n}}$ is $1$, and therefore using (\ref{eq_y_sandwichn}) the limit $y_i^{\infty}$ of the series $y_i^{(n)}$ equals the limit of the main series, $y_i^{\infty} = z_i^{\infty}$.
	
	Note that the iterative map (\ref{eq_z_map}) is not necessarily contractive in the Euclidean metric, as it is possible that, 	
	depending on the values of $u_{\text{min}}$ and $u_{\text{max}}$ and the changes in the main series, the bounds in Eq. (\ref{eq_y_sandwichn}) initially diverge to finally converge in the limit. 
	
\end{proof}

\begin{theorem}[]
	The (unique) critical value $V^c(s)$ is the optimal expected return, that is, the one that attains the maximum expected return at every state for any policy, and we write $V^c(s)=V^*(s)$ 
	\label{th_opt_pi_z_opt}
\end{theorem}

\begin{proof}	
	
	To show that $V^c(s)$ is the optimal expected return, we note that the maximum of the functions $V_{\pi}(s)$ with respect to policy $\pi$ should be at the critical policy or at the boundaries of the simplices defined by 
	$\sum_a \pi(a|s)=1$ with $0 \leq \pi(a|s) \leq 1$ for every $a$ and $s$, as
    the expected return $V_{\pi}(s)$ is continuous and has continuous derivatives with respect to the policy except at the boundaries. At the policy boundary, there exists a non-empty subset of states $s_i$ and a non-empty set of actions $a_k$ for which $\pi(a_k|s_i)=0$.
	Computing the critical value of the expected return along that policy boundary is identical to moving from the original to a new problem where we replace the graph connectivity matrix $w_{ik}$ in Eq. (\ref{eq_z_opt2}) by a new one $v_{ik}$ such that $v_{ik} \leq w_{ik}$ (remember that at the boundary there should be an action $a_k$ that were initially available from state $s_i$, $w_{ik}=1$, that at the policy boundary is forbidden, $v_{ik}=0$).
	We now define the convergent series $z_i^{(n)}$ 
	and $y_i^{(n)}$ for the original and new problems respectively by using the iterative map (\ref{eq_z_map}) with initial conditions equal to all ones.
	We prove now that $z_i^{(n)} \geq y_i^{(n)}$ for all $i$ for $n=1,2,...$, and thus their limits obey $z_i^{\infty} \geq y_i^{\infty}$. 
	
	For $n=1$, we get
	
	\be
	z_i^{(1)} = 
	\left(
	\sum_k w_{ik} e^{\mathcal{H}_{ik}} 
	\prod_j \left(1\right)^{p_{ijk}}	
	\right)^{\gamma}
	\geq 
	\left(
	\sum_k v_{ik} e^{\mathcal{H}_{ik}} 
	\prod_j \left(1\right)^{p_{ijk}}	
	\right)^{\gamma}
	= y_i^{(1)}
	\label{eq_z_map0}
	\ee
	
	\noindent
	for all $i$, using that $w_{ik} \geq v_{ik}$ and that the power function $x^\gamma$ is increasing with its argument.
	
	Assuming that $z_i^{(n)} \geq y_i^{(n)}$ for all $i$ for some $n>0$, then 
	
	\be
	z_i^{(n+1)} = 
	\left(
	\sum_k w_{ik} e^{\mathcal{H}_{ik}} 
	\prod_j \left(z_j^{(n)}\right)^{p_{ijk}}	
	\right)^{\gamma}
	\geq 
	\left(
	\sum_k v_{ik} e^{\mathcal{H}_{ik}} 
	\prod_j \left(y_j^{(n)}\right)^{p_{ijk}}	
	\right)^{\gamma}
	= y_i^{(n+1)}
	\label{eq_z_map0}
	\ee
	
	\noindent
	using the same properties as before, which proves the assertion for all $n$ by induction.
	
	Remembering that the expected return $V(s_i)$ is increasing with $z_i$, we conclude that the expected return obtained from policies restricted on the boundaries of the simplices is no better than the original critical value of the expected return.
	
\end{proof}

\subsection{Particular examples}
\label{sec:examples}

Here we summarize the main results and specialize them to specific cases. We assume $0< \gamma < 1$, $\alpha >0$ and $\beta \geq 0$ and use the notation 
$z_i=\exp( \alpha^{-1} \gamma V^*(s_i) )$, where $V^*(s)$ is the optimal expected return, $p_{ijk}=p(s_j|s_i,a_k)$ and $\mathcal{H}_{ik} = \alpha^{-1} \beta \mathcal{H}(S'|s_i,a_k)$, where $\mathcal{H}(S'|s,a)= -\sum_{s'} p(s'|s,a) \ln p(s'|s,a)$. 

\subsubsection{Action-state entropy maximizers}\label{st-act-max}

Agents that seek to maximize the discounted action-state path entropy follow the optimal policy

\be 
\pi^*(a_k|s_i) = \frac{1}{Z_i}
\left(
w_{ik} e^{\mathcal{H}_{ik}} 
\prod_j z_j^{p_{ijk}}
\right)
\label{eq_pi_opt3}
\ee

\noindent
with 

\be
Z_i = 
\sum_{k} w_{ik} e^{\mathcal{H}_{ik}} 
\prod_j z_j^{p_{ijk'}}	 
\label{eq_Z_opt3}
\ee

\noindent
The matrix with coefficients $w_{ik} \in \{0,1\}$ indicate whether action $a_k$ is available at state $s_i$ ($w_{ik} = 1$) or not ($w_{ik}=0$).

The expected return (state-value function) in terms of the $z$ variables obeys	

\be
z_i^{\gamma^{-1}} = \sum_k w_{ik} e^{\mathcal{H}_{ik}} 
\prod_j z_j^{p_{ijk}}	 
\label{eq_z_opt3}
\ee

\subsubsection{Action-only entropy maximizers}\label{act-max}

Agents that ought to maximize the time-discounted action path entropy correspond to the above case with $\beta=0$, and therefore the optimal policy reads as

\be
\pi^*(a_k|s_i) = \frac{1}{Z_i}
\left(
w_{ik} \prod_j z_j^{p_{ijk}}
\right)
\label{eq_pi_opt4}
\ee

\noindent
with 

\be
Z_i = 
\sum_{k} w_{ik} 
\prod_j z_j^{p_{ijk}}	 
\label{eq_Z_opt4}
\ee

The state-value function in terms of the $z$ variables obeys	

\be
z_i^{\gamma^{-1}} = \sum_k w_{ik} 
\prod_j z_j^{p_{ijk}}	 
\label{eq_z_opt4}
\ee

\subsubsection{Entropy maximizers in deterministic environments}

In a deterministic environment $p_{i,j(i,k),k}=1$ for successor state $j=j(i,k)$, and zero otherwise. In this case, at every state $i$ we can identify an action $k$ with its successor state $j$. Therefore, the optimal policy is

\be
\pi^*(a_k|s_i) = \frac{w_{ij} z_j}{Z_i}
\label{eq_pi_opt5}
\ee

\noindent
with 

\be
Z_i = 
\sum_{j} w_{ij} z_{j}	 
\label{eq_Z_opt5}
\ee

The state-value function in terms of the $z$ variables reads

\be
z_i^{\gamma^{-1}} = \sum_j w_{ij} z_j	 
\label{eq_z_opt5}
\ee

\begin{figure}[t!]
\center
\includegraphics[width=\textwidth]{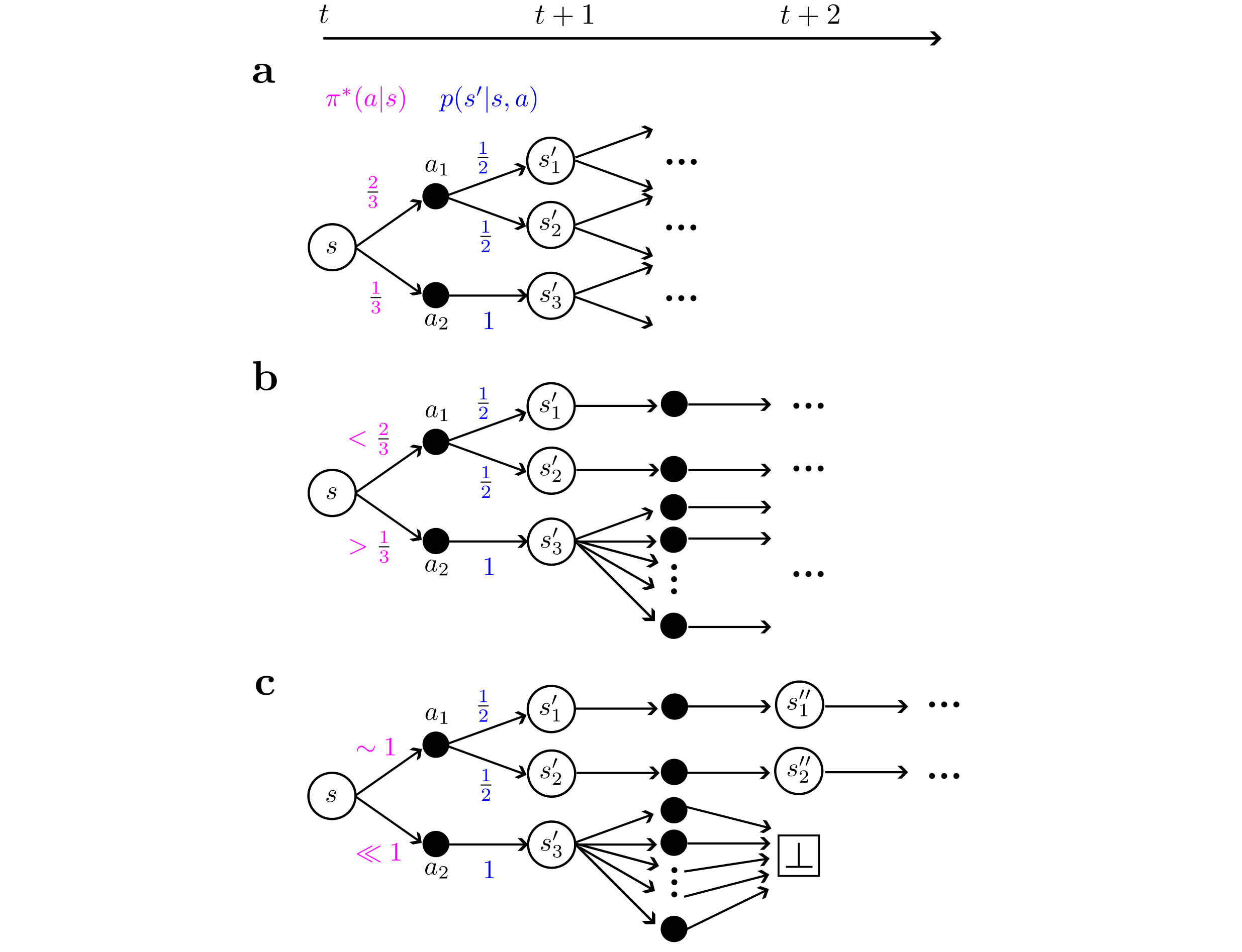} \newline
\\
 \caption{MOP agents determine stochastic policies that maximize occupancy of future action-state paths. In all panels, the three successive dots indicate that the future looks the same for all the states or actions involved from that point onwards. (a) At time $t$, the agent is faced with determining the optimal policy at state $s$. Given that taking action $a_1$ can stochastically lead to two distinct states $s_1'$ and $s_2'$, the optimal policy gives action $a_1$ twice the probability weight than to action $a_2$ (which only induces a deterministic transition to state $s_3'$). From time $t+1$, the future looks the same from all three states $s_i'$. (b) If the future does not look the same, and actually there are many more actions available at state $s_3'$ compared to $s_1'$ and $s_2'$, then more weight should be given to action $a_2$ than if the future was the same. (c) If, however, all the actions available at state $s_3'$ lead you to an absorbing state, almost zero weight should be given to action $a_2$.}
 \label{fig:schematic_formal}
 \end{figure}

\subsection{Experiments}
\label{sec:experiments}
In this subsection, we present the details for the numerical simulations performed for the different experiments in the manuscript. First, we discuss the construction of the MOP and R agents, and afterwards we present the details of each particular experiment.

\subsubsection{MOP agent}
\label{sec:h_agent}
In all the experiments presented, we introduce the MOP agent, whose name comes from the usual notation for using H to denote entropy. Therefore, the objective function that this agent maximizes in general is Eq. (\ref{eq_expected_return}). As described in section \ref{sec:examples}, the $\alpha$ and $\beta$ parameters control the weights of action and next-state entropies to the objective function, respectively. Unless indicated otherwise, we always use $\alpha = 1, \beta = 0$ for the experiments. It is important to note, as we have done before, that if the environment is deterministic, then the next-state entropy $\mathcal{H}(S'|s,a)= -\sum_{s'} p(s'|s,a) \ln p(s'|s,a)=0$, and therefore $\beta$ does not change the optimal policy, Eq. (\ref{eq_pi_opt_m}).

We have implemented the iterative map, Eq. (\ref{eq_z_map_m}), to solve for the optimal value, using $z^{(0)}_i = 1$ for all $i$ as initial condition. Theorem (\ref{th_opt_pi_z}) ensures that this iterative map finds a unique optimal value regardless of the initial condition in the first orthant. To determine a degree of convergence, we compute the supremum norm between iterations,
\[
\delta = \max_i |V_i^{(n+1)}-V_i^{(n)}|,
\]
where $V_i = \frac{\alpha}{\gamma}\log(z_i)$, and the iterative map stops when $\delta  < 10^{-3}$.

\subsubsection{R agent} 
\label{sec:r_agent}

We also introduce a reward-maximizing agent in the usual RL sense. In this case, the reward is $r=1$ for living and $r=0$ when dying. In other words, this agent maximizes life expectancy. Additionally, to emphasize the typical reward-seeking behavior and avoid degenerate cases induced by the tasks, we introduced a small reward for the Four-room grid world (see below). In all other aspects, the modelling of the R agent is identical to the MOP agent. To allow for reward-maximizing agents to display some stochasticity, we used an $\epsilon$-greedy policy, the best in the family of $\epsilon$-soft policies \cite{sutton_introduction_1998}. At any given state, a random admissible action is chosen with probability $\epsilon$, and the action that maximizes the value is chosen with probability $1-\epsilon$. Given that the world models $p(s'|s,a)$ are known and the environments are static, this $\epsilon$-greedy policy does not serve the purpose of exploration (in the sense of learning), but only to inject behavioral variability. Therefore, we construct an agent with state-independent variability, whose value function satisfies the optimality Bellman equation for this $\epsilon$-greedy policy, 

\be
V_{\epsilon}(s)=(1-\epsilon)\max_a\sum_{s'}  p(s'|s,a) \left(r+\gamma V_{\epsilon}(s')\right) +\frac{\epsilon}{|\mathcal{A}(s)|}\sum_{a,s'}  p(s'|s,a)\left(r+ \gamma V_{\epsilon}(s')\right),
\ee
where $|\mathcal{A}(s)|$ is the number of admissible actions at state $s$. To solve for the optimal value in this Bellman equation, we perform value iteration \cite{sutton_introduction_1998}.
The $\epsilon$-greedy policy for the R agent is therefore given  by
\[
\pi(a|s) = \begin{cases} 1-\epsilon + \frac{\epsilon}{|\mathcal{A}(s)|},& \text{if $a = \argmax_{a'}\sum_{s'}p(s'|s,a')\left(r +\gamma V_{\epsilon}(s')\right)$}\\ \frac{\epsilon}{|\mathcal{A}(s)|}, & \text{otherwise} \end{cases}
\]
where ties in $\argmax$ are broken randomly. Note that if $\epsilon = 0$, we obtain the usual greedy optimal policy that maximizes reward.

\begin{figure}
\center
\includegraphics[width=\textwidth]{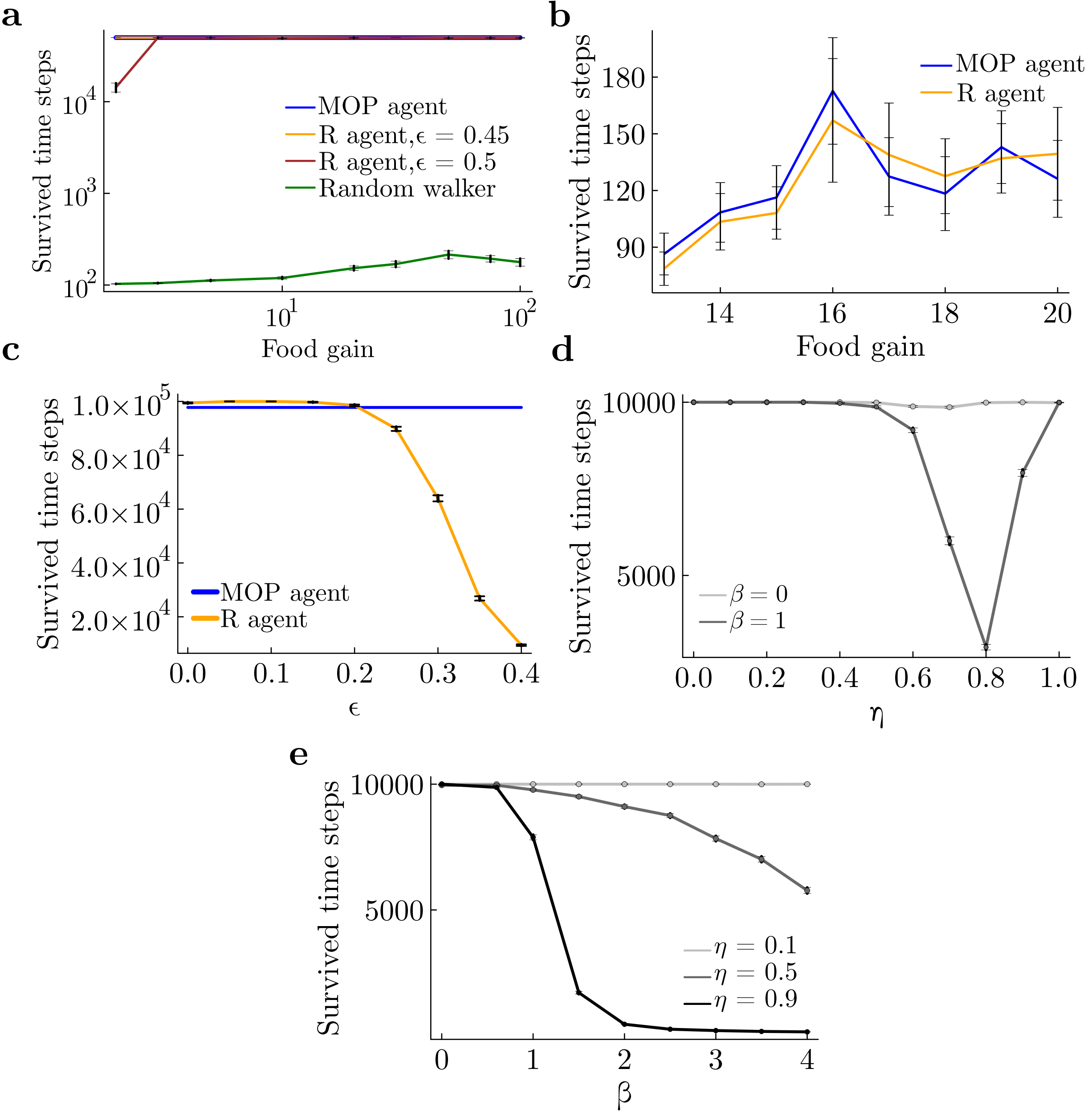} \newline
\\
 \caption{Survivability for the experiments considered in the manuscript. (a) Survivability of the various agents tested in the four-room grid world. At each 5E4 timestep episode, we recorded the survived time and averaged across episodes. (b) Survivability of the mouse for both MOP and R agents. (c) Survivability for the cartpole (Sec. \ref{sec:details_cartpole}) in the deterministic arena for the MOP agent and the $\epsilon$-greedy R agents, $\gamma = 0.98$. (d) Survivability for cartpole (Sec. \ref{sec:details_cartpole}) in the stochastic arena for the $\beta = 0$ and the $\beta = 1$ MOP agents. $\gamma = 0.99$. (e) Survivability of the cartpole (Sec. \ref{sec:details_cartpole}) MOP agents as a function of $\beta$, for various values of $\eta$. $\gamma = 0.99$}
 \label{fig:survival}
 \end{figure}

\subsubsection{Four-room grid world} 
\label{sec:4room}
\paragraph{Environment} The arena is composed of four rooms, each having size $5 \times 5$ locations where the agent can be in. From each room, the agent can go to two adjacent rooms through small openings, each located in the middle of the wall that separates the rooms. At each of these rooms, there is a food source located in the corner furthest from the openings. See Fig. \ref{fig:fourrooms} for a graphic description. Unless indicated otherwise, the discount factor is set to $\gamma = 0.99$.

\paragraph{States} The states are the Cartesian product between $(x,y)$ location and internal state $u$, which is simply a scalar value between a minimum of 0 and a maximum capacity of 100. All states such that $(x,y,u=0)$ are absorbing states, independently of the location $(x,y)$. The particular internal state $u=100$ is the maximum capacity for energy, such that even when at a food source, this internal state does not change. Therefore, the number of states in this experiment is $|\mathcal{S}| = 104 \,\text{external states} 
\times 101 \,\text{internal states} = 10504$.

\paragraph{Actions} The agent has a maximum of 9 actions: \texttt{up, down, left, right, up left, up right, down left, down right}, and \texttt{nothing}. Whenever the agent is close to a wall, the number of available actions decreases such that the agent cannot choose to go into walls. Finally, whenever the agent is in an absorbing state, only \texttt{nothing} is available.

\paragraph{Transitions} At any transition, there is a cost of 1 unit of energy for being alive. On the other hand, whenever the agent is located at a food source, there is an increase in energy that we vary parametrically that we call food gain $g$. For example, if the agent is in location $(2,1)$ at time $t$ and moves towards $(1,1)$ (where food is located), the change in energy would be $\Delta u_t = -1$, given that the change in internal energy depends only on the current state and action. If the agent decides to stay in $(1,1)$ at time $t+1$, then $\Delta u_{t+1} = -1 + g$.

\paragraph{R agent} As stated above, in this experiment we introduced an extra reward for the R agent when it reaches the food source. The magnitude is small compared to the survival reward ($1E-5$ smaller) and it mainly serves to break the degeneracy of the value function. The variability of the R agent is thus coming purely from the $\epsilon$-greedy action selection.

\paragraph{Survivability} To allow for the maximum uniform variability for the R agent, we tested various values for $\epsilon$ and observed the survivability of the agents as a function of $\epsilon$, across all the food gains tested (see Results section). The value of $\epsilon$ for which the R agent still survives as much as the MOP agent is $\epsilon = 0.45$ (see Figure \ref{fig:survival}a).

\paragraph{Noisy room} In this variation for the experiment, there is a room (the bottom right room) where transitions are uniformly random for all actions, across all possible neighboring locations. That is, for any location $s_{nr}$ in the noisy room, and any $a$ available at that location, given that it has $n(s_{nr})$ total neighbours (including the same location),
$$
p(s'|s_{nr},a) = \begin{cases} \frac{1}{n(s_{nr})} & \text{for } s' \in \text{neighbours}
\\ 0 & \text{otherwise}
\end{cases}
$$

\subsubsection{Predator-prey scenario}
\label{sec:details_mouse}
Here we provide all details of the simulated experiments. Results are shown in Fig. \ref{fig:cat_mouse}.

\paragraph{Environment}
The environment is similar to that one used for the 4-room grid world described in \ref{sec:4room}. Apart from the agent (prey), there is also another moving subject (predator) with a simple predefined policy. The grid world consists of a “home” area, a rectangle 2x3 where the agent may enter, but the predator cannot. This home area has a small opening that leads to a bigger 4x7 rectangle arena available for both the agent and the predator. The only food source is located at the bottom-right corner of the common part of the arena, so that the agent needs to leave its home to boost its energy. Additionally, there is an obstacle which separates the arena in two parts with two openings, above and under the obstacle. This obstacle allows the agent to “hide” from the predator behind it. 

\paragraph{States}
The location of the predator is part of the agent's state, such that a particular state consists of the position of the agent, the position of the predator and the amount of energy of the agent. For this case, we set the maximum amount of energy $F$ equal to the food gain. Positions are 2-dimensional, and therefore the states are 5-dimensional.
In the  used arena there are $33$ possible locations for the agent and $26$ ones for the predator, so that the total number of states 
ranges from $11154$ for $F=13$ to $17160$ for $F=20$. 

\paragraph{Actions}
The agent has the same actions as in the four-room grid world. The maximum number of available actions is therefore 9. Moving towards obstacles or walls is not allowed. 

\paragraph{Transitions} The agent
loses one unit of energy every time step and increases the amount of energy up to a given maximum capacity level $F$ only at the food source. 
If the position of both the agent and the predator are the same, then the agent is "eaten" and moves to the absorbing state of death as well as in the case of energy equal to $0$. After entering the absorbing state the agent stays there forever.

The predator also moves as the agent (horizontally, vertically, diagonally on one step or to stay still). Steps of the agent and the predator happen synchronously. The predator is “attracted” to the agent: the probability of moving to some direction is an increasing function on the cosines $\cos \alpha_k$ of the angle $\alpha_k$ between this direction of motion $k$ and the direction of the radius vector from the predator to the agent. In particular, this probability is 

\be
p^c_k=C^{-1}\exp(\kappa \cos \alpha_k)
\label{p_cat}
\ee
where $\kappa$ is the inverse temperature of the predator and $C=\sum_k \exp(\kappa \cos \alpha_k)$ is a normalization factor. These probabilities are computed only for motions available at the current location of the predator, so that e.g. for the location at the wall the motions along the wall are taken into account, but not the motion towards the wall.

\paragraph{Goal}
The goal of the MOP agent is to maximize discounted action entropy, and thus 
to find the optimal state-value function using the iterative map in Eq. (\ref{eq_z_map_m}) with $\mathcal{H}_{ik}=0$ ($\beta=0$).
While using the iterative map, we take advantage of the fact that given an action the physical transition of the agent is deterministic, but the physical transition of the predator is stochastic. Therefore, the sum over successor states $j$ in Eq. (\ref{eq_z_map_m}) is simply a sum over the predator successor states. 


\paragraph{Parameters} \label{ss:arena_fig2}
$\gamma=0.98$, $F =15$ (if another value between 13 and 20 not mentioned), $\kappa=2$. Simulation time is 5000 steps.

\paragraph{Counting rotations} \label{ss:rotations}
We define a clockwise (counterclockwise) half-rotation as the event when the
agent came from the left part of the arena to the right part over the field above (under) the wall and from the right part to the left one over the field under (above) the wall without crossing the vertical line of the wall in between. One full rotation consists of two half-rotations in the same directions performed one after another. We counted the number of full rotations in both directions
in $70$ episodes of $500$ time steps each for both MOP and R agents for different values of the food gain $F$. Error bars were computed based on these $70$ repetitions. 
The fraction of clockwise rotations to total rotations (sum of clockwise and anticlockwise rotations) for different values of $F$ is shown at Fig. \ref{fig:cat_mouse}.

\paragraph{Survivability}
The $\epsilon$-greedy R agents display some variability that depends on $\epsilon$. To select this parameter, we matched average lifetimes (measured in simulations of $5000$ steps length) between the MOP and R agents, separately for every $F$. Lifetimes are plotted in Figure \ref{fig:survival}b.

\paragraph{Videos}

We have generated  one video for the MOP agent (Video 2) and another for the R agent (Video 3), both for $F=15$, $\kappa=2$, and $\epsilon=0.06$ for the R agent so as to match their average lifetimes as described above. In the videos, green vertical bar indicates the amount of energy by the agent at current time. When the agent makes at least one full rotation around the wall, it is indicated by the written phrase ``clockwise rotation'' or ``anticlockwise rotation''. Black vertical arrow indicates direction (`up' for clockwise and `down' for anticlockwise directions) of the half-rotation in the part of arena left from the wall.  

\subsubsection{Cartpole}
\label{sec:details_cartpole}
\paragraph{Environment} A cart is placed in a one-dimensional track with boundaries at $|x| = 1.8$. It has a pole attached to it, that rotates like an inverted pendulum with its pivot point on the cart. 

\paragraph{States} The dynamical system can be described by a four-dimensional external state $(x,v,\theta,\omega)$, where $x$ is the position of the cart, $v$ is its linear velocity, $\theta$ is the angle of the pole with respect to the vertical which grows counterclockwise, and $\omega$ is its angular velocity. In this case, we model the internal state $u$ simply with the binary variable \texttt{alive, dead}, where the agent enters the absorbing state \texttt{dead} if its position exceeds the boundaries, or if its angle exceeds 36 degrees. This amplitude of angles is larger than that typically assumed (12 degrees in \cite{brockman_openai_2016}), and therefore our system is allowed to be more non-linear and unstable. 
The state space is $[-1.8,1.8]\times (-\infty,\infty)\times [-36,36]\times (-\infty,\infty)\times\{0,1\}$. To solve for the state value function in Eq. (\ref{eq_z_map_m}), we discretize the state space by setting a maximum value for the velocities. Given all the parameters (allowed $x$ and $\theta$, magnitude of the forces, masses of cart and pole, length of pole and gravity, below), we empirically set the maximum values for $|v|= 6$ and $|\omega|= 3$, which the cart actually never exceeds. Therefore, we computed the state value function in a $31 \times 31 \times 31 \times 31 \times 2$ grid (number of states = $1.8\times10^6$). 

\paragraph{Actions} Any time the agent is \texttt{alive}, it has 5 possible actions: forces of $\{-40,-10,0,10,40\}$, where zero force is understood as \texttt{nothing}. If the agent is \texttt{dead}, then only \texttt{nothing} is allowed.

\paragraph{Transitions} This dynamical system is a standard task in reinforcement learning, namely the \texttt{cartpole-v0} system of the OpenAI gym \cite{brockman_openai_2016}. The solution of this dynamical system is given in Ref. \cite{florian_correct_2007}, where we use a frictionless cartpole. The equations for angular and linear accelerations are thus
\begin{align}
    \Ddot{\theta} &= \frac{-g\sin(\theta) + \frac{\cos(\theta)}{M + m}\left(-F + m\Dot{\theta}^2l\sin(\theta)\right)}{l\left(\frac{4}{3}-\frac{m\cos^2(\theta)}{M+m}\right)}\\
    \Ddot{x} &= \frac{1}{\cos(\theta)}\left(\frac{4}{3}l\Ddot{\theta}-g\sin(\theta)\right).
\end{align}
Given a force $F$, a deterministic transition can be computed from these dynamical rules, and a real-valued state transition is observed by the agents.

\paragraph{R agent} The reward signal is 1 each time the agent is alive and 0 otherwise. To allow for some variability in the action selection of the R agent, we implement an  $\epsilon$-greedy action selection as described above. For exposition purposes, in the manuscript we set $\epsilon = 0.0$, but we also compared to an R agent with $\epsilon$ chosen such that average lifetimes between MOP and R agents are matched (see Fig. \ref{fig:survival}c and Fig. \ref{fig:cartpoleeps}).

\begin{figure}[t]
\center
\includegraphics[width=0.7\textwidth]{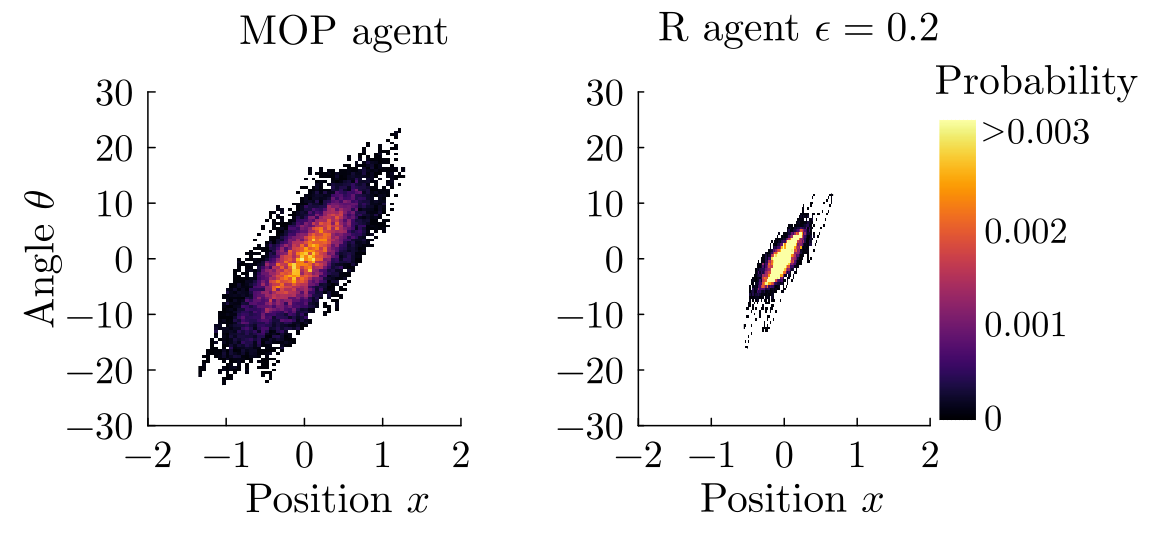} \newline
\\
 \caption{Histogram of angles and locations visited for the cartpole, as in Fig. \ref{fig:cartpole} of the main manuscript, for the MOP agent (left) and $\epsilon$-greedy R agent (right), with $\epsilon$ chosen such that MOP and R agents' lifetimes are similar (see Fig. \ref{fig:survival}c).}
 \label{fig:cartpoleeps}
 \end{figure}

\paragraph{Parameters} Mass of the cart $M = 1$, mass of the pole $m= 0.1$, length of the pole $l = 1$, acceleration due to gravity $g= 9.81$, time discretization $\Delta t = 0.02$. Unless specified differently, the discount factor was set to $\gamma = 0.98$.

\paragraph{Value interpolation}
The observed external state is a continuous four-dimensional variable, so we need to approximate the value function. In order to do so, we simply discretized the state space as described above, and use value iteration as described in Eq. (\ref{eq_v_opt_m}) in these grid points by performing a linear value interpolation for the successor states at each iteration. During a particular episode, the observed states might not be the same as the ones in the grid, so in order to compute the optimal policy at these states, we perform the same type of value interpolation as in the value iteration stage.

\paragraph{Stochastic arena} 
We introduced a slight variation to the environment, where the $x > 0$ half of the arena is noisy: agents choose an action (force), but the intended state transition of applying such an action fails with probability $\eta$ and succeeds with probability $1-\eta$. This is implemented as follows: given any state-action pair $(s,a)$ for which $x > 0$, there are two possible successor states, one corresponding to the intended action (force) chosen, and the other one corresponding to a zero force action:
\begin{equation}
    p(s' | s,a) = \begin{cases} 
			1 , & \text{if $x < 0$ and $s' \leftarrow (s,a)$}\\
            1-\eta, & \text{if $x > 0 $ and $s' \leftarrow (s,a)$} \\
            \eta, & \text{if $x > 0$ and $s' \leftarrow (s,0)$}
		 \end{cases}
\end{equation}

This stochasticity lets us differentiate between action path occupancy maximizers and action-state path occupancy maximizers by choosing any positive real value of $\beta$ in Eq. (\ref{eq_return}), because $\beta > 0$ agents will have a natural tendency to prefer $x>0$ locations.

\subsubsection{Agent-pet scenario}
\label{sec:details_pet}

An agent and a pet move in an arena with degrees of freedom that depend on the actions made by the agent, as explained next in detail.

\paragraph{Environment} A $5\times5$ arena. The middle column of arena can be blocked by a fence, a vertical obstacle that the pet cannot cross. The agent can cross it freely regardless of whether it is open or closed.
The agent can open or close the fence by performing the corresponding action when visiting the lever location, at the left bottom corner. 

\paragraph{States} 
The system's state consists of the Cartesian product of agent's location, pet’s location and binary state of the fence. So, the number of states is $1250$. 
For the sake of simplicity there is no internal states for the energy, and thus there are not absorbing states. 
The initial states of the agent and pet at the start of each episode are the middle of the second column and the right lower corner of the arena, respectively.

\paragraph{Actions} As in Sec. \ref{sec:details_mouse} the agent's actions are movements to one of the 8 neighbour locations as well as staying on the current one. Additionally, if the agent is on the ``lever'' location, an additional action is available, namely to open or close the fence, depending on its previous state.

\paragraph{Transitions}
The pet has the same available movements as the agent when the fence is open. The pet performs a random transition to any of the neighbour locations, or stays still, with the same probability.  
If the agent closes the fence, then the pet can only move on the side where it lies when closed.
For simplicity, if the fence is closed by the agent when the pet lies in the middle column, then the pet can only move to the right or left locations such that it will be at one side of the fence in the next time step.

\paragraph{Goal}
The goal of the MOP agent is to maximize discounted action-state entropy using the iterative map in Eq. (\ref{eq_z_map_m}) with
$\alpha=1$ and $\beta \in [0,1] $, parameters that measure the weight of action and state entropies, respectively.
As in the prey-predator example, we take advantage of the fact that given an action the physical transition of the agent is deterministic, while the physical transition of the pet is stochastic. Thus, the product over successor states $j$ in Eq. (\ref{eq_z_map_m}) is a product over the pet successor states.

\paragraph{Simulation details}
We ran simulations for several values of $\beta$, from $0$ to $1$ in $0.1$ steps, to interpolate between pure action entropy ($\beta=0$) and action-state entropy ($\beta=1$). We measured the fraction of time the gate was open using episodes of $2000$ steps averaged over $70$ simulations for each $\beta$, shown in Fig \ref{fig:friendly_cat}. Heat-maps in that figure correspond to the occupation probability by the pet for $\beta=0$ (left panel) and $\beta=1$ (right panel) using an episode of $5000$ steps.

\subsubsection{Quadruped-Ant}\label{sec:supplemental_ant}

\begin{figure}[t]
\center
\includegraphics[width=\textwidth]{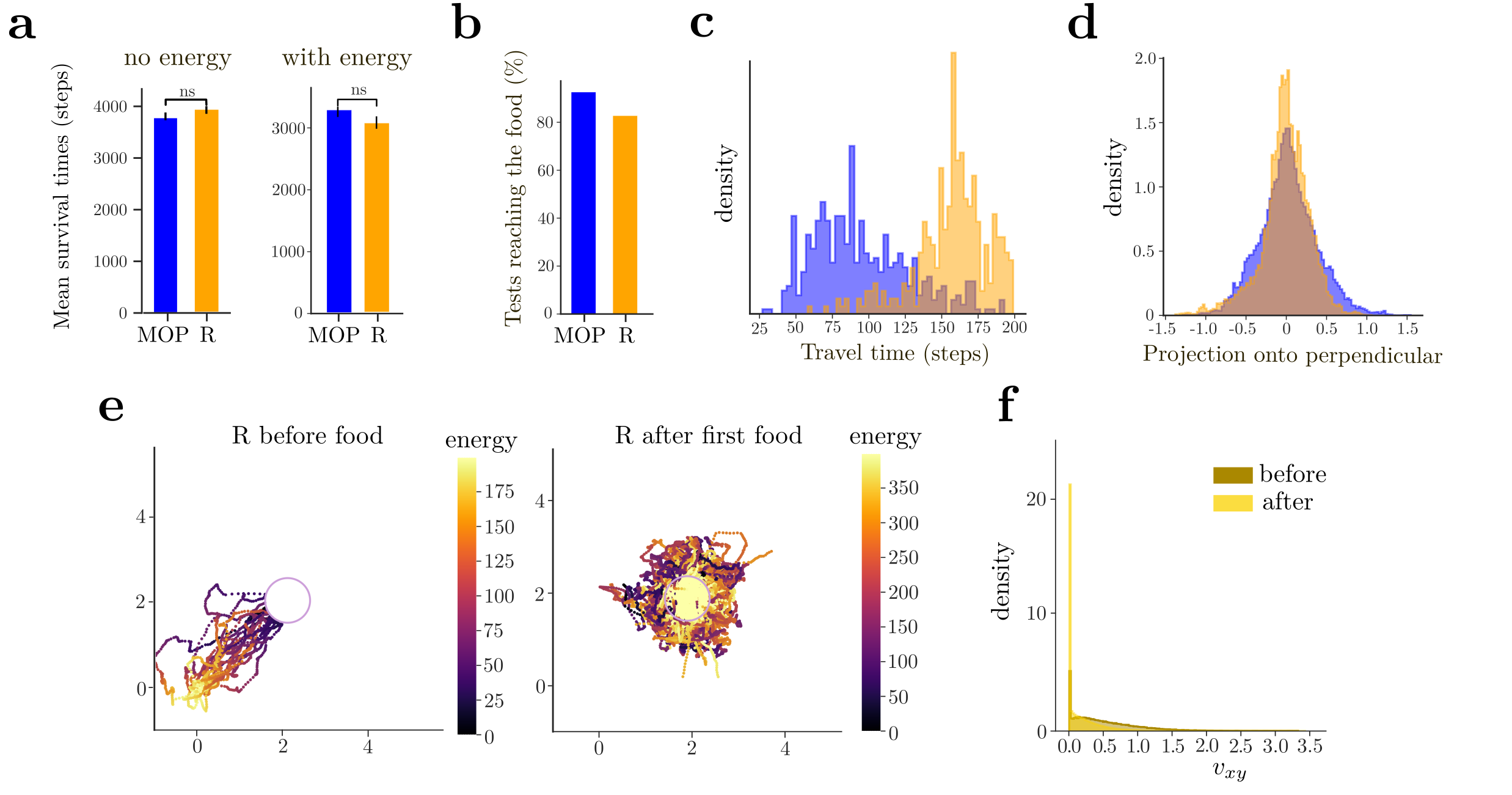} \newline
\\
 \caption{Comparison between the MOP agent and the R agent in the high-dimensional quadruped (ant) environment from Gymnasium. (a) In both experiments, $\epsilon$ for the R agent is chosen as to match the average survival time of the MOP agent. (b) Both MOP and R agents are able to reach the food source in most of the test runs. (c) Probability density function of the travel time, defined as the time the agent spends before encountering the food source for the first time. 
 (d) Probability density function of the projection of all points in a trajectory, for all trajectories, onto the line perpendicular to the shortest path connecting the origin and the food source during travel time (main diagonal). 
 (e) Display of 20 randomly chosen trajectories for the R agent before (left) and after (right) finding the food source for the first time. Colorcode defined by the energy level of the agent. (f) Probability density function of the planar speed before and after finding the food for the first time for the R agent. Both distributions show a peak at very low velocities, indicating prolonged periods of time in which the ant performs very little translational movements.}
 \label{fig:supplemental_ant}
 \end{figure}

Our goal is to show that MOP also works in high-dimensional, continuous action-state spaces. 
We employ the Ant-v4 environment from OpenAI's Gymnasium as our testing ground. We benchmark the performance of our entropy-maximizing agent against agents using an $\epsilon$-greedy strategy with rewards $R=1$ for every step except for absorbing states, where $R=0$ (R agent). All the relevant hyperparameters used to train these agents are provided in Table \ref{table:ant}.


\begin{table}[ht]
    \centering
    \caption{Hyperparameters for Ant environment}
    \begin{tabularx}{0.7\linewidth}{lX}
        \toprule
        Parameter & Value \\
        \midrule
        optimizer & Adam \\
        learning rate & \(3 \times 10^{-4}\) \\
        discount \((\gamma)\) & 0.999 \\
        replay buffer size & \(10^6\) \\
        number of hidden layers (all networks) & 2 \\
        number of hidden units per layer & 256 \\
        number of samples per minibatch & 100 \\
        number of training epoch & 300 \\
        steps per epoch & 10000 \\
        initial random steps & 20000 \\
        maximum episode length & 5000 \\
        nonlinearity & ReLU \\
        target smoothing coefficient \((\tau)\) & 0.005 \\
        number of agents & 5 \\
        test runs & 100 \\
        \bottomrule
    \end{tabularx}
    \label{table:ant}
\end{table}

In the first experiment, we study the behavioral variability of our agents and their average lifetime. The agent begins at the $(x,y)$ coordinate $(0,0)$ and follows its designated policy algorithm. The agent is considered "dead" if either it takes a step that results in the z-coordinate of its torso falling outside the range $[0.3, 1.0]$, or when the episode concludes.

In the second experiment, the agent possesses an energy value, represented as a scalar. The agent dies once it consumes all its energy. Each step taken by the agent consumes one energy point. It commences with an initial energy of 200, and its maximum energy capacity is set at 400. A food source is situated at the $(2,2)$ coordinate within the arena. Should the agent approach this source within a distance less than 0.5 from the center of its torso, it receives an energy boost of 25. The permissible z-coordinate range for the agent's torso remains consistent with the first experiment. In addition to the state vector provided by the OpenAI Gym environment, we incorporate the agent's energy level, its absolute position, and the food source's position. Interestingly, we found that the MOP agent travels to the food source much faster than the R agent (Supplemental Fig. \ref{fig:supplemental_ant}c, travel time distribution for the MOP agent is shifted towards short times), seemingly appearing less risky, given the stochastic nature of both agents' action selection. Even when the MOP agent travels faster, its trajectories are more variable compared to the R agent (Supplemental Fig. \ref{fig:supplemental_ant}d, projection of trajectories on a line perpendicular to the straight line that joins the origin and the food source).
In this second experiment, we find one out of five R agents not being able to reach the food source in a significant percentage of the test runs. For this reason, we excluded that agent from the analyses.

In the third experiment, the $x>0$ portion of the arena produces state transition noise. In the unperturbed case (first experiment), given a state $s$ and action $a$, the agent transitions deterministically to a state $s^p = step(s,a)$ (given by the Gymnasium package). For this experiment, we apply discrete noise to the resulting state $s'$ in the following way: the $i$-th coordinate of the new state $s_i'$ now independently transitions with probability $1/2$ to either of two states that are close to the unperturbed transition $s^p$. Specifically,
\begin{equation}\label{eq:cases_ant_beta}
p(s_i'|s,a) =\begin{cases}
    1/2 & \text{if} \quad x>0 \quad \text{and} \quad s_i' = (1+u) s^p_i  \\
    1/2 & \text{if} \quad x>0 \quad  \text{and} \quad s_i' = (1-u) s^p_i  \\
    1 & \text{if} \quad x\leq 0\quad \text{and} \quad s_i' = s^p_i
\end{cases},
\end{equation}
where 
$u$ is the noise magnitude parameter, which makes the transition noisy with probability $1/2$ (note that if $u=0$, the transition is deterministic, and corresponds to the unperturbed transition). The perturbations thus scale with respect to the unperturbed transition $s^p$, so that each coordinate gets noise proportional to its magnitude. We apply this noise only to the coordinates given by the 27-dimensional observation vector provided by Gymnasium, so we do not apply noise to the $x,y$ coordinates directly (see details at Ref. \cite{towers_gymnasium_2023}). We do not implement an energy constraint in this experiment. The parameter $\alpha$ is set to a constant equal to 1. 
Note that the intrinsic reward for the next-state transition obtained from being in $x>0$ is independent of $u > 0$, as $-\beta\log\left(p(s'|s,a)\right) = \beta\log2$, given that the stochasticity of the transition does not depend on the action. Following Eq. (\ref{eq:cases_ant_beta}), when $u = 0$, this intrinsic reward vanishes.

We found that $\beta \geq 0$ MOP agents are sensitive to the added noise, and they all managed to survive for almost the whole duration of the episodes after training (Fig. \ref{fig:ant_beta}a). First of all, when there is no noise ($u = 0$), the transition is deterministic, and agents do not show any preference to either side of the arena (Fig. \ref{fig:ant_beta}b, grey line, c, first row). When noise magnitude is finite but small ($u = 0.01$), $\beta = 0$ MOP agents do not show a significant preference between halves of the arena (Fig. \ref{fig:ant_beta}b,c). However, $\beta > 0$ MOP agents show a preference for the half of the arena that produces state transition noise (Fig. \ref{fig:ant_beta}b blue line, c second row). If the noise magnitude is larger ($u = 5\%,7\%$), small $\beta$ MOP agents (including $\beta = 0$) avoid the noisy half of the arena, given that noise can more easily cause the agent to fall (Fig. \ref{fig:ant_beta}b,c). Crucially, with increasing $\beta > 0$ we see an increasing preference for the noisy half of the arena (Fig. \ref{fig:ant_beta}b, increasing curves), without much effect on survival rates (Fig. \ref{fig:ant_beta}a). 

\begin{figure}[t!]
\center
\includegraphics[width=0.8\textwidth]{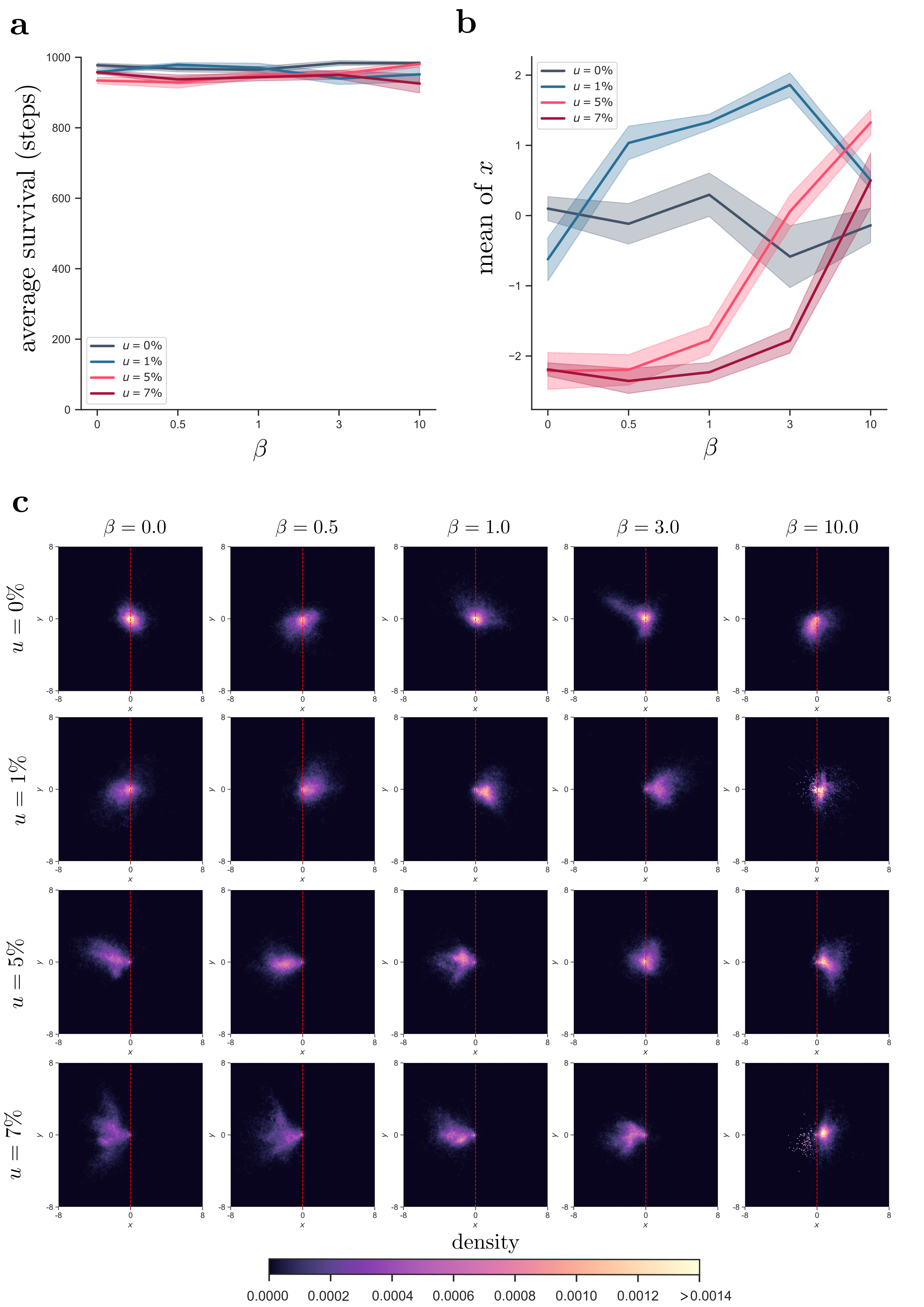} \newline
\\
 \caption{Ant shows flexible preferences for stochastic transitions in the half plane $x > 0$ as a function of the $\beta$ parameter, which controls the preference for state transition entropy, for a constant $\alpha=1$. Averages are across 1000 episodes for each of the 5 different random seeds. (a) Average survival times for the agents show that all agents learned to approximately survive 1000 step episodes. (b) Mean of $x$ position of the ant, as a function of next-state entropy weight $\beta$, for various noise magnitudes $u$.  (c) Position heatmaps of all agents for each combination of parameters. }
 \label{fig:ant_beta}
 \end{figure}

\subsection{Differences with KL regularization}\label{sec:supplemental_KL}

\begin{figure}[ht]
\center
\includegraphics[width=\textwidth]{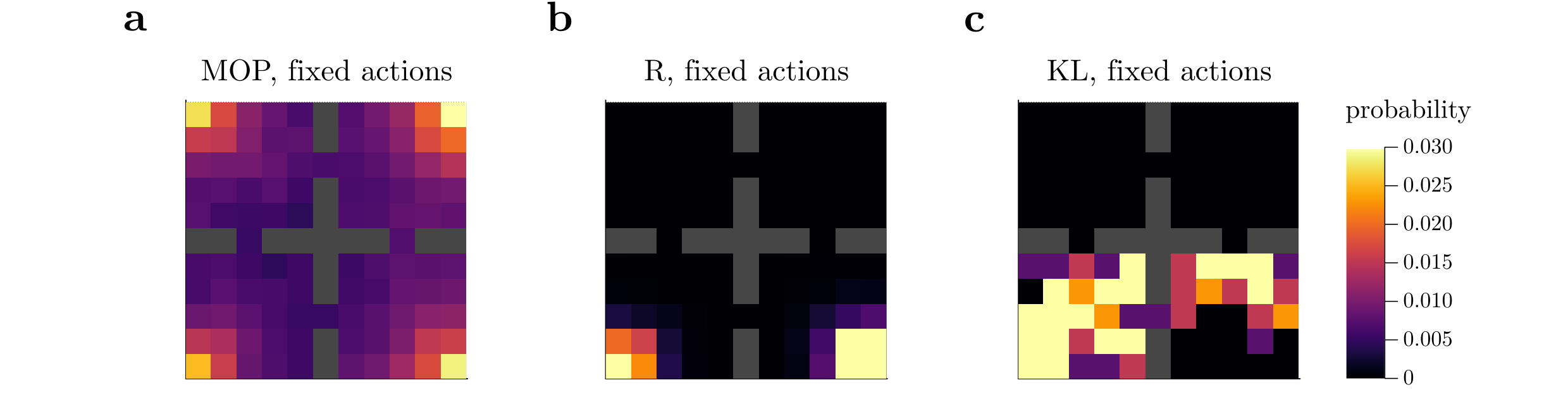} \newline
\\
 \caption{Fixing the number of actions for non-absorbing states in the gridworld environment, instead of having it be variable across states. We fix this number of actions at 9 for non-absorbing states. (a) MOP agent with fixed actions, (b) R agent with fixed actions, (c) KL regularization agent with fixed actions.
 }
 \label{fig:comparisonKL}
 \end{figure}

Given the similarity of our objective, Eq. (\ref{eq_return}) to a KL regularization scheme \cite{todorov_linearly-solvable_2006}, here we contrast predictions of using a KL divergence (relative entropy) compared to an absolute entropy objective. A relative entropy objective would look like a maximization of the cumulative immediate reward given by the negative KL divergence between a behavioral policy $\pi(a|s)$ and a default policy $\pi_0(a|s)$,

\begin{align}
    -D_{\text{KL}}(\pi(a|s)||\pi_0(a|s)) &=\sum_a \pi(a|s) \ln \frac{\pi_0(a|s)}{ \pi(a|s)} \\
    &=H(\pi(a|s)) - \ln(|\mathcal{A}(s)|) \label{eq:degenerateKL},
\end{align}
where the second equation comes from considering a default policy that is a uniform over actions, conveying the idea that we want to be as close to a uniform policy as possible. However, the lack of an extrinsic reward makes this case degenerate, in the sense that all states are equally preferred. We can see this by realizing that the highest possible immediate intrinsic reward in this case is zero, given that KL divergence is always non-negative. Thus, the optimal policy at all states is uniform over all available actions in each state. For instance, for absorbing states, where only one action is available,  Eq. (\ref{eq:degenerateKL}) is zero, making a ``KL agent" be equally attracted to non-absorbing and absorbing states, completely opposite to the motivation of our work, and illustrated in Supplemental Fig. \ref{fig:comparisonKL}c, where the agent dies very quickly. Furthermore, having a variable or a fixed number of actions for non-absorbing states is the same for a KL agent, since the relative entropy regularizes over the number of actions. This is in stark contrast with MOP, which intrinsically prefers states with a high number of actions. Having a fixed number of actions for non-absorbing states affects the behavior of MOP agents, which we can see in our gridworld experiment, comparing Fig. \ref{fig:fourrooms}b and Fig. \ref{fig:comparisonKL}. 

Finally, one could imagine setting up a default policy for a KL agent with a different set of available actions than the behavioral policy. In particular, we can set the action set for the default uniform policy to be fixed everywhere, including absorbing states. This amounts to shifting the immediate reward by a scalar everywhere, resulting in an equivalent objective as MOP
However, it is hard to see how one can justify allowing the default policy to have a different set of actions than the behavioral policy, especially because of the sum over actions in Eq. (\ref{eq:degenerateKL}) implies that we sum over all (im)possible actions and implicitly set the probability $\pi(a|s)$ of the behavioral policy to be zero for actions that are not in its support. In contrast, MOP does not have to deal with this problem, and can easily handle constant or variable number of actions for non-absorbing states.

\subsection{Comparison to Empowerment and Active Inference}

\subsubsection{Empowerment}\label{sec:empowerment_supplemental}

In this subsection we compare the behaviors attained by the MOP and empowered (MPOW) agents. We implemented empowerment for the 4-room gridworld and cartpole experiments. 

\subsubsubsection{4-room gridworld}
For the 4-room gridworld, we implemented empowerment in its original discrete formulation \cite{klyubin_empowerment_2005}. That is, we take the definition of empowerment of a particular state $s_t$ at time $t$ as the channel capacity between the agent's n-step actions $a^n_t=(a_t,a_{t+1},...,a_{t+n-1}) \in \mathcal{A}^n$ at this state, and the resulting states $s_{t+n}$,
\begin{equation}\label{eq:empowerment_discrete}
    \mathcal{C}(s_t) = \max_{p(a_t^n|s_t)} \sum_{\mathcal{A}^n, \mathcal{S}} p(s_{t+n}|s_t,a_t^n)p(a_t^n|s_t)\log\left(\frac{p(s_{t+n}|s_t,a_t^n)}{\sum_{\mathcal{A}^n}p(s_{t+n}|s_t,a_t^n)p(a^n_t|s_t)}\right),
\end{equation}
where $p(a_t^n|s_t)$ is the probability distribution of $n$-step actions that mutual information is maximized over, and $p(s_{t+n}|s_t,a_t^n)$ is the $n$-step world model, computed as $$p(s_{t+n}|s_t,a_t^n) = p(a_t^n|s_t) \prod_{\tau=0}^{n-1} p(s_{t+\tau+1}|s_{t+\tau},a_{t+\tau}) $$.

This maximization procedure is done via the Blahut-Arimoto algorithm \cite{blahut_computation_1972}, with a tolerance of $1\times 10^{-12}$ for $\norm{p_{k+1}(a_t^n|s_t) - p_{k}(a_t^n|s_t)}$, where $k$ is the iteration number of the algorithm. The initial condition for the $n$-step action probabilities is uniform over actions, and for this particular environment, very few iterations were needed for convergence (typically 3 or 4).

We initialize an agent at a particular location (in the center of a room, $(x,y) = (3,3)$), with an internal energy of $E = 30$, so that the initial state is $s = (E, x, y) = (30,3,3)$. The agent looks ahead at all possible immediately successor states $s_{t+1}$, computes their empowerment, and greedily chooses the action that corresponds to the successor state with highest empowerment (environment is deterministic). In our particular formulation, we allowed for a stochastic choice of action in case of empowerment ties between successor states. Note that the behavioral policy (greedy maximization of empowerment) and the probability of the $n$-step actions over which mutual information is maximized are different \cite{klyubin_empowerment_2005}. 

Given the nature of the arena, we implemented $5$-step empowerment, to give the agent enough lookahead to consider going into other rooms, while keeping the computations tractable, given the large amount of 1-step actions (9 for center cells). Usually, empowerment assumes a fixed amount of actions across states, and simply considers inconsequential actions to end in the current state, such as running into a wall resulting in staying in the same place. We implemented this original formulation of empowerment, although it is possible to implement state-dependent action sets, as for our formulation of MOP. This would still be meaningful for empowerment, as having more actions available results in more distinct successor states, producing similar predictions as in the original formulation of empowerment. 

\subsubsubsection{Cartpole}
For the case of the cartpole experiment, we implemented continuous-state empowerment, as developed in \citep{jung_empowerment_2011},

\begin{equation}\label{eq:empowerment_continuous}
    \mathcal{C}(s_t) = \max_{p(a_t^n|s_t)} \sum_{\mathcal{A}^n}p(a_t^n|s_t) \int_{\mathcal{S}} p(s_{t+n}|s_t,a_t^n)\log\left(\frac{p(s_{t+n}|s_t,a_t^n)}{\sum_{\mathcal{A}^n}p(s_{t+n}|s_t,a_t^n)p(a^n_t|s_t)}\right)\dd s_{t+n},
\end{equation}
where $p(s_{t+n}|s_t,a_t^n)$ is now a probability density over successor states $s_{t+n}$.

In order to have enough lookahead without needing high $n$, we used $3$-step empowerment with each action in the $3$-step action held constant for $k = 10$ time steps, in order for the computation of empowerment to be meaningfully different between states. Following \cite{jung_empowerment_2011}, we constructed a Gaussian process from where successor states $s_{t+n}$ can be drawn for each of the actions, only in the computation of empowerment (real dynamics are still deterministic). The standard deviation of the noise that blurs successor states was $\sigma = 0.01\mathbf{I}_{4\times4}$, as in \cite{jung_empowerment_2011}, independent of the action. The number of Monte Carlo samples needed to be drawn to approximate the high dimensional integral in Eq. (\ref{eq:empowerment_continuous}) was $N_{MC} = 300$. The computation of empowerment is done similarly as in the gridworld, through a Blahut-Arimoto algorithm described in \cite{jung_empowerment_2011}. Similarly, the agent looks ahead at successor states, computes their empowerment and greedily chooses the action that corresponds to the state with the highest empowerment.

\subsubsection{Active Inference}\label{sec:active_inference}

Second, we compared with an active inference approach \cite{da_costa_reward_2023}.
Note that our experiments assume full observability of states, although the partial observability condition has often been studied under active inference \cite{tschantz_reinforcement_2020}.
The Expected Free Energy (EFE) is defined as the quantity
\be
  G_{\pi,t}(s_t) = \sum_{\bar{s}_{t+1},\bar{a}_t} p_{\pi}(\bar{s}_{t+1},\bar{a}_{t}|s_t)
        \log \frac{p(\bar{s}_{t+1}|\bar{a}_{t},s_t)}{q(\bar{s}_{t+1})},
        \label{eq:FEF-G}
\ee
\noindent
which is to be minimized as a function of the policy $\pi$, which is allowed to change as a function of the state. Here $\bar{s}_{t+1} = (s_{t+1},s_{t+2},...,s_{T})$ and $\bar{a}_t = (a_t,a_{t+1},...,a_{T-1})$, that is, the sequence of future states and actions respectively from time $t$ up to some finite time $T$ given that the initial state at time $t$ is $s_t$.
Thus, $p_{\pi}(\bar{s}_{t+1},\bar{a}_{t}|s_t)$ and $p(\bar{s}_{t+1}|\bar{a}_t,s_t)$ refer to the join probability of future states and actions, and their conditional, respectively, given the initial state. 
The quantity $q(\bar{s}_{t+1})$ factorizes as $q(\bar{s}_{t+1})=\prod_{\tau=t}^{T-1}q(s_{\tau+1})$, where $q(s)$ is a time-independent probability describing the "desired" states of the agent, capturing the idea that desired states are independent of time. 
Note that $G_{\pi,t}(s_t)$ is the expectation over actions given a policy $\pi$ of the KL divergence between $p(\bar{s}_{t+1}|\bar{a}_t,s_t)$ and $q(\bar{s}_{t+1})$, that is, $G_{\pi,t}(s_t)=\mathbb{E}_{\bar{a}_t \sim \pi} \textit{KL}(p(\bar{s}_{t+1}|\bar{a}_t,s_t)||q(\bar{s}_{t+1}))$.
Note that because the time horizon is finite, here we need to consider time-dependent policies, so $\pi(a_t|s_t)$ is understood as the probability of selecting action $a_t$ at time $t$ given that the state at time $t$ is $s_t$. Time-independent policies will be suboptimal in general in finite horizon MDPs.

\begin{figure}[t]
\center
\includegraphics[width=\textwidth]{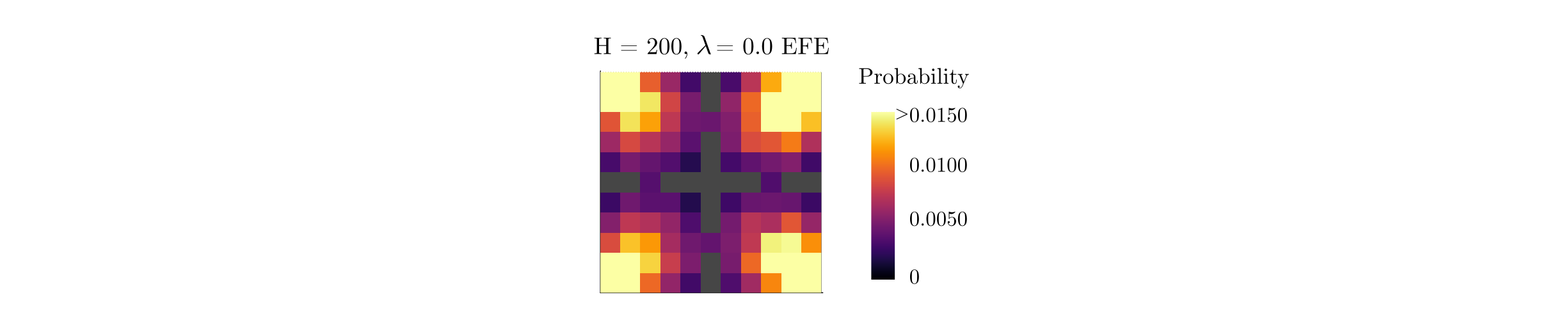} \newline
\\
 \caption{In our grid-world environment, the expected free energy (EFE) agent only visits a restricted portion of the arena, as long as the target distribution is not perfectly uniform (see Supplemental Sec. \ref{sec:supplemental_activeinference_details}). In the limit of infinite temperature ($\lambda  = 0$), the EFE degenerates to a survival maximization and the stochasticity is due to the Free Energy degeneracy across actions.}
 \label{fig:degenerateEFE}
 \end{figure}

Minimizing the objective in Eq. (\ref{eq:FEF-G}) is similar to MOP in that state transition entropy is being maximized, but it differs in that there is no action entropy and there is a regularizing distribution $q(s)$ towards which states should converge on the long run. 
The latter distinction highlights a difference in focus of the EFE and MOP approaches, but they can be made similar by just taking $q(s)$ to be uniform in state space.
However, the former difference is essential: the optimal policy of the EFE will be deterministic (see Sec. \ref{sec:active_inference}), while the optimal policy of MOP is stochastic. 
Therefore, one expects to find much larger behavioral variability under MOP than under EFE with uniform preference over all states.


By virtue of the Markov property, we have $p_{\pi}(\bar{s}_{t+1},\bar{a}_t|s_t) = \prod_{\tau=t}^{T-1} \pi(a_{\tau}|s_{\tau}) p(s_{\tau+1}|s_{\tau},a_{\tau})$ and $p(\bar{s}_t|\bar{a}_t,s_t) = \prod_{\tau=t}^{T-1} p(s_{\tau+1}|s_{\tau},a_{\tau})$.
Therefore, the objective in Eq. (\ref{eq:FEF-G}) can be recursively written as
\be
  G_{\pi,t}(s_t) = \sum_{s_{t+1},a_t} 
        \pi(a_{t}|s_{t}) p(s_{t+1}|s_{t},a_{t})
        \left[
        \log \frac{p(s_{t+1}|s_t,a_{t})}{q(s_{t+1})}
        +
        G_{\pi,{t+1}}(s_{t+1})
        \right]
        \label{eq:FEF-G-t}
\ee
\noindent
for $t<T-1$, while the terminal value is
\be
  G_{\pi,T-1}(s_{T-1}) = \sum_{s_{T},a_{T-1}} 
        \pi(a_{T-1}|s_{T-1}) p(s_{T}|s_{T-1},a_{T-1})
        \log \frac{p(s_{T}|s_{T-1},a_{T-1})}{q(s_{T})}
        \label{eq:FEF-G-T-1},
\ee
\noindent
as at time $T$ the episode terminates. 

Note that the above formalization slightly generalizes EFE \cite{da_costa_reward_2023} by allowing the possibility that the optimal policy is stochastic. Next we show that the optimal policy is deterministic. 

To find the optimal policy, we proceed backwards in time \cite{sutton_introduction_1998}. At time $T-1$ the optimal policy is deterministic because Eq. (\ref{eq:FEF-G-T-1}) is linear in the policy. The only exception is that there could be ties between several actions having the same value of the objective, in which case one can be always chosen arbitrarily, or they can be chosen randomly. Therefore, the optimal action is
\be
    a^{*}_{T-1}(s_{T-1}) = \argmin_{a}  \sum_{s_{T}} 
        p(s_{T}|s_{T-1},a)
        \log \frac{p(s_{T}|s_{T-1},a)}{q(s_{T})}
\ee
\noindent
and define the optimal return at time $T-1$ as
\be
  G^{*}_{T-1}(s_{T-1}) = \sum_{s_{T}}
        p(s_{T}|s_{T-1},a^{*}_{T-1}(s_{T-1}))
        \log \frac{p(s_{T}|s_{T-1},a^{*}_{T-1}(s_{T-1}))}{q(s_{T})}
        \label{eq:FEF-G-T-1-opt},
\ee
\noindent
Proceeding backwards, with $t=T-2, T-3, ...$, we find that again for all times the optimal policy is deterministic, and that the optimal action is 
\be
a^{*}_{t}(s_{t}) = \argmin_{a}  \sum_{s_{t+1}} 
        p(s_{t+1}|s_{t},a)
        \left[
        \log \frac{p(s_{t+1}|s_t,a)}{q(s_{t+1})}
        +
        G^{*}_{{t+1}}(s_{t+1})
        \right]
        ,
        \label{eq:FEF-G-t}
\ee
\noindent
where the optimal return is recursively computed as
\be
  G^{*}_{t}(s_t) = \sum_{s_{t+1}} 
        p(s_{t+1}|s_{t},a^{*}_{t}(s_{t}))
        \left[
        \log \frac{p(s_{t+1}|s_t,a^{*}_{t}(s_{t}))}{q(s_{t+1})}
        +
        G^{*}_{{t+1}}(s_{t+1})
        \right]
        .
        \label{eq:FEF-G-t}
\ee

\begin{figure}[!t]
    \centering
    \begin{minipage}{.35\textwidth}
        \centering
        \includegraphics[width=\linewidth]{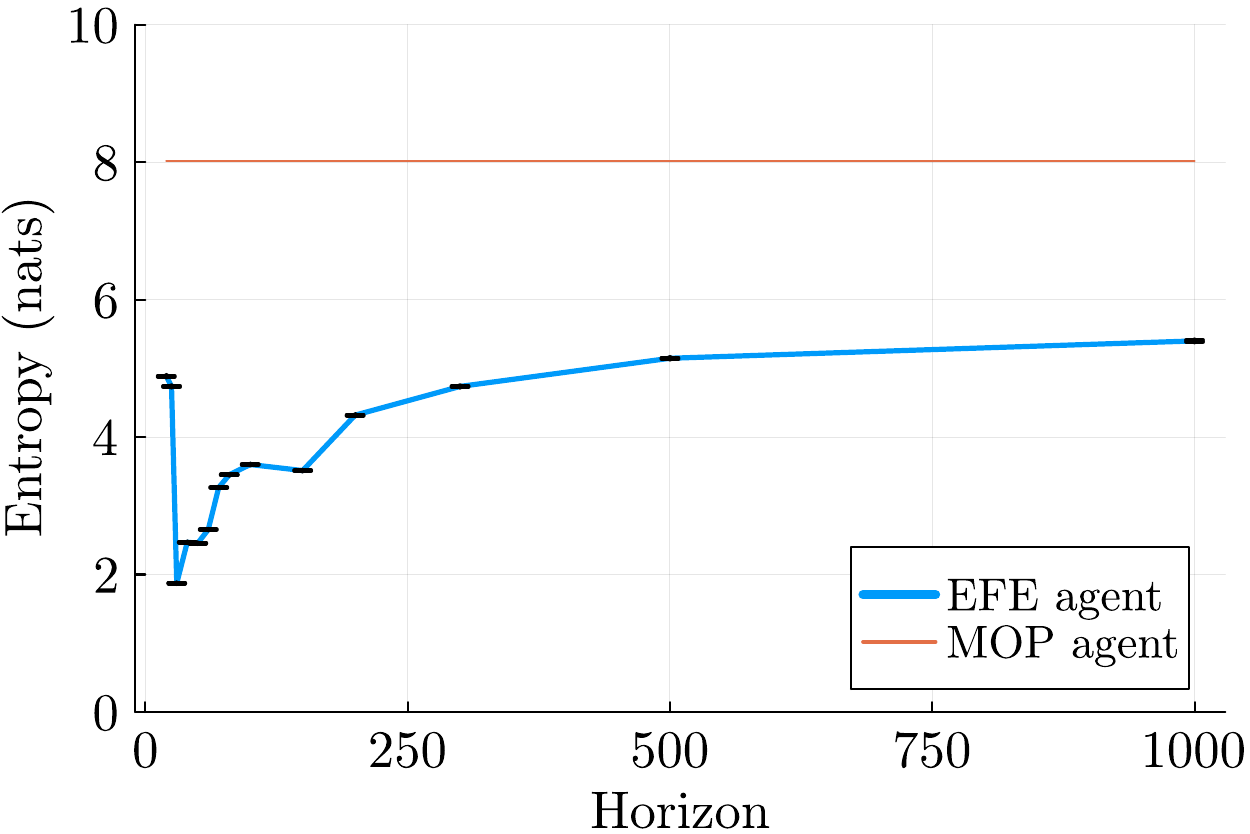}
    \end{minipage}%
    \begin{minipage}{0.65\textwidth}
        \centering
        \includegraphics[width=\linewidth]{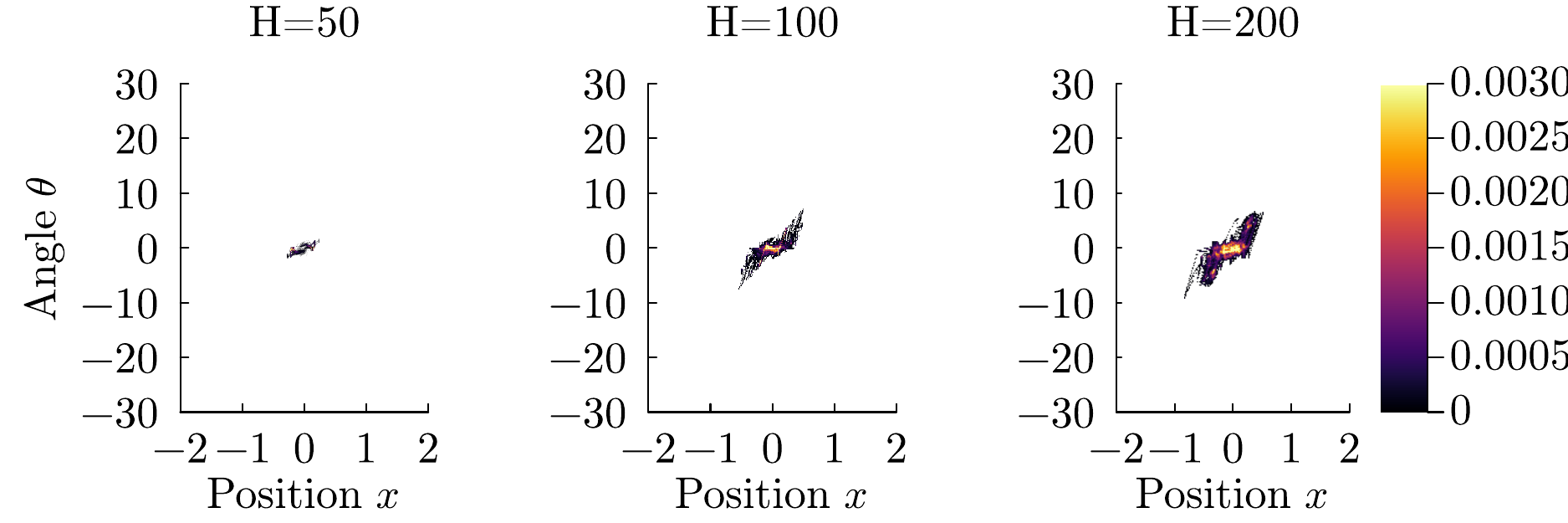}
    \end{minipage}
    \caption{(Left) Entropy of distribution of visited state space as a function of horizon shows EFE agent is never as good as the MOP agent in generating variability. (Right) Increasing the horizon lookahead makes EFE agent similar to the R agent described in the main manuscript (see Fig. \ref{fig:cartpole})}
    \label{fig:entropyEFE}
\end{figure}


\subsubsubsection{Discounted infinite-horizon sophisticated inference is identical to reward maximization under deterministic dynamics} 

Here, we show that under an infinite horizon, a discounted expected free energy that considers state-dependent policies in the future is equivalent to reward maximization under deterministic dynamics. We start with the same assumption as before that minimizing EFE optimally needs to consider future states where the agent is minimizing EFE. In a discounted, infinite horizon case, this becomes

\be
  G_{\pi}(s) = \sum_{s',a} 
        \pi(a|s) p(s'|s,a)
        \left[
        \log \frac{p(s'|s,a)}{q(s')}
        +
        \gamma G_{\pi}(s')
        \right]
        ,
        \label{eq:FEF-G-t_discounted}
\ee
where $\gamma < 1$.  Under deterministic dynamics $p(s' | s, a) = 1$ for only one state $s'$, i.e. $s' = s'(s,a)$. So we can rewrite the EFE as
\be
  G_{\pi}(s) = \sum_{a} 
        \pi(a|s)
        \left[
        -\log\left(q(s'(s,a))\right)
        +
        \gamma G_{\pi}(s'(s,a))
        \right]
        .
\ee
Asking to minimize $G$ is equivalent to maximizing $-G$, which means that the optimal Bellman equation for sophisticated active inference in this case turns to
\be
  G^*(s) = \max_a
        \left[
        \log\left(q(s'(s,a))\right)
        +
        \gamma G^*(s'(s,a))
        \right]
        .
\ee
Simply rewriting $\log (q(s'(s,a))) = r(s'(s,a))$ gives us the typical Belllman equation for MDPs.

In particular, when the preferred distribution is uniform on a finite portion of state space, under the presence of absorbing states outside this portion, this scheme is identical to survival maximization. This is because we can define $q(s'(s,a)) = 1/ V$, where $V$ is the volume of the portion of state space that is not absorbing, for $s'(s,a)$ that stays in this portion. For states outside this region, we can establish $q(s'(\text{absorbing})) \ll q(s'(\text{alive}))$, such that $\log (q)$ is bounded. Therefore, for long horizons, we expect the EFE agent to behave identically to our previously defined R agent that maximizes survival. We confirm this expectation in Supplemental Fig. \ref{fig:entropyEFE}.

\subsubsubsection{Details of simulations}\label{sec:supplemental_activeinference_details}

One can define a target distribution $q(s)$ through a Boltzmann distribution, instead of a hard maximization of rewards, as similar to what is done in soft RL \cite{haarnoja_reinforcement_2017,haarnoja_soft_2018}. The target distribution $q(s)$ can be defined as
\begin{equation}
    q_\lambda(s) = \frac{1}{Z_\lambda}\exp(\lambda R(s)),
\end{equation}
where $\lambda$ is an inverse temperature, which expresses how motivated the agent is to maximize reward \cite{da_costa_reward_2023}.

\paragraph{Grid world}
We take $R = \delta$ for being in the food and $R = 0$ otherwise. For a large temperature, $\lambda$ is small, and thus $q_\lambda(s)$ is very close to an uniform distribution --it has a little bump on the reward location.
Even a tiny bump breaks the symmetry of the EFE agent in deterministic environments such that it absolutely prefers the food source location, and thus behavior collapses to the occupancy of that single state (see Fig. \ref{fig:MPOW_EFE}). 

\paragraph{Cartpole}
We define the rewards similar to the R agent, $R = 1$ for non-absorbing states and $R = 0$ for absorbing states. This amounts to a uniform target distribution $q(s)$ over non-absorbing states.

 \subsection{Relationship to Maximum Entropy Reinforcement Learning and goal directedness}\label{sec:maxentRLgoal}
 
 The objective of maximizing action-state path entropy in Eq. (\ref{eq_expected_return}) for the special case $\beta=0$ can be obtained from the maximum entropy reinforcement learning (MaxEnt RL) formulation \cite{todorov_efficient_2009,ziebart_modeling_2010,haarnoja_soft_2018}
 \begin{equation}\label{eq:maxentRL}
      V_\pi (s) = \mathbb{E}_\pi \left[\sum_{t=0}^\infty \gamma^t \left(r(s_t,a_t) + \alpha \mathcal{H}(\pi(\cdot|s_t))\right)\Big| s_0 = s \right],
 \end{equation}
 by setting the reward $r(s,a)=0$ for all states and actions, and therefore there is no difference between the two approaches in this particular case. 
 However, this reduction obscures the fact that we can generate goal-directed behaviors in H-agents \emph{without} the need of specifying rewards --indeed, this is one of the main accomplishment of our work.
 To see this, we first quantify how a MaxEnt RL agent gets reward in the four-room grid world defined in Supplemental Sec. \ref{sec:4room}, as a function of the temperature parameter $\alpha$. In this case, a sensible goal is ``eating food'' (that is, defining $r(s,a)=1$ at the food locations, and zero everywhere else). Trivially, when $\alpha \ll 1$ in Eq. (\ref{eq:maxentRL}), the goal is simply to maximize the future expected reward, equivalent to the $\epsilon$-greedy R agent defined in Supplemental Sec. \ref{sec:r_agent}, for $\epsilon = 0$ (Figure \ref{fig:maxentRL}a, leftmost points). In contrast, for $\alpha \gg 1$, we recover the MOP agent in practice (due to the environment being deterministic). In this case, the agent mostly focuses on maximizing future expected entropy, and getting small eating rate (Figure \ref{fig:maxentRL}a, rightmost points). Therefore, the temperature $\alpha$ quantifies how ``goal directed'' the agent should be, where the goal here is understood as getting food, and the entropy term is understood as a regularizer that promotes exploration of the arena.

\begin{figure}[t]
\center
\includegraphics[width=\textwidth]{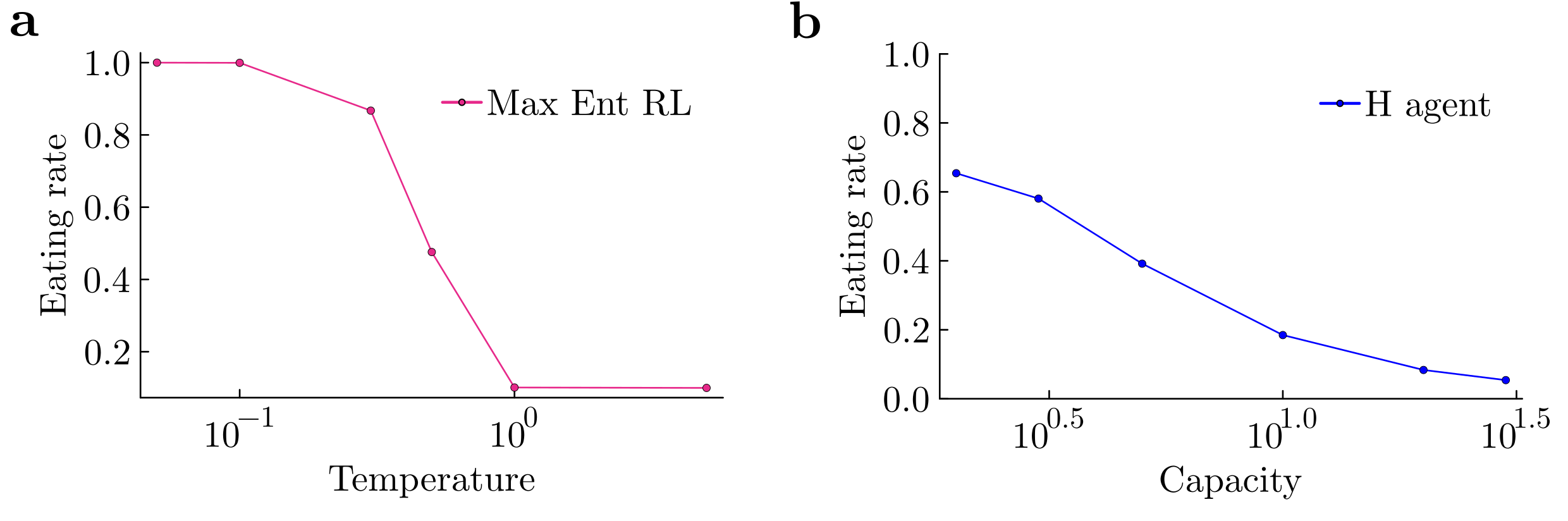} \newline
\\
 \caption{Reward is not necessary for "goal-directed" behavior. (a) Eating rate as a function of the temperature parameter $\alpha$ in Equation (\ref{eq:maxentRL}) for a MaxEnt RL agent in the four-room grid world. (b) Eating rate as a function of the capacity for a MOP agent in the four-room grid world.}
 \label{fig:maxentRL}
 \end{figure}

To aid in showing our central result that an extrinsic reward is not necessary for ``goal directed behavior'', we take the MOP agent and vary its energy capacity (see Supplemental Sec. \ref{sec:4room}). For large capacities, the MOP agent can largely ignore the food most of the time, obtaining small eating rate (Figure \ref{fig:maxentRL}b, right-most points). This is because food is conceived as the means to accomplish the goal of maximizing future path occupancy. In contrast, when capacity is small, the MOP agent needs to get the food much more frequently to avoid the absorbing state, thus getting much higher eating rates (Figure \ref{fig:maxentRL}b, leftmost points). The remarkably strong qualitative similarities between the two panels in the figure show that by reinterpreting the concept of reward, one can forego the need of specifying a reward function, and focus on more universal principles of behavior.

\subsection{Non-additivity of mutual information and channel capacity}
\label{sec:non-additivity-MI}

Here we show that mutual information over Markov chains does not obey the additive property. It suffices to prove our statement for paths of length two. Thus, we ask whether the mutual information between actions $(a_0,a_1)$ and states $(s_1,s_2)$ given initial state $s_0$

\be
\nonumber
\text{MI}_{\text{global}} = 
\sum_{a_0,a_1,s_1,s_2} p(a_0,s_1,a_1,s_2|s_0) 
\ln \frac{p(a_0,s_1,a_1,s_2|s_0)}{p(a_0,a_1|s_0)p(s_1,s_2|s_0)}
\ee

\noindent
equals the sum of the per-step mutual information

\be
\nonumber
\text{MI}_{\text{local}}  =      \sum_{a_0,s_1} 
p(a_0,s_1|s_0)  
\ln \frac{p(a_0,s_1|s_0) 
}{p(a_0|s_0)p(s_1|s_0)}
+
\sum_{a_0,a_1,s_1,s_2} p(a_0,s_1,a_1,s_2|s_0)
\ln \frac{p(a_1,s_2|s_1)}{p(a_1|s_1)p(s_2|s_1)}
\ee

\noindent
where $p(a_0,s_1,a_1,s_2|s_0)=\pi(a_0|s_0)p(s_1|s_0,a_0) 
\pi (a_1|s_1) p(s_2|s_1,a_1)$ and $p(a_0,s_1|s_0)=\pi(a_0|s_0)p(s_1|s_0,a_0)$.
Using Bayes' rule and the Markov property, the above quantities can be rewritten as

\bea
\nonumber
\text{MI}_{\text{global}} &= &
\sum_{a_0,a_1,s_1,s_2} p(a_0,s_1,a_1,s_2|s_0) 
\ln \frac{p(a_0,a_1|s_0,s_1,s_2)}{p(a_0,a_1|s_0)}
\\
\nonumber
&= &
\sum_{a_0,a_1,s_1,s_2} p(a_0,s_1,a_1,s_2|s_0) 
\ln \frac{p(a_0|s_0,s_1)p(a_1|s_1,s_2)}{p(a_0,a_1|s_0)}
\\
\nonumber
&= &
\sum_{a_0,a_1,s_1,s_2} p(a_0,s_1,a_1,s_2|s_0) 
\ln \frac{p(a_0|s_0,s_1)p(a_1|s_1,s_2)}{\pi(a_0|s_0) 
	p(a_1|s_0,a_0)}
\\
\nonumber
&= &
\sum_{a_0,a_1,s_1,s_2} p(a_0,s_1,a_1,s_2|s_0) 
\ln \frac{p(a_0|s_0,s_1)p(a_1|s_1,s_2)}{\pi(a_0|s_0) 
	\sum_s \pi(a_1|s) p(s|s_0,a_0)}
\\
\nonumber
&  = &  \sum_{a_0,s_1} 
p(a_0,s_1|s_0)  
\ln \frac{p(a_0|s_0,s_1)}{\pi(a_0|s_0)}
+
\sum_{a_0,a_1,s_1,s_2} p(a_0,s_1,a_1,s_2|s_0)
\ln \frac{p(a_1|s_1,s_2)}{ 
	\sum_s \pi(a_1|s) p(s|s_0,a_0)}
\eea

\noindent
and

\be
\nonumber
\text{MI}_{\text{local}}  =      
\sum_{a_0,s_1} 
p(a_0,s_1|s_0)  
\ln \frac{p(a_0|s_0,s_1) 
}{\pi(a_0|s_0)}
+
\sum_{a_0,a_1,s_1,s_2} p(a_0,s_1,a_1,s_2|s_0)
\ln \frac{p(a_1|s_1,s_2)}{\pi(a_1|s_1)}
\ee

The quantities $\text{MI}_{\text{global}}$ and $\text{MI}_{\text{local}}$ are remarkable similar except for the denominator in the $\ln$ of the last term in each expression. Therefore, equality between $\text{MI}_{\text{global}}$ and $\text{MI}_{\text{local}}$ holds iff

\be
\nonumber
\sum_{a_0,a_1,s_1,s_2} p(a_0,s_1,a_1,s_2|s_0)
\ln \sum_s \pi(a_1|s) p(s|s_0,a_0)
  =      
\sum_{a_0,a_1,s_1,s_2} p(a_0,s_1,a_1,s_2|s_0)
\ln \pi(a_1|s_1)
,
\ee

\noindent
which is not true for all choices of policy and transitions probabilities. To see this, take a Markov chain where the action $a_0=0$ from $s_0=0$ is deterministic, but results in two possible successor states $s_1=1$ or $s_1=2$ with equal probability $1/2$. From $s_1=1$ the policy takes actions $a_1=1$ and $a_1=2$ with probability $1/2$. From $s_1=2$ the policy is deterministic, that is, $a_1=3$ with probability $1$. A simple calculation shows that the left side equals $-\frac{3}{2} \ln 2$, while the right side equals a different quantity, $-\frac{1}{2} \ln 2 $. 

\subsection{Video captions}
\paragraph{Video 1} Animation of a portion of an episode comparing the behaviors of the MOP agent and the $\epsilon$-greedy R agent for the four-room grid world environment (see main text, Fig. \ref{fig:fourrooms}, for more details).

\paragraph{Video 2} Animation of a portion of an episode of the MOP agent (mouse) behaving in the predator-prey scenario detailed in Fig. \ref{fig:cat_mouse}.

\paragraph{Video 3} Animation of a portion of an episode of the R agent (mouse) behaving in the predator-prey scenario detailed in Fig. \ref{fig:cat_mouse}.

\paragraph{Video 4} Animation of a portion of an episode comparing the behaviors of the MOP agent and the R agent for the cartpole experiment detailed in Fig. \ref{fig:cartpole}.

\paragraph{Video 5} Animation of the state space trajectories (in an angle-position projection) traveled by the MOP and the R agents from Video 4, sped up four times the original frame rate.

\paragraph{Video 6} Animation of a portion of an episode comparing the behaviors of the MOP agent and the lifetime-matching $\epsilon-$greedy R agent for the cartpole experiment detailed in Fig. \ref{fig:cartpoleeps}.

\paragraph{Video 7} Animation of a portion of an episode comparing the behaviors of the MOP agent and the MPOW and EFE agents for the four-room grid world environment (see main text, Fig. \ref{fig:MPOW_EFE}, for more details).

\paragraph{Video 8} Animation of a portion of an episode comparing the behaviors of the MOP agent and the MPOW and EFE agents for the cartpole experiment (corresponding to Fig. \ref{fig:MPOW_EFE}).

\paragraph{Video 9} Animation of a portion of an episode comparing the behaviors of the MOP agent and the R agent for the quadruped experiment without energetic constraints (corresponding to upper row of Fig. \ref{fig:ant}).

\paragraph{Video 10} Animation of a portion of an episode comparing the behaviors of the MOP agent and the R agent for the quadruped experiment with energetic constraints (corresponding to lower row of Fig. \ref{fig:ant}).

\end{document}